\documentclass[twoside,11pt]{article}

\usepackage[preprint]{jmlr2e}

\jmlrheading{1}{2000}{1-48}{4/00}{10/00}{meila00a}{Ziyu Wang, Yuhao Zhou, Tongzheng Ren and Jun Zhu}

\usepackage{amsmath,amssymb,graphicx,hyperref,latexsym,xcolor,amsfonts,nicefrac}
\usepackage{verbatim}
\usepackage[figuresright]{rotating}
\usepackage{framed}
\usepackage{xcolor}
\usepackage{accents}
\usepackage{subcaption}
\usepackage{enumitem}
\usepackage{cancel}
\usepackage{booktabs}
\usepackage{comment}
\usepackage[ruled,vlined]{algorithm2e}
\usepackage{aligned-overset}

\input{def.set}

\newcommand{\skipInSubmission}[1]{#1}
\newcommand{\inWSSubmission}[1]{} %

\DeclareMathOperator*{\argmin}{arg\,min}

\newcommand{\conditionalPlot}[1]{
  \ifx\skipPlots\undefined
  #1
  \else
  {\begin{center}\bf TODO this plot is suppressed to save compilation time. Comment the command \texttt{\textbackslash skipplot} in \texttt{main.tex} to show it \end{center}}
  \fi
}
\newcounter{inpItem}
\renewcommand{\theinpItem}{(\roman{inpItem})}
\newenvironment{inplaceEnumerate}{\setcounter{inpItem}{0}}{}
\newcommand{\inplaceItem}{\refstepcounter{inpItem}{\theinpItem~}}
\newcommand{\parspace}{\vspace{1.8ex}}

\ShortHeadings{Quasi-Bayesian Dual IV}{Wang, Zhou, Ren and Zhu}
\firstpageno{1}

\begin{document}

\title{Quasi-Bayesian Dual Instrumental Variable Regression}

\inWSSubmission{
\author{\name Anonymous Authors}
}

\skipInSubmission{
\author{%
\name Ziyu Wang\textsuperscript{$\ast$}\phantom{\thanks{Equal contribution. 
}}
\email wzy196@gmail.com \\ 
\name Yuhao Zhou\textsuperscript{$\ast$}  \email yuhaoz.cs@gmail.com \\
\addr Tsinghua University \\ 
Beijing, 100084, China 
\AND 
\name Tongzheng Ren \email tongzheng@utexas.edu \\ 
\addr UT Austin \\ 
Austin, TX 78712, USA
\AND
\name Jun Zhu \email dcszj@mail.tsinghua.edu.cn \\ 
\addr Tsinghua University \\ 
Beijing, 100084, China 
}
}

\editor{TBD} %

\maketitle

\begin{abstract}%
Recent years have witnessed an upsurge of interest in employing flexible machine learning models for instrumental variable (IV) regression, but the development of uncertainty quantification methodology is still lacking. In this work we present a novel quasi-Bayesian procedure for IV regression, building upon the recently developed kernelized IV models and the dual/minimax formulation of IV regression. We analyze the frequentist behavior of the proposed method, by establishing minimax optimal contraction rates in $L_2$ and Sobolev norms, and discussing the frequentist validity of credible balls. We further derive a scalable inference algorithm which can be extended to work with wide neural network models. Empirical evaluation shows that our method produces informative uncertainty estimates on complex high-dimensional problems. 
\end{abstract}


\section{Introduction} \label{sec:intro}

Instrumental variable (IV) regression is a standard approach for estimating causal effect from confounded observational data, and is widely used in areas such as economics, epidemiology and clinical research \citep{angrist2008mostly,greenland2000introduction,cuzick1997adjusting}. 
In the presence of confounding, any regression method estimating $\EE(\by\mid\bx)$ cannot recover the causal relation $f_0$ between the outcome $\by$ and the treatment $\bx$, since the residual $\bu=\by-f_0(\bx)$ is correlated with $\bx$ due to the unobserved confounders. IV regression enables identification of the causal effect through the introduction of \emph{instruments}, variables $\bz$ that are known to influence $\by$ only through $\bx$. 

Recent years have seen great development in adopting flexible machine learning models for IV regression \citep{hartford2017deep,lewis2018adversarial,bennett2019deep,singh_kernel_2020,bennett_variational_2020}, which bring the promise of better adaptation to complex structures in data. 
However, their theoretical properties are less well understood than classical nonparametric approaches, as we discuss in Section~\ref{sec:related-work}. 
Moreover, there is a lack of principled uncertainty quantification measures for these flexible IV models. 
Uncertainty quantification is especially important for IV analysis, since unlike in standard supervised learning scenarios, we do not have 
(unconfounded) validation data. %
Moreover, it is possible that the instrument of choice is only weakly correlated  with %
$\bx$, in which case point estimators suffer from high variance \citep{stock2002survey}. This problem is exacerbated in the nonparametric setting, where IV estimation is typically an ill-posed inverse problem, in which statistical challenges present for the recovery of higher order nonlinear effects. We refer readers to \cite{horowitz_applied_2011} for an introduction to this matter. 

Uncertainty quantification for IV regression is challenging. 
For example, 
while it may appear natural to consider a Bayesian approach, specification of the likelihood requires knowledge of the data generating process, which is typically not assumed in IV regression. 
Consequently, %
existing work on Bayesian IV \citep{
    kleibergen1998bayesian,
    kleibergen_bayesian_2003,conley2008semi,
    lopes2014bayesian,
    wiesenfarth2014bayesian}
typically make stronger assumptions, and 
assume the following data generating process: 
$
\bx=g(\bz)+\mbf{u}_1, ~ \by=f(\bx)+\mbf{u}_2,
$ where $(\bu_1,\bu_2)$ are correlated and independent of $\bz$. 
Posterior inference is then conducted 
on $g,f$ as well as the \emph{distribution of} the unobserved confounders $(\bu_1,\bu_2)$. 
Additive, independent error in $\bx$ is 
an unnecessary assumption for point estimators, and is difficult to check in high dimensions. 
The need to model the generating process of $(\bu_1,\bu_2)$ also introduces extra risks of model misspecification, and Bayesian inference over the generative model 
is computationally expensive, especially on complex high-dimensional datasets. 
None of these issues present if only point estimation is needed.

For the above reasons, it is appealing to turn to an alternative \emph{quasi-Bayesian} approach \citep{chernozhukov_mcmc_2003,
liao_posterior_2011,kato_quasi-bayesian_2013}. 
Quasi-Bayesian analysis views IV estimation as a generalized method of moments procedure, and defines the quasi-posterior as a Gibbs distribution constructed from a chosen prior and violation of the conditional moment restrictions. 
It does not require full knowledge of the data generating process, and thus does not 
suffer from the aforementioned drawbacks. %
However, 
computation of the quasi-posterior is non-trivial in the IV setting, as its density contains a conditional expectation term $\EE(\by-f(\bx)\mid\bz)$, which itself needs to be estimated. 
Moreover, the need to estimate the conditional expectation also brings challenges to theoretical analysis of the frequentist behavior. 
So far, quasi-Bayesian analysis for nonparametric IV is only developed with classical models such as wavelet basis \citep{liao_posterior_2011,kato_quasi-bayesian_2013} or Nadaraya-Watson smoothing \citep{florens_nonparametric_2012}, and with more limited theoretical understanding compared with their frequentist counterparts (see Section~\ref{sec:related-work}). 
Besides, numerical study has been limited in previous work, and thus 
little is known about the empirical performance of the quasi-Bayesian approach for nonparametric IV. 

In this work, we present a novel quasi-Bayesian procedure for IV regression, building upon the recent 
development in kernelized IV models \citep{singh_kernel_2020}. 
We employ a Gaussian process (GP) prior and construct the quasi-likelihood based on a recently proposed dual formulation of IV estimation \citep[e.g.,][]{muandet_dual_2020,dikkala_minimax_2020}. 
Following this, posterior sampling can be cast as a minimax optimization problem, based on which we construct an inference algorithm. The algorithm is further adapted to work with flexible neural network models, 
in which case its behavior can be justified by analyzing the neural networks in the kernel regime. We evaluate the proposed method on several synthetic datasets, showing that it produces informative uncertainty estimates on complex nonlinear problems, and is particularly advantageous in the finite-sample setting. 

In addition to the methodological contributions, 
this work also contributes to the theoretical understanding of nonparametric IV methods. 
We establish optimal posterior contraction rates for both $L_2$ and Sobolev norms. 
Comparing with previous work on quasi-Bayesian IV \citep{kato_quasi-bayesian_2013}, we allow for more flexibility in the choice of models. 
From these results we deduce optimal convergence rates for the point estimator of kernelized IV, which fill in an important gap in literature, and further establish kernelized IV as a principled estimation method with potential theoretical advantages (see Section~\ref{sec:related-work} for discussion). 
We further analyze the frequentist coverage properties of the posterior credible balls, and provide intuition on the nonasymptotic behavior of the quasi-posterior. 
As a side contribution, we apply our proof technique to Gaussian process regression, and improve the result of \cite{van_der_vaart_information_2011} by relaxing a joint H\"older and Sobolev regularity assumption. 
Finally, we note that this work is a substantial extension of a conference paper \citep{wang2021scalable}, with all theoretical results in Section 4-5 being new.

The rest of the paper is organized as follows: we set up the problem in Section~\ref{sec:setup} and derive the quasi-posterior in Section~\ref{sec:qp}. 
We then present the asymptotic analysis in 
Section~\ref{sec:theory}, and discuss the nonasymptotic behavior of the quasi-posterior in Section~\ref{sec:nonasymptotic}. 
The approximate inference algorithm is presented in Section~\ref{sec:rf}. Section~\ref{sec:related-work} discuss related work, and Section~\ref{sec:exp} presents numerical studies. Finally, we discuss conclusion in Section~\ref{sec:conclusions}.

\section{Notations and Setup}\label{sec:setup}

\subsection{Notations} We use boldface ($\bx,\by,\bz$) to represent random variables on the space $\cX \times \cY \times \cZ$, regular font ($x,y,z$) to denote deterministic values. 
$[n]$ denotes the set $\{ 1, 2, \cdots, n \}$.
$\{(x_i,y_i,z_i):i\in [n]\}$ indicates the training data. 
We use the notations $X:=(x_1,\ldots,x_n)\in\cX^{ n}$, $f(X) := (f(x_1),\ldots,f(x_n))$; likewise for $Y,Z$. For finite-dimensional vectors $\theta,\theta'\in\RR^m$, 
we use $\|\theta\|_2$ and $\<\theta,\theta'\>_2$ to denote the Euclidean norm and inner product w.r.t.~the data distribution, respectively. For square-integrable functions, $\|\cdot\|_2$ and $\<\cdot,\cdot\>_2$ refer to $L_2$ norm and inner product.  
$\|\cdot\|$ denotes the sup norm. 
For any operator $A: H_1 \to H_2$ between Hilbert spaces $H_1$ and $H_2$, we denote its adjoint %
by $A^*: H_2 \to H_1$. 
When $H_1 = H_2$ and $\lambda \in\RR$, we use the notation $A_\lambda := A + \lambda I$. %

$\lesssim,\gtrsim,\asymp$ represent (in)equality up to constants. The hidden constants may depend on those introduced in the assumptions, but will not depend on data in any other way. Therefore, asymptotic claims such as the contraction results will hold uniformly, over a family of data distributions satisfying the assumptions.

\subsection{Instrumental Variable Regression}
Denote the treatment and response variables as $\bx,\by$, the instrument as $\bz$, and the true \emph{structural function} of interest as $f_0$. 
Consider the data generating process 
$
\by = f_0(\bx)+\bu,
$
where we assume the unobserved residual $\bu$ satisfies $\EE(\bu\mid\bz)=0$, but may  
correlate with $\bx$. Intuitively, the assumption requires that $\bz$ only influence $\by$ through $\bx$. 
Then $f_0$ satisfies 
\begin{equation}\label{eq:iv-gmm-conds}
\EE(\by - f_0(\bx)\mid\bz) = 0~~\text{a.s.}\;[P(dz)],
\end{equation}
where $P$ denotes the data distribution. This \emph{conditional moment restriction} (CMR) formulation %
is the standard definition in literature \citep[e.g.,][]{newey2003instrumental,kato_quasi-bayesian_2013}, and is used in the recent work on machine learning models for IV. 
It connects to GMM as \eqref{eq:iv-gmm-conds} can be viewed as a continuum of generalized moment constraints. 

Note that \eqref{eq:iv-gmm-conds} does not 
place any structural constraint on the conditional distribution $p(\bx\mid \bz)$, such as additive noise; hence, it does not require full knowledge of the data generating process. 
Also, as discussed in, e.g., \citet{hartford2017deep}, the setup can also be extended to incorporate observed confounders $\bv$, by including $\bv$ in both $\bx$ and $\bz$. 

IV regression can also be viewed as solving a linear \emph{inverse problem}: introduce the 
conditional expectation operator 
$E: L_2(P(dx))\rightarrow L_2(P(dz)), f\mapsto \EE(f(\bx)\mid \bz=\cdot)$, 
and define $b(z) := \EE(\by\mid \bz=z)$. 
Then \eqref{eq:iv-gmm-conds} can be written as 
\begin{equation*}%
E f_0 = b. 
\end{equation*}
Although IV is more difficult than standard inverse problems, due to the need to estimate $E$ from data, it is often helpful to draw parallels from the inverse problem literature. 

\subsection{Dual/Minimax Formulation of IV}
Let $\cH,\cI$ be normed function spaces on $\cX,\cZ$, respectively. 
The inverse problem view motivates the use of the following objective
\begin{equation}\label{eq:iv-obj-general}
\min_{f\in\cH}\quad \cL(f) := d_n^2(\hat{E}_n f - \hat{b}) + \efflam \Omega(f),
\end{equation}
where $\hat{E}_n: \cH\rightarrow\cI$ is an empirical approximation to the restriction of $E$ on $\cH$, 
$\hat b$ is an estimator of $b = E f_0$, 
$\{d_n\}$ is a sequence of suitable (semi-)norm on $\cI$, and $\Omega:\cH\rightarrow\RR$ is a regularization term. 
We restrict our attention to 
$\hat E_n$ which solves a ridge-regularized least square regression problem
\begin{align}
\hat E_n(f) &:= 
\argmin_{g\in\cI} \frac{1}{n}\sum_{j=1}^n (f(x_j) - g(z_j))^2 + \effnu\|g\|_{\cI}^2  \label{eq:Ehat-defn}
\end{align}
and define $\hat b$ as solving a similar problem as above, with $f(x_j)$ replaced by $y_j$. 
Consider setting $\Omega(f) := \frac{1}{2}\|f\|_\cH^2$ and
$d_n^2(g) := \frac{1}{n}\sum_{j=1}^n g(z_j)^2 + \effnu\|g\|_\cI^2$. 
Appealing to the Fenchel duality identity
$\frac{1}{2}u^2=\sup_{v\in\RR}\left(uv-\frac{1}{2}v^2\right)$, 
and rearranging terms, we arrive at the dual/minimax objective for IV regression \citep{muandet_dual_2020,bennett_variational_2020,liao_provably_2020,dikkala_minimax_2020,dai_learning_2016}:
\begin{equation}
\min_{f\in\cH} \cL(f) \equiv 
\min_{f\in\cH}
\max_{g\in\cI} \frac{1}{n}\sum_{j=1}^n \left(
  (f(x_j)-y_j)g(z_j) - \frac{g(z_j)^2}{2}
\right) - \frac{\effnu}{2}\|g\|_\cI^2 + \frac{\efflam}{2}\|f\|_\cH^2. \label{eq:dualiv-obj}
\end{equation}

The formulation \eqref{eq:dualiv-obj} is appealing for two reasons. 
From the computational aspect, it 
circumvents the need to directly compute $\hat{E}_n$, an operator between the typically infinite-dimensional spaces $\cH$ and $\cI$, 
and can be solved efficiently with stochastic gradient descent-ascent (SGDA) 
if the function spaces can be approximated with parametric models. 
Its particular choice of $d_n$ also simplifies theoretical analysis and leads to tighter error rates \citep{dikkala_minimax_2020}. \SkipNOTE{
Without it, we need to additionally show $\|\hat g\|_\cI^2 <= C \|Ef\|_2^2 / delta_n^2$, for some \emph{specific} $C$.  This is because we need to lower bound $\ell_n$ for $f \in \err$ with
$\ell_n(f) = \Psi_n(f,\hat g) \ge c_1 \Psi_n(f,g_j) - c_2 \|\hat g\|_\cI^2$
where $c_1,c_2$ are from the oracle inequality of choice.  Since $\Psi_n(f,g_j)$ is also small, we can't just show $\|\hat g\|_\cI$ has the correct order of magnitude, which would have resulted from the definition of M estimator (and our current analysis).
}
As we shall see below, both of the above advantages for point estimators carry to our quasi-Bayesian setting. 

\subsection{Kernelized Dual IV}\label{sec:bg-kernelized-iv}
In this work, we are primarily interested in the case where $\cH$ is a reproducing kernel Hilbert spaces (RKHS), although in Section~\ref{sec:rf} we discuss heuristic application to DNNs. %
The choice of $\cI$ is more flexible, but 
a major example is when it is another RKHS. In this scenario, 
we derive an alternative closed-form expression for $\cL(f)$. %

Let $k_x$ and $k_z$ denote the reproducing kernels of $\cH$ and $\cI$, respectively.\footnote{
In the following, we will abuse notation and use $k$ to refer to both kernels, whenever the denotation is clear. 
} 
 Define the evaluation operator $S_z:\cI\rightarrow\RR^n, g\mapsto (g(z_1),\ldots,g(z_n))$, and $S_x:\cH\rightarrow\RR^n$ similarly. Let $\empCz := \frac{1}{n}S_z^*S_z, \empCzx := \frac{1}{n}S_z^*S_x$ 
be the empirical kernel (cross-)covariance operators. Using the Woodbury identity, we can verify that $\hat E_n = \invEmpCz \empCzx$ matches the definition \eqref{eq:Ehat-defn} (recall $\regEmpCz = \empCz+\effnu I$), and $\hat b$ admits a similar expression. Then 
\begin{align}%
\cL(f) &= \frac{1}{2}\left\|\regEmpCz^{-1/2} \left(\empCzx f - \frac{S_z^* Y}{n}\right)\right\|_\cI^2 + \frac{\efflam}{2} \|f\|_\cH^2 \nonumber \\ 
&= \efflam \left(\frac{1}{2} (f(X)-Y)^\top ((n\efflam)^{-1}L) (f(X)-Y) + \frac{1}{2}\|f\|_\cH^2\right), \label{eq:dualiv-loss-equiv}
\end{align}
where $L = \frac{1}{n}S_z\regEmpCz^{-1}S_z^*$ is a linear map from $\RR^n$ to $\RR^n$, and thus can be identified with a matrix. 
Similar derivations appear in previous work \citep[e.g.,]{singh_kernel_2020,muandet_dual_2020}. 
\eqref{eq:dualiv-loss-equiv} will be used to derive a closed-form expression of the quasi-posterior. %

\section{Quasi-Bayesian Analysis of Dual IV}\label{sec:qp}

For any point estimator minimizing the general objective \eqref{eq:iv-obj-general}, we can conduct quasi-Bayesian analysis by specifying a prior $\Pi(df)$ for $f$ and computing the \emph{quasi-posterior}, 
defined 
through the following Radon-Nikodym derivative with respect to the prior: %
\begin{equation}\label{eq:quasi-posterior-defn-general}
\frac{d\Pi(\cdot\givendata)}{d\Pi}(f) \propto 
\exp\!\left(
    -\frac{n}{\lambda}d_n^2(\hat{E}_n f-\hat{b})
\right),
\end{equation}
where $\lambda := n\efflam$ is the scaled regularization coefficient. 
The connection between the quasi-posterior and the point estimator can be seen from the following variational characterization \citep{zhang2006eps}:
\begin{equation}\label{eq:gibbs-justification}
\Pi(\cdot\givendata) = \argmin_{\Psi} \EE_{f\sim\Psi}[\lambda^{-1}n d_n^2(\hat{E}_n f-\hat{b})] + \KL{\Psi}{\Pi},%
\end{equation}
showing that the quasi-posterior trades %
off between the properly scaled \emph{evidence} (the log \emph{quasi-likelihood})
$
\frac{n}{\lambda} d_n^2(\hat{E}_n f-\hat{b})
$, 
which characterizes the estimated violation of the GMM constraint~\eqref{eq:iv-gmm-conds}, 
and our \emph{prior belief} $\Pi(df)$.

As discussed in Section~\ref{sec:setup}, this work focuses on the dual/minimax formulation \eqref{eq:dualiv-obj} which is an instance of \eqref{eq:iv-obj-general}, 
and we employ a RKHS regularizer for $f$. Thus, our quasi-posterior is defined by 
choosing $\Pi$ as the standard Gaussian process prior $\Pi := \mc{GP}(0,k_x)$, and instantiating \eqref{eq:quasi-posterior-defn-general} with $d_n(\hat E_n f - \hat b)$ defined as  
\begin{equation}\label{eq:quasi-loglh-defn}
d_n^2(\hat E_n f - \hat b) := \max_{g\in\cI}\frac{1}{n}\sum_{j=1}^n\!\left(
    (f(x_j)-y_j)g(z_j) - \frac{g(z_j)^2}{2}
    \right) - \frac{\effnu}{2}\|g\|_\cI^2.
\end{equation}

We now continue the kernelized IV discussion in Section~\ref{sec:bg-kernelized-iv}, and derive the closed-form expression for its quasi-posterior. 
Observe that in \eqref{eq:dualiv-loss-equiv}, the first term is equivalent to the log density of the multivariate normal distribution $\cN(Y\mid f(X), \lambda L^{-1})$, as a function(al) of $f$ and $Y$.\footnote{
Here we assume the invertibility of $L$ for brevity. Alternatively, %
observe that \eqref{eq:dualiv-loss-equiv} defines a linear inverse problem with the finite-dimensional observation operator $f\mapsto \sqrt{L} f(X)$ and noise variance $\lambda^{-1} I$. Now we can follow \citep[Chapter 6]{stuart2010inverse} to derive the same quasi-posterior.
}
Hence, \eqref{eq:dualiv-loss-equiv} can be viewed as %
the objective of a kernel ridge regression (KRR) problem with a data-dependent noise covariance $\lambda L^{-1}$. To derive the closed-form expression, %
we can recall the connection between KRR and Gaussian process regression \citep[GPR; see][]{kanagawa_2018_gaussian} and 
consider the \underline{fictitious} data generating process 
\begin{equation}\label{eq:dgp-fic}
    f\sim\mc{GP}(0,k), ~~ p_{\mrm{fic}}(Y\mid f,X,Z) := \cN(Y\mid f(X), \lambda L^{-1}),
\end{equation} 
since its conditional distribution $p_{\mrm{fic}}(df\mid Y)$ will coincide 
with \eqref{eq:quasi-posterior-defn-general} \cite[Chapter 6]{stuart2010inverse}. See Remark~\ref{rmk:qb-vs-bayes} below for further discussion.

In the probability space of this fictitious data generating process, for any finite set of test inputs $x_*$, we have
$$
p_{\mrm{fic}}(f(x_*),Y\mid X,Z)\sim \cN\left(0, \begin{bmatrix}
K_{**} & K_{*x} \\ 
K_{x*} & K_{xx}+\lambda L^{-1}
\end{bmatrix}
\right),
$$ where $K_{(\cdot)}$ denote the corresponding Gram matrices with subscript $*$ denoting $x_*$ and $x$ denoting $X$ (so, e.g., $K_{*x}:=k(x_*,X)$). 
Thus, by the Gaussian conditioning formula, the marginal posterior equals %
\begin{align}    
\Pi(f(x_*)\givendata) &= p_{\mrm{fic}}(f(x_*)\mid Y) = \cN(m, S), \label{eq:qbkdiv} \\
\text{where}\quad
    m &:= K_{*x}(\lambda I + L K_{xx})^{-1} L Y, \label{eq:post-mean-calc} \\ 
    S &:= K_{**}-K_{*x}L(\lambda I+K_{xx}L)^{-1} K_{x*}, \label{eq:post-cov-calc} \\ 
L &= K_{zz}(K_{zz}+\nu I)^{-1}. \label{eq:post-L}
\end{align}
In the above $K_{zz}:=k(Z,Z)$ denotes the Gram matrix, $\nu := n\effnu$, and \eqref{eq:post-L} follows by applying the Woodbury identity to the definition of $L$ in Section~\ref{sec:setup}. 

\begin{remark}[Interpretation of \eqref{eq:dgp-fic}, Quasi-Bayes vs Bayes]\label{rmk:qb-vs-bayes}
The fictitious data generating process \eqref{eq:dgp-fic} is merely introduced for computational convenience. As discussed in e.g., \citet{florens_nonparametric_2012}, 
care should be taken in its interpretation. 
For example, one should not deduce from \eqref{eq:dgp-fic} that the quasi-posterior is only useful when $L^{1/2}(Y-f(X))\sim\cN(0,\lambda I)$ (or when whatever infinite-sample version of this statement holds).  

From a ``subjective Bayesian'' 
point of view, 
in quasi-Bayesian analyses we do not have a full ``Bayesian belief'' of the data generating process. One quick way to 
see the difference is to observe that, even though $(x_i,y_i,z_i)$ are i.i.d.~draws from the joint data distribution by assumption, 
$p_{\mrm{fic}}(y_i\mid f,X,Z)$ depends on all of $Z$ by \eqref{eq:post-L}, instead of merely $z_i$ as in a fully Bayesian model. This distinction arises because $\hat E_n$ is estimated from data. 
Nonetheless, \citet{florens_gaussian_2021} shows that the quasi-posterior can be obtained as a certain non-informative limit of a fully Bayesian posterior; but in the context of IV regression, construction of the latter appears unrealistic. 

More generally, quasi-Bayesian methods should be justified by the variational characterization \eqref{eq:gibbs-justification}, and analyses of their frequentist properties as in the following. We stress that none of the analyses below relies on $Y$ being distributed as in \eqref{eq:dgp-fic}. 
\end{remark}

\section{Asymptotic Analysis}\label{sec:theory}

In this section, we demonstrate that the quasi-posterior provides a useful notion of uncertainty, by establishing the following results: 
\begin{enumerate}
    \item Minimax optimal \emph{contraction rate} (Theorem~\ref{thm:L2}, Corollary~\ref{thm:contraction-alt-norm}), showing that the region of high posterior probability rapidly shrinks to the true function $f_0$. 
    \item (Asymptotic near-)validity of the \emph{credible balls} (Theorem~\ref{thm:credible-ball}), showing that the spread of the quasi-posterior is informative of the error of the posterior mean predictor. 
\end{enumerate}
These asymptotic results hold under $L_2$ and Sobolev norms, thus covering the recovery of both $f_0$ and its (higher-order) derivatives. %

\subsection{Assumptions in the Analysis}\label{sec:assumptions}

We first introduce our assumptions. As we discuss below, all assumptions come from previous work on nonparametric IV or Bayesian inverse problems. 
We will also provide an example where the assumptions hold in Example~\ref{rmk:assumptions} below.

\parspace
Assumption~\ref{ass:std} defines the problem. Note we assume bounded residual for simplicity; for general subgaussian residuals, the contraction rate will change by a logarithm factor. 
\begin{assumption}\label{ass:std}
The observed data $\data := \{(x_i,y_i,z_i): i\in[n]\}$ are $n$ i.i.d.~draws from the joint distribution $P(dx\times dy\times dz)$. The data distribution satisfies \eqref{eq:iv-gmm-conds}. 
The residual $\by-f_0(\bx)$ is always bounded by $B$. 
\end{assumption}

Assumption~\ref{ass:k-std} is standard. It is used to ensure that \emph{Mercer's representation} exists and constitutes an orthonormal basis (ONB) of $L_2(P(dx))$: 
consider the integral operator $
T: L_2(P(dx))\rightarrow L_2(P(dx)), f \mapsto \int f(x)k(x,\cdot)P(dx)$. 
Under Assumption~\ref{ass:k-std}, 
by \citet[Lemma 2.3, 2.12, Theorem 3.1, Corollary 3.5]{steinwart_mercers_2012},\SkipNOTE{we use the (i) <=> (ii) clause in Thm 3.1}
$T$ has an eigendecomposition $\{(\lambda_i,\bar\varphi_i)\}$, which defines a Mercer's representation, and $\{\bar\varphi_i\}$ form an ONB. 
\begin{assumption}\label{ass:k-std}
(i) $\cX$ and $\cZ$ are Polish spaces. (ii) The kernel $k_x$ is measurable, continuous, bounded and $L_2$-universal \citep{sriperumbudur2011universality}\SkipNOTE{i.e., $\cH$ dense in $L_2$}. (iii)
$P(dx)$ has full support.
\end{assumption}

Assumption~\ref{ass:H1} is taken from the literature on kernel ridge regression in the ``hard learning'' scenario \citep{fischer2020sobolev,steinwart2009optimal}. Before introducing it, we need to introduce the notion of \emph{power space}: 
\begin{definition}[\citealp{steinwart_mercers_2012}, p.~384]\label{defn:power-space}
Let $\{(\lambda_j,\bar\varphi_j):j\in\mb{N}\}$ defined as above. For $\alpha\in(0,1]$, define the \emph{power space} $[\cH]^\alpha$ using the norm 
\begin{equation*}%
\|f\|_{[\cH]^\alpha}^2 := \sum_{j=1}^\infty \lambda_i^{-\alpha} \<f, \bar\varphi_j\>_2^2.
\end{equation*}
\end{definition}
The inner product of $[\cH]^\alpha$ is defined similarly; 
see Appendix~\ref{app:power-space} for further discussion. 

Now we introduce Assumption~\ref{ass:H1}. 
Its part (i) specifies a polynomial eigendecay for $\cH$, which is standard for \emph{mildly ill-posed problems} \citep{cavalier2008nonparametric}. 
Part (ii) is the ``embedding property'' in \cite{fischer2020sobolev}. It is used to address the mismatch of the regularity of the RKHS $\cH$ and that of prior samples \citep{van_der_vaart_information_2011}, 
which makes GPR fall into the ``hard learning'' scenario. %
For Mat\'ern kernels, (ii) holds for all $b_{emb}>1$ under mild conditions listed in Section~\ref{sec:power-space-contraction}. 

As any RKHS $\cH$ can be continuously embedded into its power space $[\cH]^\alpha$, for $\alpha<1$, condition (ii) is weaker for larger $b_{emb}$. 
We merely require the existence of a $b_{emb}$, which can be arbitrarily close to $b$. 
$b_{emb}=b$ would be satisfied as long as $f_0$ is in an RKHS $\bar\cH$ with a bounded kernel, since after that we can choose $\cH := [\bar\cH]^{\frac{b+1}{b}}$; see Lemma~\ref{lem:power-space}. Requiring the existence of $b_{emb}=b-\epsilon<b$ is only ``infinitesimally stronger'', as it is satisfied if 
$
\sup_{x\in\cX} \sum_{i=1}^\infty i^{b-\epsilon} \bar\varphi_i^2(x)<\infty 
$ \citep[Theorem 5.3]{steinwart_mercers_2012}, 
whereas 
requiring $\bar\cH$ to have a bounded kernel implies 
$
\sup_{x\in\cX} \sum_{i=1}^\infty i^{b} \bar\varphi_i^2(x)<\infty.
$ 

\begin{assumption}\label{ass:H1}
\emph{(i)}
$
 \lambda_j\asymp j^{-(b+1)}
$ for some $b>1$, \emph{and (ii)} for some $b_{emb}<b$,  
the power space $[\cH]^{b_{emb}/(b+1)}$ can be continuously embedded into $L_\infty(P(dx))$. %
\end{assumption}

Assumption~\ref{ass:approx-relaxed} requires $f_0$ to be well approximated in $\cH$. 
The first two inequalities in (b) is equivalent to bounding an $L_2$ \emph{concentration function} of $f_0$, 
and is weaker than the sup-norm requirement in previous work on GPR \citep{van_der_vaart_rates_2008,van_der_vaart_information_2011}. For the sup norm, we place the very basic requirement that 
approximation errors do not diverge. 

(a) %
is a more intuitive condition and implies (b), as shown in  
Lemma~\ref{lem:approx-validation}. %
It requires $f_0\in\bar\cH$. 
The power space $\bar\cH$ is larger than $\cH$, but also has an RKHS structure; see Appendix~\ref{app:power-space}. It often has a clear interpretation: 
for example, when $\cH$ is a Sobolev space (e.g., a Mat\'ern RKHS), 
$\bar\cH$ will be a lower-order Sobolev space 
under mild conditions.  
(a) implies that our prior knowledge $f_0\in\bar\cH$ should be matched by using a more regular RKHS $\cH\subset \bar\cH$. This is a standard practice, and arises from the mismatch between the regularity of $\cH$ and that of GP prior samples \citep{van_der_vaart_information_2011}. 

Note that throughout this work we only state assumptions in the rate-optimal cases. 
Suboptimal specifications lead to deteriorated rates, but consistency can still be possible. %
The conference version of this work \citep{wang2021scalable} also established consistency under fewer assumptions. 
\begin{assumption}\label{ass:approx-relaxed}
\SkipNOTE{removed the sup norm bound which follows from an interpolating ineq}
Define $$\bar\cH := [\cH]^{\frac{b}{b+1}}.$$
We require that one of the following holds: {\em (a)} $f_0\in\bar\cH$, %
{\em Or (b)} 
For all $j\in\mb{N}$, there exists $f_j^\dagger\in\cH$ such that 
\begin{equation}\label{eq:approx-relaxed-orig}
\|f_j^\dagger\|_\cH^2 \lesssim j^{\frac{1}{b+1}}, ~~ 
\|f_0-f_j^\dagger\|_2^2 \lesssim j^{-\frac{b}{b+1}}, ~~
\|f_0-f_j^\dagger\|^2 \lesssim 1.
\end{equation}
\end{assumption}

\paragraph{}

Assumption~\ref{ass:ill-posed} restricts 
to the {mildly ill-posed} setting, thus matching Assumption~\ref{ass:H1}~(i). The \emph{severely ill-posed} setting is often treated separately; it should be matched by kernels with an exponential eigendecay.

\begin{assumption}\label{ass:ill-posed}
The conditional expectation operator $E:L_2(P(dx))\rightarrow L_2(P(dz))$ has singular values $\nu_j\asymp j^{-p}$, where $p\ge 0$ is a constant. 
\end{assumption}

Assumption~\ref{ass:link} imposes the 
\emph{link and reverse link conditions} on the orthonormal series $\{\bar\varphi_i\}$. %
Variants of these conditions appear in all previous work on optimal $L_2$ rate for nonparametric IV \citep[e.g.,][]{horowitz_specification_2012,chen_estimation_2012,kato_quasi-bayesian_2013},\footnote{
While it does not appear in recent ML works on IV, they did not establish convergence rates for the recovery of nonparametric $f_0$, and only treated $Ef_0$.  The rates for $Ef_0$ are also suboptimal; 
see Section~\ref{sec:related-work}. 
}  
and is considered relatively weak in Bayesian inverse problem \citep{agapiou2013posterior,knapik_general_2018}.
A stronger assumption, which is also common in the Bayesian literature \citep{knapik_bayesian_2011,szabo_frequentist_2015}, requires that $\{\bar\varphi_i\}$ diagonalize the operator $E^* E$. 

Note that \eqref{eq:reverse-link-cond} implies injectivity of $E$, which implies identification in $L_2(P(dx))$. Without assuming identifiability, it is still possible to derive suboptimal rates in the semi-norm $\|E(\cdot)\|_2$; see the conference version of this work. 

\begin{assumption}\label{ass:link}
\phantom{.}%
Let $\bar\varphi_i$ be defined as above. 
For any $j\in\mb{N}$, denote by $\mrm{Proj}_j$ the $L_2$ projection onto $\mrm{span}\{\bar\varphi_i: i\le j\}$. 
Then we have, for all $f\in L_2(P(dx))$ and $j\in\mb{N}$, 
\begin{align}
\|Ef\|_2^2\gtrsim j^{-2p} \|\mrm{Proj}_j f\|_2^2 \tag{RL},  \label{eq:reverse-link-cond}\\ 
\|Ef\|_2^2\lesssim \sum_{j=1}^\infty j^{-2p}\<f,\bar\varphi_i\>_2^2 \tag{L}. \label{eq:link-cond}
\end{align}
\end{assumption}

Assumption~\ref{ass:I-approx}
is used to control $\|\hat E_n f -Ef\|_2$, and subsequently the estimation error of the quasi-likelihood, for ``typical'' $f$ sampled from the GP prior. We specify $\delta_n^2$ so that under Assumptions~\ref{ass:H1} and \ref{ass:ill-posed}, the estimation error can be optimal. 
This requirement of optimal error rate 
is intuitive and appears in all previous work with optimal rates. %

Part (i) follows \cite{dikkala_minimax_2020} and controls the variance part of the error with local Rademacher complexity. It can be fulfilled by an entropy integral bound 
(\citealp{wainwright2019high}, Corollary 14.3); By standard entropy bounds \citep[e.g.,][Chapter 4]{gine2021mathematical}, when $\cI_1$ is contained in an $L_2$-Sobolev or H\"older ball, (i) merely requires the $L_2$ estimation error (using the Sobolev or H\"older ball) be $O(\delta_n)$. 
Part 
(ii) controls the bias, and matches the assumption in most recent work \citep{muandet_dual_2020,liao_provably_2020}; 
we consider the restriction of $E$ to $\bar\cH$ instead of $\cH$, 
as Assumption~\ref{ass:approx-relaxed} matches the regularity of $f_0$ with $\bar\cH$. 
(ii) can be relaxed to allow an $L_2$ approximation error of $O(\delta_n)$, which enables the use of e.g.~Nystr\"om approximation for $\hat E_n$; see Appendix~\ref{app:proof-nystrom}. 

\begin{assumption}\label{ass:I-approx} 
Denote by $\cI_1$ the unit norm ball of $\cI$. 
Let $\delta_n$ be the critical radius of the local Rademacher complexity %
of $\cI_1$ \citep[Eqs.~(14.4)]{wainwright2019high}. 
Then 
{\em (i)} functions in $\cI_1$ are uniformly bounded, and $\delta_n^2\lesssim n^{-\frac{b+2p}{b+2p+1}}$, 
{\em and (ii)} 
the restriction of $E$ to $\bar\cH$ has image contained in $\cI$, and is bounded, i.e., $\|E f\|_\cI\lesssim \|f\|_{\bar\cH}$ for all $f\in\bar\cH$. 
\end{assumption}

We conclude this subsection with the following example, adapted from \citet{horowitz_applied_2011}. %

\begin{example}\label{rmk:assumptions}%
Let $\cX=\cZ$ be the circle $\mb{S}^1=\mb{T}^1$, the joint distribution of $X$ and $Z$ be absolutely continuous w.r.t.~the Hausdorff measure, and the corresponding density $f_{XZ}$ be bounded on both sides. 
Let $\bar\cH$ be in the Sobolev space $W^{b/2,2}(\mb{T}^1)=
\{f: \sum_{i=1}^\infty \<f, \bar\varphi_i\>_2^2 i^b <\infty
\}$ where $\{\bar\varphi_i\}$ denotes the sinusoidal basis \citep{triebel_theory_1983,strichartz_analysis_1983}. 
It then follows from \cite{steinwart_mercers_2012}\SkipNOTE{Lemma 2.6} that $\bar\cH$ 
is an RKHS with the (bounded) kernel having the Mercer representation $
\bar k_x(x,x') = \sum_{i=1}^\infty i^{-b} \bar\varphi_i(x)\bar\varphi_i(x')$.\footnote{
    In fact $\bar k_x$ is the Mat\'ern kernel, and closed-form expressions are provided in \citet{borovitskiy_matern_2021}. 
} 
Suppose $f_0\in\bar\cH$, and define $\cH := W^{(b+1)/2,2}, \cI:=W^{(b+2p)/2,2}$ which have similar RKHS structures. 
Then Assumptions \ref{ass:k-std},~\ref{ass:H1},~\ref{ass:approx-relaxed}\SkipNOTE{~(a)} are satisfied. \ref{ass:I-approx}~(i) follows from the entropy bound for periodic Sobolev spaces \citep[e.g., Corollary 4.3.38,][]{gine2021mathematical}. 
The remaining assumptions 
can be satisfied if we construct $E$ so that its left and right singular vectors match the Mercer basis, as in Section~\ref{sec:exp-asymp-val}. 
We refer to \citet[Section 6.1]{chen2011rate} for another example where the singular vectors do not have to match $\{\bar\varphi_i\}$. %

The construction generalizes easily to higher dimensional toruses. 
As discussed above, the assumptions on the kernels can also be satisfied if we work with $[0,1]^d$ as opposed to the torus $\mb{T}^d$, and use the corresponding Mat\'ern kernels, 
but explicit construction of valid $E$ is not as straightforward, since we no longer have $\EE\bar\varphi_{i+1}\equiv 0$. Nonetheless, it is known that 
$p$ is lower bounded by the differentiability of the joint density: if $f_{XZ}$ has continuous mixed derivatives up to order $r$, we will have $p\ge r$ \citep[p.~365]{horowitz_applied_2011}. 
\end{example}

\subsection{Contraction Rates}\label{sec:contraction}

\subsubsection{$L_2$ contraction rates}

Theorem~\ref{thm:L2} below is our main result. \eqref{eq:L2-rate-maintext} establishes an $L_2$ contraction rate of $\epsilon_n$ for the recovery of $f_0$, and 
\eqref{eq:L2-rate-CE-maintext} is a contraction rate in the norm $\|E(\cdot)\|_2$. 

\begin{theorem}%
    \label{thm:L2}
Fix $
\lambda=1, \effnu = C \delta_n^2
$ where $C>0$ is an universal constant determined by the assumptions. 
Under the assumptions in Section~\ref{sec:assumptions}, 
there exists $M>0$ such that for $\epsilon_n^2 = n^{-\frac{b}{b+2p+1}}$, we have
\begin{align}
\lim_{n\to\infty} \PP_{\data} \Pi\!\left(
    \{f: \|f-f_0\|_2\ge M \epsilon_n\}
    \givendata
\right) &= 0,\label{eq:L2-rate-maintext} \\ 
\lim_{n\to\infty} \PP_{\data} \Pi\!\left(
    \{f: \|E(f-f_0)\|_2\ge M \delta_n\}
    \givendata
\right) &= 0. \label{eq:L2-rate-CE-maintext}
\end{align}
In addition, if Assumption~\ref{ass:std} is relaxed to allow general subgaussian residuals $\by-f_0(\bx)$, the above equations will hold with $\epsilon_n$ and $\delta_n$ multiplied by $\log n$. 
\end{theorem}

The proof is deferred to Appendix~\ref{app:proof-thm-L2}. 
It is based on the posterior contraction framework \citep{ghosal2000convergence}, and shares similarity with 
the analysis of Bayesian inverse problem under the link condition  \citep{knapik_general_2018}. 
While IV is also a linear inverse problem, it is more challenging as we need to estimate the ``forward operator'' $E$ from data. 
Thus, 
contrary to \cite{knapik_general_2018}, in which the rate for $Ef$ can be established easily and the rate for $f$ is derived from it, 
in our case \eqref{eq:L2-rate-CE-maintext} and \eqref{eq:L2-rate-maintext} have to be proved simultaneously, since they require to control the same estimation error about $E$. 
For this purpose, we combine the argument in \citep{dikkala_minimax_2020} %
with the interpolation space arguments in \citep{steinwart2009optimal}. 

Under our assumptions,  
the $L_2$ rate $\epsilon_n$ is minimax optimal \citep[Corollary 1]{chen2011rate}.\footnote{\citep{chen2011rate} allows for additional assumptions, %
and thus applies to our setting. Such additional assumptions are common in the analysis of concrete estimators, as opposed to establishing lower bounds.
} 
The rate for $Ef$ also matches the minimax rate for least square regression on the dataset $\{(y_i,z_i): i\in[n]\}$, although its optimality is less clear due to the additionally observed $X$. 
The contraction rates also imply high-probability convergence rates for the posterior mean estimator, or $L_2$ rates (for $f$ and $Ef$) if we truncate prediction with a large constant, as in e.g., \citet{steinwart2009optimal}. 
We compare with previous work in Section~\ref{sec:related-work}. 

Finally, observe that under Assumptions \ref{ass:H1},~\ref{ass:ill-posed},~\ref{ass:I-approx}, 
our choice for the regularizer $\effnu$ is asymptotically optimal for the estimation of $\EE(f(\bx)\mid \bz)$, for functions $f$ with the same regularity as the GP prior (e.g., $f\in\bar\cH$). 
This suggests that this choice may be identified from data using an algorithm in previous work (see \citealp{singh_kernel_2020}, or our Appendix~\ref{app:hps-sel}). %

\subsubsection{Sobolev norm contraction rates}\label{sec:power-space-contraction}
Suppose 
$\cX$ is a bounded open subset of $\mb{R}^d$ with a smooth boundary, $P(dx)$ has Lebesgue density bounded on both sides, and  
$\cH$ is equivalent to a Sobolev space $W^{\frac{(b+1)d}{2},2}$ (e.g., when using the Mat\'ern kernel).\footnote{
The order $\frac{(b+1)d}{2}$ is determined by the eigendecay assumption; see also \citep{kanagawa_2018_gaussian}. 
}
Then for all $b'\in (0,b+1)$, 
$[\cH]^{\frac{b'}{b+1}}$ is norm-equivalent to the Sobolev space $W^{\frac{b'd}{2},2}$; see 
\citet[Sections 7.33, 7.67]{adams2003sobolev}; \citet[Section 4]{fischer2020sobolev}. 
Therefore, we can derive Sobolev norm rates by investigating the power space, for which we have the following:

\begin{corollary}%
\label{thm:contraction-alt-norm}
Let all assumptions in Section~\ref{sec:assumptions} hold; in particular, for 
Assumption \ref{ass:approx-relaxed} let (a) hold. 
Set $\efflam,\effnu$ as in Theorem~\ref{thm:L2}. Then for some sufficiently large $M>0$, we have, for all $0\le b'<b$ and $\epsilon_{b',n}^2 := n^{-\frac{b-b'}{b+2p+1}}$, 
$$
\PP_{\data}(\Pi(\{f: \|f-f_0\|_{[\cH]^{\frac{b'}{b+1}}}\ge M\epsilon_{b',n}\}\givendata)) \rightarrow 0.
$$
\end{corollary}

The proof is in %
Appendix~\ref{app:proof-alt-norm}. 
In the above setting, 
the rate $\epsilon_{b',n}$ is %
optimal for estimating integer-order derivatives, i.e., when %
$bd/2\in\mb{N}$ (\citealp{chen_optimal_2015}, Theorem 2.4). 

\subsubsection{Side result: improved rate for Gaussian process regression} 
Our proof technique can be applied to GPR for unconfounded regression, i.e., when $\EE(\by-f_0(\bx)\mid\bx)=0$. In this case, we improve the result in \cite{van_der_vaart_information_2011} for Mat\'ern priors, by relaxing their assumption of joint  
H\"older and Sobolev regularity to Sobolev regularity alone. 
More generally, our Assumption~\ref{ass:approx-relaxed} relaxes their assumption (in the rate-optimal case), by requiring approximability in $L_2$ norm instead of sup norm. 
The relaxation makes it clear that the result is minimax optimal. 
We state the result in power space norms, thus covering both 
$L_2$ and Sobolev spaces:

\begin{theorem}%
\label{thm:GPR-L2}
Let %
$\cH$ 
satisfy Assumption~\ref{ass:H1}, and $f_0\in\bar\cH$. Let $\Pi(df\mid \data)$ be the Gaussian process posterior under a normal likelihood. Suppose the residual $\be := \by-f_0(\bx)$ is bounded by $B$ and $\EE(\be\mid \bx)=0$.\footnote{
As in the IV setting, we assume bounded noise for simplicity, and general subgaussian noise can be treated by inflating $B$ (and thus the final rate) by a logarithm factor. We specify normal likelihood for the GP for simplicity. It is also a standard practice as it makes computation easy. 
} 
Then for sufficiently large $M>0$, we have, for all $b'\in [0,b)$ and $\epsilon_{b',n}^2=n^{-\frac{b-b'}{b+1}}$, 
$$
\PP_\data \Pi\!\left(\{f: \|f-f_0\|_{[\cH]^{\frac{b'}{b+1}}}\ge M\epsilon_{b',n}\}\givendata\right) \rightarrow 0.
$$
\end{theorem}

The proof is made possible by a refined bound on the marginal likelihood;  
see Appendix~\ref{app:proof-GPR}.  
Note that \cite{yang_frequentist_2017} also establishes minimax rate under the same assumptions, but for a different prior scaled by a data-dependent factor. Our result for the unscaled prior appears new.

\subsection{Frequentist Properties of the Credible Balls}

\newcommand{\bayesRad}[1][\cdot]{\tilde r_n(\gamma,#1)}
\newcommand{\freqRad}[1][\cdot]{\hat r_n(\gamma,#1)}
\newcommand{\optRate}{{\epsilon^{(\cdot)}_n}}
\newcommand{\powerSpace}{{[\cH]^\alpha}}
\newcommand{\powerRate}{\epsilon_{b',n}}

We now study the frequentist validity of the quasi-Bayesian \emph{credible ball}s, 
$
B(\hat f_n, \bayesRad) := \{f: \|f-\hat f_n\|_{(\cdot)}\le \bayesRad\},
$ where $\hat f_n$ denotes the posterior mean function, $(\cdot)$ denotes an arbitrary norm, $\gamma\in(0,1)$ is the credibility level, and
$$
\bayesRad := \inf\!\left\{r_n(\data): 
\forall \|f_0\|_{\bar\cH} \le 1,
\PP_{\data}\!\left(\Pi(\|f-\hat f_n\|_{(\cdot)}\le r_n(\data) \givendata)\ge \gamma\right)=1\right\}
$$ 
is its radius. We will compare $\bayesRad$ with the radius of the frequentist \emph{confidence ball}
$$
\freqRad := \inf\{r_n: \forall \|f_0\|_{\bar\cH}\le 1, \PP_{\data}(\|f_0 - \hat f_n\|_{(\cdot)} \le r_n)\ge \gamma\}.
$$
Note that $\tilde r_n$ is a random variable determined by $\data$, since in its definition the inequality must hold with $\PP_{\data}$ probability 1; in contrast, $\hat r_n$ is deterministic because confidence balls can fail with probability $1-\gamma$, which could account for failures in estimating $\hat E_n$. 

We shall show that for the power space norms which include the $L_2$ norm, the radius of the credible balls will eventually have the correct order of magnitude. %
Formally, for any fixed $\gamma\in(0,1)$, $\alpha=\frac{b'}{b+1}\in [0,\frac{b}{b+1})$, 
\begin{equation}\label{eq:credible-ball-goal}
\lim_{n\rightarrow\infty}\PP_{\data}\!\left(\bayesRad[{[\cH]^\alpha}] \asymp \freqRad[{[\cH]^\alpha}]\right) = 1.
\end{equation}
Similar results (for other problems) are provided in \cite{knapik_bayesian_2011,szabo_frequentist_2015,rousseau_asymptotic_2017}. 
From a theoretical point of view, asymptotically valid confidence balls can then be constructed by inflating the radius of the credible ball, $\tilde r_n$, with a constant, or a slowly growing function of $n$ \citep{rousseau_asymptotic_2017}, although the inflation often appears unnecessary in practice. 
A more precise characterization (e.g., showing that $\frac{\tilde r_n(\gamma,\cdot)}{\hat r_n(\gamma,\cdot)}\approx 1$) %
is often difficult for general nonparametric models. 

For the frequentist confidence ball, 
Corollary~\ref{thm:contraction-alt-norm} implies 
$\|\hat f_n - f_0\|_{\powerSpace} \le M\powerRate$ with $\PP_{\data}$ probability $\rightarrow 1$, 
\SkipNOTE{
    The above claim follows from the claim that, $\|\hat f_n - f_0\|_2^2 > M\epsilon_n^2$ implies $\Pi({\|f-f_0\|_2^2 \ge M\epsilon_n^2} \givendata) > 1/2$.
    The latter claim holds because the posterior can be generated as $\hat f_n + \delta f$, where $\delta f$ is a zero-mean GP, and thus $\Pi(\<\delta f, f_0-\hat f_n\>_2>0 \givendata) = 1/2$. 
}
and thus for any fixed $\gamma$, $\hat r_n(\gamma, \powerSpace)\lesssim \powerRate$; on the other hand, the minimax lower bound implies that %
$\hat r_n(\gamma, \powerSpace)\gtrsim \powerRate$. Thus $$
\hat r_n(\gamma,\powerSpace)\asymp \powerRate. 
$$
\SkipNOTE{
    asymp implies $n$ sufficiently large
}
For the credible ball, Theorem~\ref{thm:L2} implies that for any fixed $\gamma<1$, we have 
$$
\PP_{\data}(\bayesRad[\powerSpace] \le M\powerRate)\rightarrow 1.
$$
Therefore, to prove \eqref{eq:credible-ball-goal}, it suffices to show that $\bayesRad[\powerSpace]\gtrsim \powerRate$ with probability $\rightarrow 1$. We prove this using the strategy in \cite{rousseau_asymptotic_2017}:

\begin{theorem}%
\label{thm:credible-ball}
Under the conditions in Corollary~\ref{thm:contraction-alt-norm}, %
there exists $\rho>0$ such that as $n\to\infty$, 
\begin{align*}
    \PP_{\data}(\Pi(\{\|f-\hat f_n\|_{\powerSpace} \le \rho \powerRate\}\givendata)) &\rightarrow 0.
\end{align*}
Consequently, \eqref{eq:credible-ball-goal} holds.
\end{theorem}
The proof is in Appendix~\ref{app:proof-cred-ball}. 

\section{Discussion of the Non-asymptotic Behavior}\label{sec:nonasymptotic}

The asymptotic analysis above relies on various assumptions on the interaction between the kernel and the data distribution. While they are standard in literature, they cannot account for potential misspecification. The analysis also focused on the recovery of the \emph{full parameter} $f$, which does not provide optimal guarantees for the \emph{evaluation of linear functionals} (e.g., evaluation functionals $f\mapsto f(x_0)$; see \citealp{knapik_bayesian_2011}). 
In this section, we complement the results above, by providing some intuition on the non-asymptotic behavior of the closed-form quasi-posterior \eqref{eq:qbkdiv}, without any of the assumptions introduced above. 

Fix a bounded linear functional in $\cH$ with Riesz representer $k_*\in\cH$, and 
consider the marginal posterior of $\<k_*,f\>_\cH$.\footnote{
While the posterior assigns zero mass on $\cH$, it is known that the above marginal posterior is well-defined \citep{van_der_vaart_reproducing_2008,eldredge_analysis_2016}.
}
Repeating the derivation in Section~\ref{sec:qp}, we can see that 
$
\Pi(\<k_*,f\>_{\cH}\givendata) = \cN(m, S), 
$
where $m$ and $S$ are defined as in Eq.s \eqref{eq:post-mean-calc}-\eqref{eq:post-cov-calc}, with $K_{*X}$ replaced by $(S_x k_*)^\top$.\footnote{%
     Observe that $(S_x k_*)^\top =  (\<k_*,k(x_1,\cdot)\>_\cH,\ldots,\<k_*,k(x_n,\cdot)\>_\cH)$. 
} Following their definitions, we can reformulate the equality as 
$$
\Pi(\<k_*,f\>_{\cH}\givendata) = \cN(m,S) = \cN(\<k_*,\hat f_n\>_\cH, \<k_*, \cC k_*\>_\cH), 
$$
where we can verify 
$\hat f_n$ matches the posterior mean. The expression of $\cC$ can be found in Eq.~\eqref{eq:equiv-cov-form}. 

To check if the quasi-posterior quantifies the uncertainty in the estimation of linear functional evaluations, it thus suffices to compare the error of the posterior mean predictor, 
$\<k_*,\hat f_n-f_0\>_{\cH}^2$, with the marginal variance $\<k_*, \cC k_*\>_\cH$. With some algebra (see Appendix~\ref{sec:gp-cov-unconfounded-proof} for the proof), we can show that: 
\begin{proposition}\label{prop:gp-covariance-unconfounded-dgp}
Let the data $\data$ be generated from 
an \underline{unconfounded} data generating process: $y_i := f(x_i) + e_i$, where $\{e_i\}$ are i.i.d.~random variables with zero mean and variance $\lambda$, and $e_i\perp\!\!\!\!\perp\{x_i,z_i\}$. 
Then for any $k_*\in\cH$ and any choice of $(n,X,Z)$, we have %
\begin{align}
\<k_{*}, \cC k_*\>_\cH &= 
\sup_{\|f_0\|_{\cH}=1} \EE_{Y\mid X,Z}\<k_*, \hat{f}_n-f_0\>_{\cH}^2 + 
\<k_*, (\Delta\cC) k_*\>_\cH \label{eq:cov-as-worst-case-error} \\ 
&= \EE_{f\sim\mc{GP}(0,k_x)} \EE_{Y\mid X,Z}\<k_*, \hat{f}_n-f_0\>_{\cH}^2 + 
\<k_*, (\Delta\cC) k_*\>_\cH, \label{eq:cov-as-avg-error}
\end{align}
where $\Delta\cC$ is a non-negative operator defined in the proof.
\end{proposition}

The above proposition %
is inspired by similar results for GPR (\citealp{kanagawa_2018_gaussian}, Proposition 3.8; \citealp{srinivas_gaussian_2012}). 
The first result \eqref{eq:cov-as-worst-case-error} connects the marginal variance to the worst-case prediction error,\footnote{
It applies to the worst-case $f_0$. 
We can remove the average over $P(Y\mid X,Z)$ with the Markov inequality; this is less problematic %
as $Y$ is finite dimensional.
} but only applies to $f_0\in\cH$. This assumption is different from the optimal choice when recovering the full parameter, where we match the regularity of $f_0$ with the GP prior; see the discussion around Assumption~\ref{ass:approx-relaxed} and the reference therein. However, 
this trade-off is well-known in %
literature \citep{knapik_bayesian_2011,srinivas_gaussian_2012}, %
and is not unique to (quasi-)Bayesian methods. Instead, it 
stems from the infinite dimensionality of the prior. While it can be avoided for %
sufficiently smooth functionals, evaluation functionals usually fall out of this category \citep{knapik_bayesian_2011}. 

Nonetheless, %
Eq.~\eqref{eq:cov-as-avg-error} shows that
the marginal variance can be useful 
in the ``full parameter regime'', %
by relating it to the average prediction error. This type of result also appears in the bandit literature \citep{srinivas_gaussian_2012}. 
While there are caveats around its interpretation, 
e.g., sets of misspecified functions can have zero prior measure \citep{ghosal2017fundamentals}, it may explain the good coverage of credible intervals observed in practice. 

The main limitation of Proposition~\ref{prop:approx-inf} is that it only applies to the unconfounded case. 
In the general case, it does not constitute a complete and rigorous analysis, and only hopefully provides some intuition, by accounting for the approximation error within the model, as well as part of the estimation error due to the sampling of %
$X$ and $Z$. 
We restrict to this setting so that the prediction error is available in closed form. Still, the 
additional term $\<k_*, \Delta\cC k_*\>_\cH$ is non-negative and might compensate for some additional sources of error,  %
see Appendix~\ref{app:DeltaC} for an informal discussion. 
Also, the main appeal of the discussion in this subsection is that it does not place any assumption on the choice of the models (or hyperparameters); for example, it does not require $\cI$ to contain the image of $E$ restricted on $\bar\cH$. 
Therefore, it suggests that 
the quasi-Bayesian uncertainty may still reflect a nontrivial part of prediction error, even if the conditional expectation estimator %
is badly misspecified.

\section{Approximate Inference via a Randomized Prior Trick}\label{sec:rf}

We now turn to approximate inference with parametric models, such as random feature models. While other approaches are also possible (e.g., using Nystr\"om approximation), this approach has the advantage of being able to extend to wide NN models. 
We construct an algorithm by 
extending the ``randomized prior'' trick for GPR \citep{osband2018randomized}, to work with (quasi-)likelihoods with an optimization formulation as in \eqref{eq:dualiv-obj}. 

\subsection{Approximate Inference with Random Feature Approximation}\label{sec:rf-rf}

We first consider approximate inference with random feature models. 
Introduce the random feature approximation $k_z(z,z')\approx \tilde{k}_{z,m}(z,z'):=\frac{1}{m}\phi_{z,m}(z)^\top \phi_{z,m}(z')$, where $\phi_{z,m}$ takes value in $\RR^m$. Then the map 
$\varphi\mapsto \frac{1}{\sqrt{m}}\varphi^\top\phi_{z,m}(\cdot)=:g(\cdot;\varphi)$ parameterizes an approximate RKHS $\tilde{\cI}$; and for all $c>0$, 
the random function $g(\cdot;\varphi)$, where $\varphi\sim\mc{N}(0,c I)$, is distributed as  
$\mc{GP}(0, c \tilde k_{z, m})$. 
The notations related to $k_x$ are similar and thus omitted. 
Now we can state the objective: 
\begin{proposition}\label{prop:randomized-prior}
Let $\phi_0\sim\cN(0,\lambda\nu^{-1} I), \theta_0\sim\cN(0, I)$, 
$\tilde{y_i}\sim\cN(y_i,\lambda)$. 
Then the optima $\theta^*$ of
\begin{align}\label{eq:rf-obj}
\min_{\theta\in\RR^m}\max_{\phi\in\RR^m}& \sum_{i=1}^n \left((f(x_i;\theta)-\tilde{y}_i)g(z_i;\phi) - \frac{g(z_i;\phi)^2}{2}\right) - \frac{\nu}{2}\|\phi-\phi_0\|_2^2 + \frac{\lambda}{2}\|\theta-\theta_0\|_2^2
\end{align}
parameterizes a random function which 
follows the quasi-posterior distribution \eqref{eq:qbkdiv}, where the kernels are replaced by the random feature approximations. 
\end{proposition}

The proof is in Appendix~\ref{app:dual-algo-deriv}. 
Given the above proposition, we can sample from the random feature-approximated quasi-posterior by solving \eqref{eq:rf-obj} with stochastic GDA; the approximation errors will be analyzed in the following.  

The objective \eqref{eq:rf-obj} is closely related to \eqref{eq:dualiv-obj}: 
Appendix~\ref{app:rf-props-equiv} shows it is equivalent to 
\begin{equation}\label{eq:rf-obj-fs-maintext}
\min_{f\in\tilde{\cH}}\max_{g\in\tilde{\cI}} 
    \sum_{i=1}^n 
    \left((f(x_i)-\tilde{y}_i)g(z_i) - \frac{g(z_i)^2}{2}\right) - \frac{\nu}{2}\|g-g_0\|_{\tilde{\cI}}^2 + \frac{\lambda}{2}\|f-f_0\|_{\tilde{\cH}}^2,\tag{\ref{eq:rf-obj}'}
\end{equation}
which differs from \eqref{eq:dualiv-obj} only in the regularizers: instead of regularizing the norm of $f$ and $g$, it encourages the functions to stay close to randomly sampled ``anchors'' \citep{pearce20a}. 
A similar relation is also observed in \citet{osband2018randomized}, which transforms GPR to the optimization problem
$
\min_{f\in\tilde{\cH}} \sum_{i=1}^n (f(x_i)-\tilde{y}_i)^2 + \lambda\|f-f_0\|_2^2. 
$
In both cases, 
the resultant approximate inference algorithm has the same time complexity as ensemble training for point estimation. 

\subsection{Application to Neural Network Models} 
While the algorithm can be directly applied to NN models as in \cite{osband2018randomized}, we follow \cite{he2020bayesian} and modify the objective, to account for 
the difference between the neural tangent kernel \citep[NTK;][]{jacot_neural_2018} of a wide NN architecture and the GP kernel of the corresponding infinite-width Bayesian neural network \citep{neal2012bayesian}. %
Concretely, we modify \eqref{eq:rf-obj} as 
\begin{align}
\min_\theta \max_\phi \sum_{i=1}^n\left(
(\tilde{f}_\theta(x_i)-\tilde{y}_i)\tilde{g}_\phi(z_i)-\frac{\tilde{g}_\phi(z_i)^2}{2}
\right) - \frac{\nu}{2}\|\phi-\phi_0\|^2_2 + \frac{\lambda}{2}\|\theta-\theta_0\|^2_2,\label{eq:ntkrf-obj}\\
\text{where} ~~
\tilde{g}_\phi(z):=g(z;\phi)-g(z;\phi_0)+\tilde{g}_0(z), ~~ \tilde{g}_0(z) := 
\sqrt{\frac{\lambda}{\nu}} 
\big\<\bar{\phi}_0,\:\frac{\partial g}{\partial \phi}\big|_{\phi=\phi_0}(z)\big\>,\nonumber
\end{align}
and $\phi_0$ denotes the initial value of $\phi$, 
$\bar{\phi}_0\sim\cN(0, I)$ is a set of randomly initialized NN parameters independent of $\phi_0$; and $\tilde{f_\theta}$ is defined similarly.

We only give a formal justification for the modification, {under the assumption}\footnote{
See %
\cite{liao_provably_2020} for an analysis of the linearization error in a similar setting to ours.
} that the NNs remain in the kernel regime throughout training, so that 
$g(z;\phi)-g(z;\phi_0) = \<
\phi-\phi_0, \frac{\partial g(z)}{\partial \phi}|_{\phi_0}
\>_2$ \citep{lee2019wide}. 
Thus, for the purpose of analyzing $g(\cdot;\phi)-g(\cdot;\phi_0)$, we can view $g$ as a random feature model with the parameterization $\phi\mapsto \<\phi, \frac{\partial g(z)}{\partial \phi}|_{\phi_0}\>_2$. By the argument in Appendix~\ref{app:rf-props-equiv}, we can show that the weight regularizer $\|\phi-\phi_0\|_2$ is equivalent to $\|g(\cdot;\phi)-g(\cdot;\phi_0)\|_{\tilde{\cI}}=\|\tilde{g}_\phi-\tilde{g}_0\|_{\tilde{\cI}}$, where $\tilde{\cI}$ is determined by the NTK 
$k_{g,ntk}(z,z') := 
\<\frac{\partial g(z)}{\partial \phi}|_{\phi_0}, \frac{\partial g(z')}{\partial \phi}|_{\phi_0}\>_2$. 
Similar arguments can be made for $\tilde{f}_\theta$ and $\tilde{f}_0$. 
Consequently, \eqref{eq:ntkrf-obj} is equivalent to an instance of \eqref{eq:rf-obj-fs-maintext} 
with $\tilde{\cH},\tilde{\cI}$ defined by the NTKs.

Implementation details, including hyperparameter selection, are discussed in Appendix~\ref{app:impl-exp-details}.

\subsection{Convergence Analysis}
We provide a quick analysis of the inference algorithm, by showing that for any fixed set of test points $x_*$, SGDA can approximate the marginal distribution $\Pi(f(x_*)\givendata)$ arbitrarily well given a sufficient computational budget. This implies that the approximate posterior can be good enough for prediction purposes.

We consider standard random feature models as in Section~\ref{sec:rf-rf}. 
We place several mild assumptions listed in Appendix~\ref{app:rf-assumptions}; they are satisfied by common approximations such as the random Fourier features \citep{rahimi2007random}. The SGDA algorithm %
is described in detail in Appendix~\ref{app:ana-gda}. Now we have the following result, which will proved in Appendix~\ref{app:prop-approxinf-proof}:

\begin{proposition}%
\label{prop:approx-inf}
Fix the training data $\data$ and hyperparameters 
$\lambda,\nu>0$. 
Then there exist a sequence of choices of $m$ and SGDA step-size schemes, such that for any $l\in\mb{N}$, we have 
$$\textstyle
\sup_{x^*\in\cX^l} \max(\|\hat{\mu}_m - \mu_m\|_2,
\|\hat{S}_m - S_m\|_F) \overset{p}{\rightarrow} 0.
$$
In the above, 
$\hat{\mu}_m,\hat{S}_m$ denote the mean and covariance of the approximate marginal posterior for $f(x_*)$, $\mu,S$ correspond to the true posterior, $\|\cdot\|_F$ denotes the Frobenius norm, and 
the convergence in probability is defined with respect to the sampling of random feature basis. 
\end{proposition}

\section{Related Work}\label{sec:related-work}

\subsection{Quasi-Bayesian Analysis and IV} 

Quasi-Bayesian analysis for general GMM problems was first developed in 
\citep{zellner_bayesian_1995,kim_limited_2002,chernozhukov_mcmc_2003}, which studied parametric models. 
The use of the quasi-posterior is motivated from the maximum entropy principle, based on which similar ideas have been developed in the machine learning literature \citep[e.g.,][]{jaakkola1999maximum,dudik2007maximum,zhu2009maximum}. 

For nonparametric IV, \cite{liao_posterior_2011} established consistency results. 
\citet{kato_quasi-bayesian_2013} established minimax contraction rates, as well as a Bernstein von-Mises type theorem for smooth functionals, in a setting where the GP prior and $\hat E_n$ are modeled with truncated series models. 
Our theoretical results allow for more flexibility, as we can either  
adapt the results to work with truncated series models, or 
use the orthonormal series to construct our kernels. 
Our result also avoids their requirement that both series models must have the same number of basis. %
\citet{florens_nonparametric_2012} derived a quasi-posterior using general GP priors and Nadaraya-Watson smoothing for $\hat{E}_n$. 
They proved an $L_2$ rate under the different assumption of \emph{source condition}, which requires $f_0$ to be in the range of $(E^*E)^{\beta}$. 
It is unclear if their rate is optimal, but it can only utilize limited smoothness of $f_0$ ($\beta\le 2$), and 
their rate cannot be better than $O(n^{-1/4})$. This is different from our result, where $\epsilon_n$ approaches $O(n^{-1/2})$ as $b\rightarrow\infty$. 
On the empirical side, 
no numerical study was presented in \citet{liao_provably_2020,kato_quasi-bayesian_2013}, whereas \cite{florens_nonparametric_2012} only provided numerical studies for the mean estimator. 

\subsection{Kernelized IV} 

Our quasi-Bayesian procedure builds upon the kernelized IV\footnote{
Note that \citet{zhang_maximum_2020} also studied RKHS methods for IV, with a different use of the kernel. 
} methods \citep{singh_kernel_2020,muandet_dual_2020} and the dual formulation of IV regression \citep{bennett2019deep,muandet_dual_2020,liao_provably_2020,dikkala_minimax_2020}. 
The estimation objective \eqref{eq:dualiv-obj} is from \citep{dikkala_minimax_2020}; the formulations in \citep{muandet_dual_2020,singh_kernel_2020} are slightly different. 
All of these works have only provided nonparametric error rates in the norm $\|E(\cdot)\|_2$. 
Among them,  
\cite{dikkala_minimax_2020} provide the best rate of $\max\{n^{-\frac{b}{2(b+1)}}, \delta_n\}$. As discussed in Section~\ref{sec:assumptions}, the optimal choice of $\cI$ should yield $\delta_n\asymp n^{-\frac{b+2p}{2(b+2p+1)}}$, so the rate in \citep{dikkala_minimax_2020}
can only be optimal %
if $p=0$. Unfortunately, the $p=0$ case is uninteresting as it excludes the challenge of ill-posedness. 
We complete the theoretical picture, by 
improving the $\|E(\cdot)\|_2$ rate to $\delta_n$, and providing the first minimax rates for $L_2$ and Sobolev norms.

\subsection{Classical Nonparametric Methods} 

In the setting of \cite{chen2011rate}, 
minimax rates have been established for classical methods. 
It is thus interesting to compare the kernelized IV method with them. 
Smoothing-based methods also attain optimal $L_2$ rates \citep{horowitz_specification_2012}, 
but it is unclear if optimal recovery of $f_0$ and its derivatives can be achieved simultaneously, i.e., with the same hyperparameter \citep{chen_optimal_2015}.
This is different from our estimator, and the sieve estimators \citep{chen_optimal_2015}, where optimal $L_2$ and Sobolev rates are achieved by the same choice of hyperparameter. 

Sieve estimators \citep{newey2003instrumental,blundell_semi-nonparametric_2007,chen_estimation_2012} model $f_0$, and possibly also $\hat E_n$, with orthogonal series. %
The kernelized IV estimator is thus similar to a sieve estimator using the Mercer bases $\{\bar\varphi_j\}$. 
One difference is that we include the term $\effnu\|g\|_\cI^2 = \effnu\|\hat E_n(f-f_0)\|_\cI^2$ in the objective for $f_0$. 
More importantly, 
the classical sieve estimators require the basis functions to be determined \emph{a priori}, %
as their construction involves explicit truncation of the series model. 
This prevents adaptation to regularity properties of the data distribution. 
In our case, however, the series $\{\bar\varphi_j\}$ is implicitly determined by the data distribution, and thus allows for some adaptation to the regularity properties. E.g., if $\bx$ or $\bz$ is supported on a low-dimensional subspace of $\RR^d$, the kernelized estimators automatically ``construct'' the bases according to that subspace. 
On the other hand, the classical estimators have been studied extensively, with a wider range of theoretical results established in literature (e.g., recovery of nonlinear functional evaluations, validity of specification tests). 

\subsection{Other Related Work}\label{sec:other-related-work} Other recent work applying ML methods to nonlinear IV include \citep{hartford2017deep,zhang_maximum_2020,xu_learning_2020}. 
They focused on estimation methodology as opposed to uncertainty quantification, and did not 
provide results for nonparametric convergence.\footnote{
\citet{bennett_variational_2020} also analyzed parametric models. 
While \cite{liao_provably_2020} provided an $L_2$ error bound, it is in the \emph{infinite sample} setting, and describes the dependency %
on computational budgets. 
}  
Among them, \citet{zhang_maximum_2020} also discussed a connection between a GP (quasi-)posterior and a leave-one-out validation statistic, but the GP is used for hyperparameter selection and its frequentist properties (e.g., contraction rates, validity of credible sets) remain unclear. 
For the closely related problem of causal effect estimation with proxy variables \citep{miao2018identifying}, which also has a CMR structure, \cite{mastouri2021proximal} establish an RKHS-norm rate for the structural function, under source conditions similar to \cite{florens_nonparametric_2012}. Their rate is also lower bounded by $O(n^{-1/4})$, although note the stronger norm.  %

Alternative approaches to uncertainty quantification include Bayesian inference and bootstrap. We have discussed the limitations of Bayesian IV in Section~\ref{sec:intro}. Here we note that while the closely related \emph{control function approach} is not less general than CMR \citep{horowitz_applied_2011}, standard Bayesian approaches additionally assume independent errors, 
making their assumption strictly stronger than CMR. Nonetheless, the quasi-Bayesian approach is connected to standard Bayesian modeling; 
\citet{florens_gaussian_2021} shows that for unconditional moment restriction (MR) problems, the quasi-posterior can be recovered as a non-informative limit of a sequence of standard Bayesian posteriors, where a prior on the joint data distribution becomes diffuse. However, computation of the Bayesian posteriors therein is expensive, and generalization to nonparametric CMR problems appears nontrivial. \SkipNOTE{
They have strong assumption on the covariance of the nuisance parameter (full data distribution), which has conceptual difficulties (Remark 2.1) and is difficult to implement even for a finite number of MRs (Sec 2.1.2). 
} 

For the classical sieve estimator, bootstrap enjoys a variety of theoretical guarantees, but non-asymptotic justification is relatively lacking. 
The quasi-Bayesian approach could be more appealing in this aspect, due to results such as \eqref{eq:gibbs-justification} and \eqref{eq:cov-as-worst-case-error}, although there is a similar lack of complete non-asymptotic analysis. %
The need for bootstrap alternatives is also justified through analyses in the parametric regime, where 
nonparametric bootstrap inference for the classical two stage least squares (2SLS) method %
is known to be unreliable when instrument strength is weak \citep{moreira2004bootstrap,
flores-lagunes_finite_2007,davidson_wild_2010}.

\section{Synthetic Experiments}\label{sec:exp}

In this section we evaluate the proposed method. We first study the asymptotic and pre-asymptotic behavior of the quasi-posterior in a controlled setting; then, we compare its predictive performance with various baselines.\footnote{
Code for the experiments can be found at \url{https://github.com/meta-inf/qbdiv}.}

\subsection{Experiment in a Controlled Setting}\label{sec:exp-asymp-val}

We first experiment in a setting where all assumptions hold with known parameters, and the models are correctly specified. 
This setting allows us to validate the asymptotic theory, but also provides a first glimpse at the pre-asymptotic behaviors. 

The setup follows Example~\ref{rmk:assumptions}: we set $\cH,\cI$ as the Sobolev spaces on $\mathbb T^1$ with the appropriate orders. Points on $\mathbb T^1$ are represented with a coordinate in $[0,1]$. 
We sample $f_0$ from the GP prior, and define the
joint density $f_{XZ}$ as 
$$
f_{XZ}(x,z):=1 + \sum_{j=2}^\infty 0.2\cdot (j-1)^{-p} \bar\varphi_j(x)\bar\varphi_j(z),
$$
where $\bar\varphi_j(x) := \cos(2\pi [j/2](x) + (j-[j/2])\pi)$ 
constitutes the Mercer basis of $k_x$ and $k_z$. It can be verified that $f_{XZ}$ is a valid density, and the marginal distributions of $x$ and $z$ are both the uniform distribution. 
We introduce confounding by generating the observations as 
$$
(z_i,u_i)\overset{i.i.d.}{\sim} \cU[0,1], ~~ x_i := F^{-1}_{x\mid z=z_i}(u_i), ~~ y_i = f_0(x_i) + \tilde\sigma \cos(2\pi u_i),
$$
where $F^{-1}_{x\mid z=z_0}(\cdot)$ denotes the inverse CDF transform of the conditional distribution, and 
$\tilde\sigma$ is chosen so that $\mrm{Var}[u] = \sigma^2$. We approximate $f_{XZ}$ and the kernels by truncating the series sum at $J=400$, and further approximate $f_{XZ}$ with linear interpolation on a $500\times 500$ grid. We vary $(b,p)\in \{(2,0.5), (3,0.5), (1,2)\}$, and $n\in\{1,3,6,12,24,48\}\times 10^3$. The hyperparameters are set as $\lambda := \sigma^2 = 1$, $\effnu := 20 (\frac{n}{1000})^{-\frac{b+2p}{b+2p+1}}$.

When $n$ is large, we find the closed-form expressions \eqref{eq:post-mean-calc}-\eqref{eq:post-L} to be numerically unstable, likely due to the need to invert the $n\times n$ matrix $\lambda I+L K_{xx}$, whose eigenvalues may have a decay rate of $\lambda_i\asymp i^{-b+2p+1}$. 
We find Nystr\"om approximation for $L$ to be an effective workaround; see Appendix~\ref{app:nystrom-expression} for a description of the algorithm. As shown Appendix~\ref{app:proof-nystrom}, the approximation will not affect the asymptotic rates when the number of inducing points reaches $O(n^{\frac{1}{b+2p+1}})$. 
We use $50$ inducing points.

Figure~\ref{fig:gt-sim} plots the expected $L_2$ error averaged over the posterior measure, $\int \|f-f_0\|_2^2 \Pi(df\givendata)$, which upper bounds the squared contraction rate. As we can see, the observations validate the asymptotic theory. 

\begin{figure}[t]
     \centering
\includegraphics[width=0.7\linewidth]{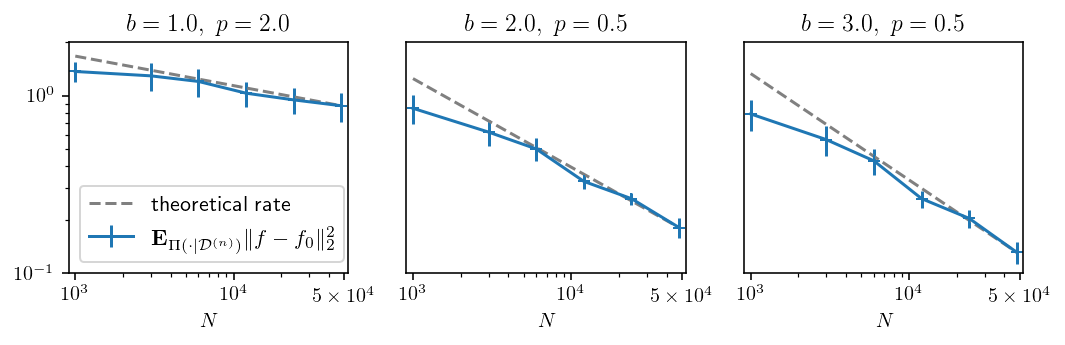}
\caption{Average $L_2$ error $\EE_{\Pi(df\givendata)}\|f-f_0\|_2^2$ vs the theoretical rate $O(\epsilon_n^2)$, for varying choices of $(b,p)$. Error bars indicate standard deviations estimated from 10 independent trials. 
The theoretical rate is scaled for clarity.
}\label{fig:gt-sim}
\end{figure}

\begin{figure}[b]
    \centering
\includegraphics[clip,trim={8cm 0 4cm 0},width=.795\linewidth]{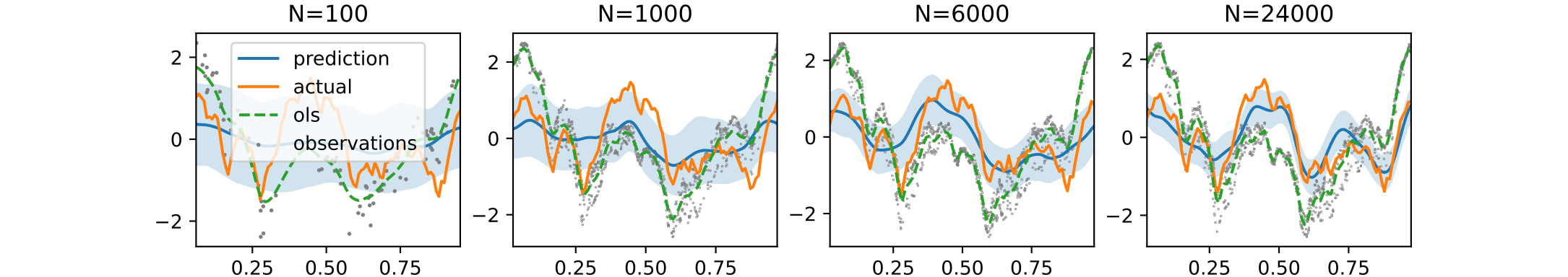}
\hfill 
\includegraphics[clip,trim={0.2cm 0.11cm 0 0.17cm},width=.195\linewidth]{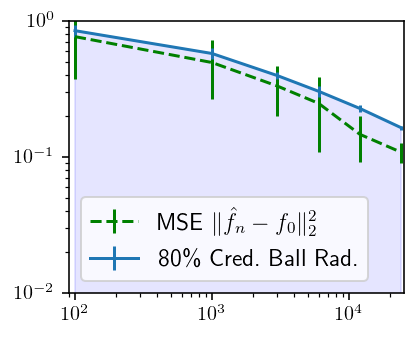}
\caption{Visualizations of %
the quasi-posterior, %
with $b=2,p=0.5$. 
Left: sample runs; 
shade indicates pointwise $80\%$ credible interval, dots indicate training data, and ``ols'' denotes a (biased) KRR estimator for reference.  
Right: comparison of the radius of the $80\%$ $L_2$ credible ball with the error of the posterior mean estimator. 
}\label{fig:gt-sim-viz}
\end{figure}

Figure~\ref{fig:gt-sim-viz} visualizes the uncertainty estimates from the quasi-posterior. 
We can see that the quasi-posterior has reliable coverage in the pre-asymptotic regime. Note that we plot credible intervals for illustrative purposes only; recall our theory does not provide guarantee for its coverage, even though the discussion in Section~\ref{sec:nonasymptotic} suggests that good coverage may be achieved, for ``benign'' $f_0$ as in this setting. 

Additional results are provided in Appendix~\ref{app:details-gt}. In brief, we find that
\begin{enumerate}
\item Across all settings, the posterior spread $\int \|f-\hat f_n\|_2^2 \Pi(df\givendata)$ is approximately as large as the $L_2$ error of the point estimator, $\|\hat f_n-f_0\|_2^2$; this further suggests the $L_2$ credible balls are consistently informative and reliable, as its radius is an upper percentile of the random variable $\|f-\hat f_n\|_2$. 
\item The $L_2$ contraction rate for the derivative $f'$ (i.e., the $W^{1,2}$ rate for $f$) also appears consistent with the theory, although convergence to the asymptotic regime is slower. 
\end{enumerate}

\subsection{Predictive Performance: 1D Simulation}\label{sec:exp-1d}

We now evaluate the predictive performance of the proposed method on a variety of 1D synthetic datasets adapted from \cite{lewis2018adversarial}. We modify their setup to incorporate a nonlinear first stage, in a way similar to 
\cite{singh_kernel_2020,chen_optimal_2015}:
\begin{align*}%
    z := \mrm{sigmoid}(w), ~~
    x := \mrm{sigmoid}\!\left(\frac{\alpha w+(1-\alpha)u'}{\sqrt{\alpha^2+(1-\alpha)^2}}\right), ~~
    y_i \sim \cN(f_0(2x-1) + 2u, 0.1),
\end{align*}
where $(u,u')$ are normal random variables with unit variance and a correlation of $0.5$, $w\sim\cN(0,1)$ is independent of $(u,u')$, 
$\alpha$ is a parameter controlling the instrument strength, and $f_0$ is constructed from the \texttt{sine}, \texttt{step}, \texttt{abs} or a \texttt{linear} function. We choose $N\in\{200,1000\}$ and $\alpha\in\{0.05, 0.5\}$. 

Our 
baselines include BayesIV \citep{wiesenfarth2014bayesian}, a flexible Bayesian model based on B-splines and Dirichlet process mixture; we also include bootstrap on 2SLS with ridge regularization, either applied directly to the input features (Linear), on their polynomial expansion (Poly), or on the same kernelized models (KIV) as ours.%
\footnote{
    We do not compare with \cite{florens_nonparametric_2012} since their source code is unavailable. %
} 
Hyperparameter for the kernelized IV methods are selected by cross validation based on the observable first-stage and second-stage losses as in previous work \citep{singh_kernel_2020,muandet_dual_2020}; see Appendix~\ref{app:hps-sel}. 
For kernels, we choose the 
RBF and Mat\'ern kernels, but defer the results for Mat\'ern kernels to appendix for brevity. 
See Appendix~\ref{app:details-1d} for the detailed setup. 

\begin{figure}[bt]
\centering
\includegraphics[width=\linewidth]{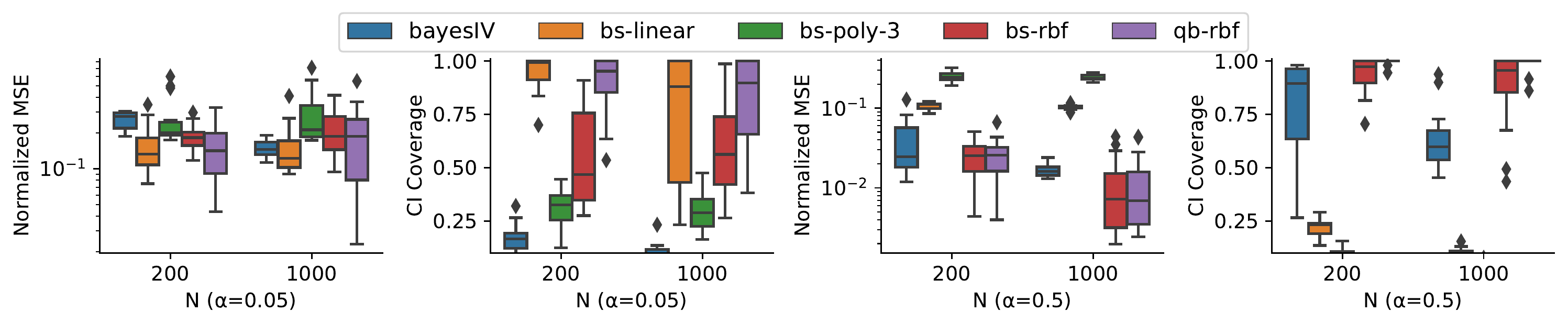}
\caption{Test MSE and CI coverage on the \texttt{sine} dataset. The left two plots correspond to $\alpha=0.05$, while the right two correspond to $\alpha=0.5$. \texttt{bs} denotes bootstrap, \texttt{qb} denotes quasi-Bayesian. 
We report the average CI coverage over the data distribution.
}\label{fig:cubic-sin}
\end{figure}

We report MSE, coverage rate of $95\%$ credible intervals (CIs) and CI width in this experiment. 
While there is no theoretical guarantee for CI coverage, our discussion in Section~\ref{sec:nonasymptotic} and the last experiment  
suggest reliable coverage can be possible in practice. Besides, 
it can be much easier to provide pointwise CIs with good coverage \emph{when averaged} over the input space, than to provide reliable CIs for the worst-case input; the former is still useful in many applications. 

Normalized MSE and CI coverage on the \texttt{sine} datasets are plotted in Figure~\ref{fig:cubic-sin}. We report results on 20 independently generated datasets. As we can see, quasi-Bayesian inference provides the most reliable uncertainty estimates, especially in the $\alpha=0.05$ setting. 
The CIs are typically conservative (the larger variation when $\alpha=0.05$ is partly due to the instability of the hyperparameter selection procedure). 
However, as shown in Appendix~\ref{app:addi-1d}, it is still informative, and properly reflects the sample size and instrument strength. 

Full results on all datasets, visualizations and additional experiments are deferred to Appendix~\ref{app:addi-1d}. As a summary, (i) %
on the \texttt{abs} and \texttt{linear} datasets where the kernel choices are more justified, the results are qualitatively similar to the sin dataset. Moreover, the over-smoothed RBF kernel appears to have similar coverage comparing with the optimal kernel, and follow a similar contraction rate. 
(ii) 
All methods have deteriorated performance on the \texttt{step} dataset, although the quasi-posterior still provides more coverage. 
(iii) Uncertainty estimates produced by the approximate inference algorithm are similar to that from the exact quasi-posterior. 
The first two findings above may be attributed to the fact that all three datasets, especially $\texttt{abs}$ and $\texttt{linear}$, are ``regular'' in most regions of the input space.

\subsection{Predictive Performance: Airline Demand}\label{sec:exp-demand}

Finally, we consider the more challenging demand simulation proposed in \cite{hartford2017deep}. The dataset simulates a scenario where we need to predict the demand of airline tickets $y$, as a function of the price $x$, and two observed confounders: customer type $s$, and time of year $t$. 
The data generating process is 
$$
x:=(z+3)\psi(t)+25+u', ~~ y:=f_0(x,t,s)+u, ~~
f_0(x,t,s):=100+(10+x)\cdot s\cdot\psi(t)-2x
$$
where $(u,u')$ are standard normal variables with correlation $\rho$, $z\sim\cN(0,1)$ is independent of $(u,u')$, and $\psi$ is a nonlinear function whose shape is given in Figure~\ref{fig:demand-main}. The variable $s$ either varies across $\{0,\ldots,6\}$ (the lower-dimensional setting), or is observed as an MNIST image representing the corresponding digit; the latter case emulates the real-world scenario where only high-dimensional surrogates of the true confounder is observed. 
We only report results for $\rho=0.5$, noting that results using other choices of $\rho$ have been similar. 
We use $n\in\{10^3,10^4\}$ for the lower-dimensional setting, and use $n=5\times 10^4$ for the image setting. 

\begin{figure}[btp]
    \centering
    \begin{subfigure}[b]{0.24\linewidth}
        \includegraphics[width=\linewidth]{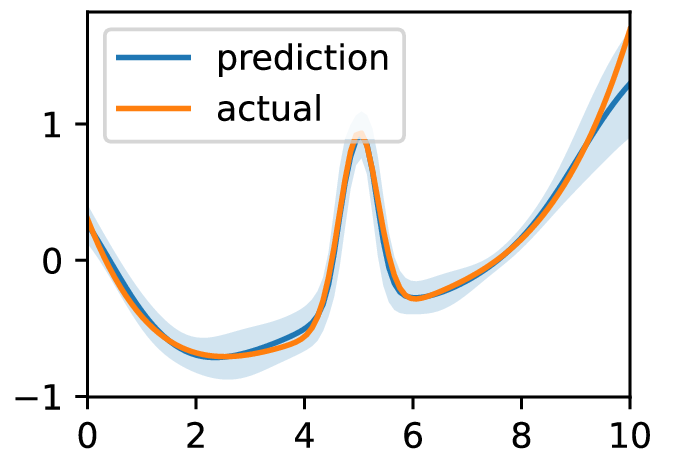}
        \caption{lower-dim, $N=10^3$}
    \end{subfigure}~~~~
    \begin{subfigure}[b]{0.24\linewidth}
        \includegraphics[width=\linewidth]{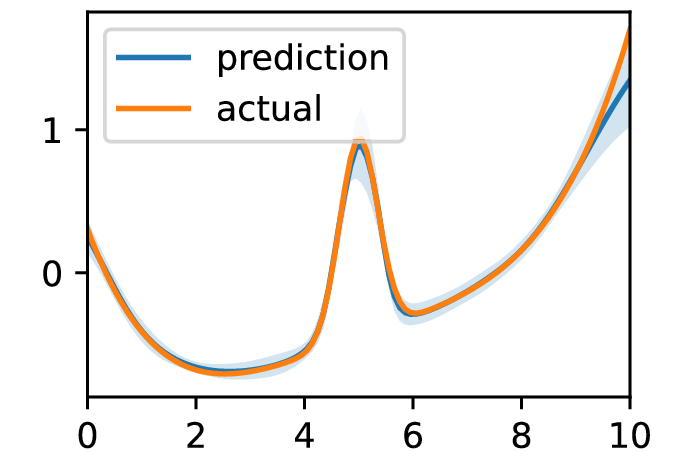}
        \caption{lower-dim, $N=10^4$}
    \end{subfigure}~~~~
    \begin{subfigure}[b]{0.24\linewidth}
        \includegraphics[width=\linewidth]{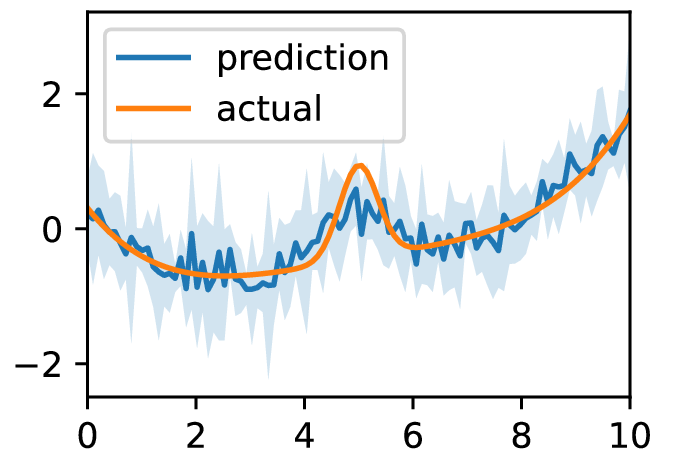}
        \caption{image setup}
    \end{subfigure}
    \caption{Approximate quasi-posteriors in the demand simulation. 
    We plot a cross-section by fixing $s,x$ to their mean values and varying $t$. 
    Shade indicates pointwise $95\%$ credible interval. 
    }\label{fig:demand-main}%
\end{figure}
We compare our method with bootstrap on the same model, BayesIV, and bootstrap on linear or polynomial models. Performance of other point estimators on this dataset has been reported in \cite{singh_kernel_2020,muandet_dual_2020,xu_learning_2020}, compared with which our method is generally competitive. 
We implement the dual IV model using both an RBF kernel and DNN models.  
See Appendix~\ref{app:demand-setup} for details.

We report the test MSE and CI coverage for $N=1000$ 
in Table~\ref{tbl:hllt-main}, 
and visualize all approximate quasi-posteriors for the NN models in Figure~\ref{fig:demand-main}. 
As we can see, when implemented with DNNs, our method produces uncertainty estimates with excellent coverage, which also correctly reflects the information available in the dataset: the CI is wider when $N$ is smaller, or in the high-dimensional experiment where estimation is harder. 
Bootstrap has a noticeably worse performance when $N=1000$. Still, it 
performs well in the (arguably less interesting) large-sample setting, with a CI coverage similar to our method; see Appendix~\ref{app:demand-full-results}. This is because on this dataset, the total instrument strength is stronger due to the presence of observed confounders, and the NN model is a good fit. Consequently, the asymptotic behavior of bootstrap can be observed when $N$ is large. %

Both methods have poorer performances when we switch to the RBF kernel, although the quasi-posterior is still more reliable. 
Given the discussion in Section~\ref{sec:nonasymptotic}, we hypothesize that both $\cI$ and $\cH$ are misspecified in this case, which seems reasonable as the magnitude of $f_0$ varies greatly as $s,x$ varies (not shown in Figure~\ref{fig:demand-main}). In this case the quasi-posterior could only reflect the uncertainty within the model. 
These results show that NN models can be advantageous, which our inference algorithm supports. 

The other baselines perform poorly due to their inflexibility; in particular, note that BayesIV uses additive models for both stages (e.g., $f(x,t,s)=f_1(x)+f_2(t)+f_3(s)$) which do not approximate this data generating process well. 
Full results and visualizations are deferred to Appendix~\ref{app:demand-full-results}. 
\begin{table}[bth]\centering
\small
\begin{tabular}{cccccccc}
\toprule
Method & BS-Linear & BS-Poly & BayesIV & BS-RBF & QB-RBF & BS-NN & QB-NN \\ \midrule
NMSE & 
$.37\pm.01$ &
$.31\pm.06$ &
$.28\pm.04$ &
$.17\pm.01$ &
$.17\pm.01$ &
$.06\pm.03$ &
{${.04\pm.00}$} \\
CI Cvg. & 
$.09\pm.01$ &
$.15\pm.03$ &
$.27\pm.06$ &
$.45\pm.02$ &
$.77\pm.02$ &
$.86\pm.02$ &
{${.94\pm.01}$} \\
CI Wid. 
& $.09\pm .01$
& $.16\pm .04$
& $.08\pm .06$
& $.18\pm .02$
& $.37\pm .01$
& $.14\pm .01$ 
& $.26\pm .04$ \\
\bottomrule
\end{tabular}
\caption{Test normalized MSE, average CI coverage and CI width on the demand dataset, with $N=1000$, averaged over $20$ trials.
}\label{tbl:hllt-main}
\end{table}

\section{Conclusion}\label{sec:conclusions}%

In this work we present a quasi-Bayesian procedure for kernelized and NN-based IV models. We analyze the frequentist behavior of the proposed quasi-posterior, 
 and derive an approximate inference algorithm that can be applied to neural network models. 
Empirical evaluations show that the proposed method scales to large and high-dimensional datasets, and can be advantageous in the finite-sample setting, or when the instrument strength is weak.

\appendix

\newcommand{\frp}{\bar{f}_{rp}}
\newcommand{\grp}{\bar{g}_{rp}}
\newcommand{\Cx}{{\cC}_x}
\newcommand{\empCx}{{\hat{\cC}_x}}
\newcommand{\regCx}{{\cC}_{x,\efflam}}
\newcommand{\regEmpCx}{\hat{\cC}_{x,\efflam}}
\newcommand{\fApprox}{{\bar f^\dagger_n}}
\newcommand{\fGT}{\bar f_0}
\renewcommand{\empLnorm}[1]{\|{#1}\|_n}
\renewcommand{\Lnorm}[1]{\|{#1}\|_2}

\newcommand{\Proj}[2]{\mrm{Proj}_{{#1}}{{#2}}}
\newcommand{\DeltaFApprox}{\widetilde{\Delta f}}

\makeatletter
\newcommand{\pushright}[1]{\ifmeasuring@#1\else\omit\hfill$\displaystyle#1$\fi\ignorespaces}
\newcommand{\pushleft}[1]{\ifmeasuring@#1\else\omit$\displaystyle#1$\hfill\fi\ignorespaces}
\makeatother

\renewcommand{\mcomment}[1]{\pushright{\mgray{\text{(#1)}}}}
\renewcommand{\fApprox}{{f^\dagger_n}}
\newcommand{\fApproxM}{{f^\dagger_m}}

\section{Proof of the Asymptotic Results}\label{app:proof-asymptotic}

This section contains all proofs for the asymptotic results in Section~\ref{sec:theory}. 

\subsection{Additional Notations and Conventions}\label{app:proof-conventions}

We recall the following notations. 
$\bar\cH,\bar\varphi_i,\lambda_i$ are defined in Section~\ref{sec:assumptions}. 
$\data=(\bX,\bY,\bZ)$ denotes the observed data. 
$\Pi$ denotes the GP prior, and $\Pi_n(df\mid\data)$ denotes the quasi-posterior. 

Denote by $\|\cdot\|_n$ the $L_2$ norm w.r.t.~the empirical data distribution. 
For any $x\in\RR$, define $[x]$ as the maximum integer not exceeding $x$. 
Define 
$$
\ell_n(f) := \empLnorm{\hat{E}(f(\bx)-\by)}^2+\effnu \|\hat{E}(f(\bx)-\by)\|_\cI^2, 
$$
so that 
$
\frac{d\Pi_n(\cdot\mid\data)}{d\Pi}(f) \propto \exp\!\left(-\frac{n}{2\lambda} \ell_n(f)\right).
$
We overload notations and define the constant $B$ as 
$$
B := \max\{\|f_0\|, \|E f_0\|_\cI, \max|\by-f_0(\bx)|\}
$$
where the last term denotes the bound of the random variable (which is bounded by Assumption~\ref{ass:std}.)

Throughout this section, the denotation of the constants introduced in the proof ($C_1,\ldots$, $c_1,\ldots$) may change by each line. 
The constants hidden in $\lesssim,\gtrsim,\asymp$ will be uniformly bounded, across all possible data distributions and model choices satisfying the assumptions in Section~\ref{sec:assumptions}. Therefore, the asymptotic results will hold uniformly across these distributions.

\subsection{Background and Auxiliary Results}\label{app:aux}

\subsubsection{Background on Power Spaces}\label{app:power-space}

Recall the definition of power space in Definition~\ref{defn:power-space}. 
\begin{lemma}\label{lem:power-space}
Let $\cH$ satisfy Assumption~\ref{ass:k-std} and \ref{ass:H1}. Then, for any $\alpha\in\left [\frac{b_{emb}}{b+1}, \infty\right )\supset \left[
    \frac{b}{b+1}, \infty 
\right)$, we have 
\begin{inplaceEnumerate}
\inplaceItem 
the function space $\cH^\alpha$ given by 
\begin{equation}\label{eq:powerrkhs-defn}
\cH^\alpha = \Bigg\{\sum_{j=1}^\infty a_j\bar\varphi_j: ~
\sum_{j=1}^\infty \lambda_j^{-\alpha} a_j^2 <\infty
 \Bigg\}
\end{equation}
(and the norm defined accordingly) 
is an RKHS, with the reproducing kernel $$
k_\alpha(x,x') := \sum_{j=1}^\infty \lambda_j^{\alpha} \bar\varphi_j(x)\bar\varphi_j(x'). 
$$ 
\inplaceItem \label{item:power-space-bounded-kernel}
$k_\alpha$ is bounded. 
\inplaceItem \label{item:power-space-emb}
The RKHS $\cH^\alpha$ satisfies Assumption~\ref{ass:H1}~(ii), with the exponent $\frac{b_{emb}}{b+1}$ replaced by $\frac{b_{emb}}{\alpha(b+1)}$. 
\end{inplaceEnumerate}
\end{lemma}
\begin{proof}
(i)-(ii):  
As the kernel $k_x$ is bounded, 
by \citet[Lemma 2.3, \SkipNOTE{bounded kernel => compact embedding to L2}
Definition 4.1, \SkipNOTE{requires compact embedding and the following display}
Proposition 4.2\SkipNOTE{requires no further conditions}]{steinwart_mercers_2012}, (i) will be true if for all $x\in\cX$, we have 
$$
\sum_{j=1}^\infty \lambda_j^{\alpha} \bar\varphi_j^2(x)<\infty.
$$
This will also imply (ii). 
The above display will hold if we can apply Theorem 5.5 in \citet{steinwart_convergence_2019},\SkipNOTE{``Eq.(33) will hold for $S=T$''}\footnote{
While their theorem only state the result for $\alpha\le 1$, the $\alpha>1$ case follows since $\sum_j \lambda_j^\alpha \bar\varphi_j^2(x) \le  
\sum_j \lambda_j \bar\varphi_j^2(x)$. The definitions in \citet{steinwart_mercers_2012} allow for any $\alpha>0$. 
}
which requires that their Assumptions K and CK hold, and that $[\cH]^\alpha$ can be continuously embedded into $L_\infty(P(dx))$. Under our Assumption~\ref{ass:k-std}, 
most requirements in Assumptions K and CK can be verified immediately from the discussions therein, with the only possible exception of 
separability of $L_2(P(dx))$, which follows because $\{\bar\varphi_i\}$ form an ONB; see the discussion around Assumption~\ref{ass:k-std}. 
The embedding property follows from our Assumption~\ref{ass:H1} (ii), and the fact that  
the property is stronger when $\alpha$ is smaller \citep[see e.g.,][Eq.~(7)]{fischer2020sobolev}. 

Now, (iii) follows from the assumption, and the observation that $[\cH^\alpha]^{\alpha'} = [\cH]^{\alpha\alpha'}$. 
\end{proof}

Observe that the $\cH^\alpha$-norm defined in \eqref{eq:powerrkhs-defn} coincides with the definition of the $[\cH]^\alpha$ norm. 
Technically the two spaces are different, since $[\cH]^\alpha$ is a normed subspace of $L_2(P(dx))$ and consists of equivalence classes of functions. But the above lemma shows that when $\alpha$ is sufficiently large, $\cH^\alpha$ can be continuously embedded into $[\cH]^\alpha$. Conversely, each equivalence class in $[\cH]^\alpha$ consists of functions with the same generalized Fourier coefficients $\{\<f,\bar\varphi_i\>_2\}$ (and thus the same $[\cH]^\alpha$-norm), and corresponds to a unique function in $\cH^\alpha$ with an equal $\cH^\alpha$-norm. Therefore, we do not distinguish between them in the text, and refer to $\bar\cH = [\cH]^{\frac{b}{b+1}}$ as having an RKHS structure. 

\begin{lemma}\label{lem:interpolation}
Let $\cH$ satisfy Assumptions~\ref{ass:k-std}, \ref{ass:H1}. Then %
for any $f \in\cH$, we have 
\begin{align}
\|f\|_{\bar\cH}^2 &\lesssim (\|f\|_\cH^2)^{\frac{b}{b+1}} (\|f\|_2^2)^{\frac{1}{b+1}}, \nonumber \\ 
    \|f\|^2 &\lesssim (\|f\|_\cH^2)^{\frac{b}{b+1}} (\|f\|_2^2)^{\frac{1}{b+1}}. \label{eq:interp-ineq} %
\end{align}
\end{lemma}
\begin{proof}
For the $\bar\cH$ norm, observe that for the Mercer basis $\{\bar\varphi_i\}$, 
 $
(\|\bar\varphi_i\|_\cH^2)^{\frac{b}{b+1}} \asymp (i^{b+1})^{\frac{b}{b+1}} 
\asymp 
\|\bar\varphi_i\|_{\bar\cH}^2
$, and thus 
$
\|\bar\varphi_i\|_{\bar\cH}^2 \lesssim (\|\bar\varphi_i\|_\cH^2)^{\frac{b}{b+1}} (\|\bar\varphi_i\|_2^2)^{\frac{1}{b+1}}, 
$
and for general $f$ the claim follows from H\"older's inequality $\<\cdot,\cdot\>\le \|\cdot\|_{\frac{b+1}{b}} \|\cdot\|_{b+1}$. \SkipNOTE{in the sequence space}

\eqref{eq:interp-ineq} now follows from the $\bar\cH$ norm bound, and Lemma~\ref{lem:power-space}~\ref{item:power-space-bounded-kernel}. 
\end{proof}

\subsubsection{Preliminary Results for the GP Prior}\label{app:gp-sieve}

In the posterior contraction framework, we need to construct a sequence of parameter subsets $\{\Theta_m\}$, \emph{sieves} \citep{ghosal2000convergence}, such that %
their prior mass approaches 1 
 sufficiently fast, yet the complexity of $\Theta_m$ grows slowly. 
See \citet{van_der_vaart_rates_2008,van_der_vaart_information_2011} for a standard construction for GP priors. 
In this subsection, we construct $\Theta_m$ by adapting the result in 
\citet{steinwart_convergence_2019}. %
\eqref{eq:valid-sieve} is the main result.

\begin{lemma}\label{lem:preliminary-sieve}
Let $W\sim\Pi$ be a Gaussian process with covariance kernel $k_x$, where $k_x$ satisfies our Assumptions~\ref{ass:k-std},~\ref{ass:H1}~(i). 
Then for all $b'\in [0,b), m>0$, we have, for some constant $C_1, C_2, C_3>0$ and $\tau_m'^2 = m^{-\frac{b-b'}{b+1}}$, 
\begin{align}
&\Pi(\{\|W\|_{[\cH]^{\frac{b'}{b+1}}}^2\le \tau_m'^2\}) \ge e^{-C_1 m^{\frac{1}{b+1}}}, \label{eq:sobolev-small-ball}\\ 
&\Pi(\Theta_m') := \Pi(\{W = f_h + f_e: \|f_h\|_\cH\le C_2 m^{\frac{1}{2(b+1)}}, \|f_e\|_{[\cH]^{\frac{b'}{b+1}}}\le \tau_m'\}) \ge 
1 - e^{-C_3 m^{\frac{1}{b+1}}}. \label{eq:sobolev-sieve}
\end{align}
\end{lemma}
\begin{proof}
By \citet[Corollary 4.9]{steinwart_convergence_2019}, 
\SkipNOTE{in their notation: $\alpha=b+1, (1-\beta)\alpha=b'$. For GP they only require $\alpha\beta>1$.}
we have $
-\log \Pi(\{\|W\|_{[\cH]^{\frac{b'}{b+1}}}^2\le \epsilon^2) \asymp \epsilon^{-\frac{2}{b-b'}} 
$ which implies \eqref{eq:sobolev-small-ball}. \eqref{eq:sobolev-sieve} then follows from 
Borell's inequality \citep[see, e.g.][Proposition 11.17]{ghosal2017fundamentals}. 
\end{proof}

\begin{corollary}\label{corr:valid-sieve}
Let $\Pi,W,k_x$ be defined as in the above lemma. Then (i)
for any $b'\in [0, b)$, 
there exists constant $C_{b'},C_{b'}'>0$ such that for all $m\in\mb{N}$ and $\tau_m := m^{-\frac{b/2}{b+1}}$, 
\begin{align*}
\Pi(\Theta_{m,b'}) &:= \Pi(\{f = f_h+f_e : 
    \|f_h\|_\cH^2\le C_{b'} m\tau_m^2, 
    \|f_e\|_2^2\le C_{b'} \tau_m^2, 
    \|f_e\|_{[\cH]^{\frac{b'}{b+1}}}^2 \le C_{b'} m^{-\frac{b-b'}{b+1}}\}) \\ 
    &\ge 1 - e^{-C_{b'}'m\tau_m^2}. 
\end{align*}
(ii) If additionally Assumption~\ref{ass:H1}~(ii) holds,  there will exist constants $C,C'>0$ such that for all $m\in\mb{N}$, 
\begin{equation}
\Pi(
\Theta_{m}) := \Pi( \{f = f_h+f_e :~ 
    \|f_h\|_\cH^2\le C m\tau_m^2, ~
    \|f_e\|_2^2\le C \tau_m^2, ~
    \|f_e\|^2 \le C \}
) \ge 1 - e^{-C'm\tau_m^2}. \label{eq:valid-sieve}
\end{equation}
\end{corollary}
\begin{proof}
(i): We claim that $\Theta_m'$ defined in the above lemma satisfies the definition of $\Theta_{m,b'}$. To see this, for any $f\in\Theta_{m}'$ let $f_h,f_e$ be defined as in \eqref{eq:sobolev-sieve}. Let $$
j := [m^{\frac{1}{b+1}}], ~
\tilde f_h := f_h + \Proj{j}{f_e}, ~
\tilde f_e := \Proj{>j}{f_e}. 
$$
Then by Lemma~\ref{lem:sobolev} below applied to $[\tilde\cH] := [\cH]^{\frac{b'}{b+1}}$, we have 
$$
\max\{\|\Proj{j}{f_e}\|_{[\tilde\cH]},\|\tilde f_e\|_{[\tilde\cH]}\} \le \|f_e\|_{[\tilde\cH]}, ~~
\|\tilde f_e\|_2^2 \le \|f_e\|_{[\tilde\cH]}^2 j^{-b'} 
\lesssim m^{-\frac{b-b'}{b+1}} m^{-\frac{b'}{b+1}}  = m^{-\frac{b}{b+1}}.
$$
And 
\begin{align*}
\|\Proj{j}{f_e}\|_\cH^2 \asymp 
    \sum_{i=1}^j \<f_e, \bar\varphi_i\>_2^2 i^{b-b'+b'+1} \le 
    j^{b-b'+1} \sum_{i=1}^j \<f_e, \bar\varphi_i\>_2^2 i^{b'} 
    &=
    j^{b-b'+1} \|\Proj{j}f_e\|_{[\tilde\cH]}  \\
    &\le
    m^{\frac{b-b'+1}{b+1}} \|f_e\|_{[\tilde\cH]}^2  = m^{\frac{1}{b+1}},
\end{align*}
Thus, $\|\tilde f_h\|^2_\cH \le 2(\|f_h\|^2_\cH + \|\Proj{j}{f_e}\|_\cH^2)
\lesssim m^{\frac{1}{b+1}}$. 
(ii): Pick any $b'\in (b_{emb}, b)$. Then $[\tilde\cH]$ is equivalent to an RKHS with a bounded reproducing kernel (Lemma~\ref{lem:power-space}), so $\|f_e\|\lesssim\|f_e\|_{[\tilde\cH]}\lesssim 1$, and 
$\Theta_{m,b'}$ defined above satisfies \eqref{eq:valid-sieve}. 
\end{proof} 

Finally, the $b'=0$ case of Lemma~\ref{lem:preliminary-sieve} will be used separately, so we list it below:

\begin{corollary}\label{corr:L2-sieve}
Let $\tilde k$ satisfy Assumptions~\ref{ass:k-std}, and the eigenvalues of its integral operator decay at $\tilde \lambda_i\asymp i^{-\tilde b}$. Let $\tilde W$ be a GP with covariance kernel $\tilde k$. Then there exists $C_1,C_2>0$ such that for all $m>0$, we have 
$$
\Pi\Bigl(
    \Bigl\{\|\tilde W\|_2^2\le C_1 m^{-\frac{\tilde b}{\tilde b+1}}\Bigr\}\Bigr) \ge e^{-C_2 m^{\frac{1}{\tilde b+1}}}. 
$$
\end{corollary}

\subsubsection{Technical Lemmas}

It is helpful to keep the following quantities in mind: the squared critical radius 
$
    \delta_n^2 \asymp n^{-\frac{b+2p}{b+2p+1}}, 
$ 
for our GP prior 
$
\tau_n^2 = n^{-\frac{b}{b+1}},
$
and for $m = [n^{\frac{b+1}{b+2p+1}}]$, we have $
\tau_m^2 \asymp n^{-\frac{b}{b+2p+1}}\asymp \epsilon_n^2,  m\tau_m^2 = n^{\frac{1}{b+2p+1}} \asymp n\delta_n^2.
$

In the following lemma, recall the definitions 
$$
\Proj{j}{f} := \sum_{i=1}^j \<f,\bar\varphi_i\>_2 \: \bar\varphi_i, %
\Proj{>j}{f} := f - \Proj{j}{f}.
$$
And by Lemma~\ref{lem:power-space}, $\{\bar\varphi_i\}$ also corresponds to the Mercer representation of $k_\alpha$ when $\alpha\ge \frac{b_{emb}}{b+1}$.
\begin{lemma}\label{lem:sobolev}
Let $\cH$ satisfy Assumptions~\ref{ass:k-std},~\ref{ass:H1}, $[\tilde\cH] = [\cH]^{\alpha}$ for some $\alpha\in[0,1]$, $\tilde b := (b+1)\alpha$, and  
$\tilde \lambda_i := \lambda_i^{\alpha} \asymp i^{\tilde b}$. 
Then for any $f\in[\tilde\cH]$, we have 
\begin{inplaceEnumerate}
\inplaceItem
$
\Proj{j}{f} = \sum_{i=1}^j \<f, \sqrt{\tilde\lambda_i}\bar\varphi_i\>_{[\tilde\cH]} \: \sqrt{\tilde\lambda_i}\bar\varphi_i, 
$
\inplaceItem\label{item:sobolev-ortho-proj}
$\Proj{j}{\cdot}$ equals the orthogonal projection onto the subspace spanned by 
$\{\tilde\lambda_i^{1/2}\bar\varphi_i: i\le j\}$, 
the top $j$ scaled orthonormal basis functions of $[\tilde\cH]$,  %
\inplaceItem\label{item:sobolev-norm}
$$
\|\Proj{>j}{f}\|_2 \lesssim \|f\|_{[\tilde\cH]} \, j^{-\frac{\tilde b}{2}}.
$$
\end{inplaceEnumerate}
\end{lemma}
\begin{proof}
\newcommand{\dbtilde}[1]{\accentset{\approx}{#1}}
(i) follows by the definition of the $[\cH]^\alpha$ norm, and the fact that $\{\bar\varphi_i\}$ is an ONB in $L_2(P(dx))$. 
(ii) follows from (i) and the fact that 
$\{\sqrt{\tilde\lambda_i}\bar\varphi_i\}$ forms an ONB in $[\tilde\cH]$. 
Now (iii) follows because %
\begin{align*}
    \|f - \Proj{j}{f}\|_2^2 &= 
    \sum_{i\ge j} \<f,\sqrt{\tilde\lambda_i}\tilde{\varphi}_i\>_{[\tilde\cH]}^2 \|\sqrt{\tilde\lambda_i}\tilde{\varphi}_i\|_2^2  %
    \le \|\Proj{>j}{f}\|_{[\tilde\cH]}^2 \cdot \tilde\lambda_j \|\tilde\varphi_j\|_2^2 \lesssim 
    \|f\|_{[\tilde\cH]}^2 j^{-\tilde b}.
\end{align*}
\end{proof}

The following lemma applies to, e.g., $f\in\Theta_m$, and $f-f_0$ by Assumption~\ref{ass:approx-relaxed}. 
\begin{lemma}\label{lem:refined-sieve}
For $m = [n^{\frac{b+1}{b+2p+1}}]$ and any $f = \tilde f_h + \tilde f_e$ where 
\begin{equation}\label{eq:refined-decomp-in}
\|\tilde f_h\|_\cH^2\le m\tau_m^2, ~~
\|\tilde f_e\|_2^2\le \tau_m^2, ~~
\|\tilde f_e\| \le 1, %
\end{equation}
we have the decomposition $f=f_h+f_e$ where 
\begin{equation}\label{eq:refined-decomp-out}
\|f_h\|_\cH^2 \lesssim m\tau_m^2, ~~
\|f_e\|_2^2\le \tau_m^2, ~~
\|E f_e\|_2^2\lesssim \delta_n^2, ~~
\|f_e\|\lesssim 1, %
\end{equation}
where the constants in $\lesssim$ only depend on $k_x$ and the constant in Assumption~\ref{ass:link}.
\end{lemma}
\begin{proof}
\SkipNOTE{Recall $\{(\lambda_i, \bar\varphi_i)\}$ denote the Mercer basis of $k_x$, and $\lambda_i\asymp i^{-(b+1)}, \|\bar\varphi_i\|_2=1$.} 
As $\tilde f_e \in L_2(P(dx))$, for any finite $j$, $\Proj{j}{\tilde f_e}$ is well-defined and in $\cH$, and for $m' = [m^{\frac{1}{b+1}}] \asymp n^{\frac{1}{b+2p+1}}$, we have 
$$
\|\Proj{m'}{\tilde f_e}\|_\cH^2 = \sum_{i=1}^{m'} \lambda_i^{-1} |\<\tilde f_e, \bar\varphi_i\>_2|^2
\le \lambda_{m'}^{-1} \|\tilde f_e\|_2^2 \lesssim (m')^{b+1}\tau_m^2 
\asymp m\tau_m^2,
$$
where the last inequality follows from $b > 1$. Thus, for 
$$
f_h := \tilde f_h + \Proj{m'}{\tilde f_e}, ~~
f_e := \tilde f_e - \Proj{m'}{\tilde f_e}
,
$$
we have $\|f_h\|_\cH^2\lesssim m\tau_m^2$ as advertised. 
By definition of $\mrm{Proj}_j$ 
we have 
\begin{align*}
\|f_e\|_2 \le \|\tilde f_e\|_2 \le \tau_m, ~~
\|E f_e\|_2^2 \overset{(a)}{\lesssim} (m')^{-2p} \|f_e\|_2^2 
\lesssim n^{-\frac{2p}{b+2p+1}} n^{-\frac{b}{b+2p+1}} = n^{-\frac{b+2p}{b+2p+1}}.
\end{align*}
where (a) follows from Assumption~\ref{ass:link}. 
Finally, by \eqref{eq:interp-ineq}, %
\begin{align*}
\|f_e\| &\le \|\tilde f_e\| + \|\Proj{m'}{\tilde f_e}\| 
\lesssim \cancel{1+}
    \|\Proj{m'}{\tilde f_e}\|_\cH^{\frac{b}{b+1}} \|\Proj{m'}{\tilde f_e}\|_2^{\frac{1}{b+1}}
\lesssim 
m^{\frac{1/2}{b+1}\cdot \frac{b}{b+1}} \cdot m^{\frac{-b/2}{b+1}\cdot \frac{1}{b+1}} 
= 1.
\end{align*}
\end{proof}

We provide the following form of Bernstein's inequality \citep[e.g.,][Theorem 6.12]{steinwart2008support} for reference. \SkipNOTE{Set $\tau=\eta n\sigma^2/B^2$.}
\begin{lemma}\label{lem:bernstein}
Let the random variable $\bl$ be bounded by $B_l$ and have variance $\mrm{Var}(\bl)\le \sigma_l^2$. Let $l_1,\ldots,l_n$ be i.i.d.~copies of $\bl$. Then for all $\eta>0$, 
$$
\PP\left(\left|\frac{1}{n}\sum_{i=1}^n l_i - \EE\bl\right|\ge \left(
    \sqrt{2\eta}+\frac{2\eta}{3}
\right)\frac{\sigma_l^2}{B_l}\right) \le \exp\left(-\frac{n\sigma_l^2}{B_l^2}\cdot\eta\right). 
$$
\end{lemma}

\subsection{Marginal Quasi-Likelihood Bound}\label{app:qloglh}

We now bound the marginal likelihood, which will appear as a denominator in the decomposition of the posterior mass. We proceed in three steps: 
\begin{enumerate}
    \item 
Lemma~\ref{lem:E-smallball} creates a ``small-ellipsoid'' subset of the parameter space, on which $\|E(f-f_0)\|_2$ can be controlled. $\|E(f-f_0)\|_2$ connects to the log quasi-likelihood $\ell_n$, so this is similar to the (shifted) small-ball subset $\{\|W-f_0\| \le \tau_n\}$ in the analysis of GPR. 
\item 
Lemma~\ref{lem:den-subset} refines the subset so that other quantities needed in Step 3 are bounded. 
\item 
Proposition~\ref{lem:den} further refines the subset in a data-dependent way. On the final subset we can control the estimation error of $\ell_n$, and thus the marginal likelihood can be lower bounded. 
\end{enumerate}

\begin{lemma}\label{lem:E-smallball}
Let $W\sim \Pi$ be a draw from the prior Gaussian measure, $f_h\in\cH$ be an arbitrary function. Then %
\begin{equation}
\log \Pi(\{
    \|E (W-f_h)\|_2^2 \le n^{-\frac{b+2p}{b+2p+1}} 
\}) \lesssim 
 -n^{\frac{1}{b+2p+1}}, \label{eq:small-ball-condexp-ub} \\ 
\end{equation}
Moreover, if we have $\|f_h\|_\cH^2 \lesssim n^{\frac{1}{b+2p+1}}$, 
then
\begin{align}
\log \Pi(\{
    \|E (W-f_h)\|_2^2 \le n^{-\frac{b+2p}{b+2p+1}}, ~
    \|W-f_h\|_2^2 \le n^{-\frac{b}{b+2p+1}}
\}) &\gtrsim 
 -n^{\frac{1}{b+2p+1}}. \label{eq:small-ball-condexp}
\end{align}
\end{lemma}
\begin{proof} 
Let $W = \sum_{i=1}^\infty \lambda_i^{1/2}\epsilon_i \varphi_i$ be the Karhunen-Lo\`eve expansion, where $
(\lambda_i, \varphi_i)
$
are the Mercer decomposition of $k_x$. 
We restrict ourselves to the probability-1 event that the series converge in $L_2(P(dx))$ \citep[Theorem 3.1]{steinwart_convergence_2019}. 
For any $l\in\mb{N}$, by Assumption~\ref{ass:link}~\eqref{eq:link-cond} we have
\begin{align*}
\|  E\Proj{l}{(W-f_h)}\|_2^2 &\le C_2 \sum_{j=1}^l j^{-2p}\<W-f_h, \bar\varphi_j\>_2^2
=  C_2 \sum_{j=1}^l j^{-2p} (\epsilon_j \lambda_j^{1/2} - \<f_h, \bar\varphi_j\>_2)^2, \\ 
\|E\Proj{>l}{(W-f_h)}\|_2^2 &\le C_2 (l+1)^{-2p}\|\Proj{>l}{(W-f_h)}\|_2^2, \\ 
\|E(W-f_h)\|_2 &\le 
\|E\Proj{l}{(W-f_h)}\|_2 + \|E\Proj{>l}{(W-f_h)}\|_2. 
\end{align*}
In light of the second display above, the term 
$\|E\Proj{>l}{(W-f_h)}\|_2$ vanishes as $l\rightarrow\infty$, for any fixed $W$ in the $\Pi$-probability 1 subset, and
\begin{align*}
\|E (W-f_h)\|_2^2 &\le \lim_{l\rightarrow\infty} \|E \Proj{l}{(W-f_h)}\|_2^2 
\le C_2 \sum_{j=1}^\infty j^{-2p} (\epsilon_j \lambda_j^{1/2} - \<f_h, \bar\varphi_j\>_2)^2 \\ 
&= C_2 \sum_{j=1}^\infty (\epsilon_j j^{-p}\lambda_j^{1/2} - \<\tilde f_h, \bar\varphi_j\>_2)^2,
\end{align*}
where 
$
\tilde f_h := \sum_{j=1}^\infty j^{-p}\<f_h, \bar\varphi_j\>_2\bar\varphi_j.
$
Consider the kernel $\tilde k(x,x') = \sum_{j=1}^\infty j^{-2p}\lambda_j \bar\varphi_j(x) \bar\varphi_j(x')$, and the induced RKHS $\tilde\cH$. 
Then the process 
$\tilde W := \sum_{j=1}^\infty j^{-p}\lambda_j^{1/2} \epsilon_j \bar\varphi_j$ is a standard Gaussian process defined with $\tilde k_x$, where the series in the RHS converges in $L_2$ a.s.~\citep{steinwart_convergence_2019}; 
and the RHS in the last display equals $C_2 \|\tilde W - \tilde f_h\|_2^2$. 
Thus we have, for any $\epsilon>0$, $
\{\|E(W - f_h)\|_2 \le \epsilon\} \supset 
\{C_2 \|\tilde W - \tilde f_h\|_2 \le \epsilon\}.
$ 
Similarly, by Assumption~\ref{ass:link}~\eqref{eq:reverse-link-cond}, we can reverse the inequalities above with a different $C_2'$. Combining the two results we have, for any $\epsilon>0$ and \emph{all $\|f_h\|_\cH^2\lesssim n\delta_n^2$}, 
\begin{equation}\label{eq:small-ball-1}
\{C_2 \|\tilde W - \tilde f_h\|_2 \le \epsilon\}\subset
\{\|E(W - f_h)\|_2 \le \epsilon\} \subset 
\{C_2' \|\tilde W - \tilde f_h\|_2 \le \epsilon\}.
\end{equation}

For $\delta_n^2\asymp n^{-\frac{b+2p}{b+2p+1}}$, we have 
\begin{align}
\Pi(\{\|\tilde W - \tilde f_h\|_2\le \delta_n\}) 
\overset{(a)}{\ge} 
e^{-\frac{1}{2}\|\tilde f_h\|_{\tilde\cH}^2} \Pi(\{
    \|\tilde W\|_2 \le \delta_n/C
    \})
\overset{(b)}{\ge} e^{-C' n\delta_n^2}.\label{eq:shifted-small-ball-prob}
\end{align}
where (a) can be found in, e.g., \citet[Lemma I.28]{ghosal2017fundamentals}, and (b) holds because 
$
\|\tilde f_h\|_{\tilde \cH}^2 %
= \|f_h\|_{\cH}^2 \asymp n^{\frac{1}{b+2p+1}},
$ 
and Corollary~\ref{corr:L2-sieve} applied to $(\tilde\cH,b+2p)$.\footnote{
The application is valid since $\tilde k_x$ is obviously $L_2$ universal, although technically this is not necessary for the small-ball probability bound to hold. 
} 
By the same argument we have, \emph{for all $f_h\in\cH$,}
\begin{equation}\label{eq:shifted-small-ball-prob-ub}
 \Pi(\{\|\tilde W - \tilde f_h\|_2\le \delta_n\}) 
\le 
e^{-\frac{1}{2}\|\tilde f_h\|_{\tilde\cH}^2} \Pi(\{
    \|\tilde W\|_2 \le \delta_n/C
    \})
\le  \Pi(\{
    \|\tilde W\|_2 \le \delta_n/C
    \})
\le e^{-C'' n\delta_n^2}.
\end{equation}
Now \eqref{eq:small-ball-condexp-ub} is proved by combining the above display and the second relation in \eqref{eq:small-ball-1}. %

It now remains to prove \eqref{eq:small-ball-condexp}. 
Fix $
m = [n^{\frac{b+1}{b+2p+1}}],
m' = [n^{\frac{1}{b+2p+1}}].
$ Then by \eqref{eq:shifted-small-ball-prob}, 
\begin{align*}
\log\Pi(\Theta_{dz}) &:= \log\Pi(\Theta_{dzl} \cap \Theta_{dzh}) \\ 
&:= 
\log\Pi(\{
    \|\Proj{m'}{(\tilde W- \tilde f_h)}\|_2 \le \delta_n/(2C_2), ~
    \|\Proj{>m'}{(\tilde W- \tilde f_h)}\|_2 \le \delta_n/(2C_2)
\}) \\ &\ge 
\log\Pi(\{
    \|\tilde W- \tilde f_h\|_2 \le \delta_n/(2C_2)
\})
 \gtrsim  -n\delta_n^2.
\end{align*}
And by \eqref{eq:small-ball-1}, on $\Theta_{dz}$ we have $$
\max\{\|E(\Proj{m'}{(W-f_h)})\|_2, \|E(W-f_h)\|_2\} \le \delta_n.
$$
By Assumption~\ref{ass:link}~\eqref{eq:reverse-link-cond}, this implies that 
$$
\|\Proj{m'}(W-f_h)\|_2^2 \le (m')^{2p} \delta_n^2 = n^{-\frac{b}{b+2p+1}} \asymp \tau_{m}^2.
$$
Moreover, 
\begin{align*}
\log \Pi(\Theta_{dh}) &:= 
\log \Pi(\|\Proj{>m'}(W - f_h)\|_2 \le \tau_{m}) \ge 
\log \Pi(\|W - f_h\|_2 \le \tau_{m}) %
\overset{(a)}{\gtrsim} -n^{\frac{1}{b+2p+1}}, 
\end{align*}
where (a) follows the same reasoning as \eqref{eq:shifted-small-ball-prob}. 
Combining the last three displays, we can see that the event in \eqref{eq:small-ball-condexp} is implied by 
$\Theta_{dh}\cap \Theta_{dz}$, after an inconsequential scaling. 
So it remains to lower bound $\Pi(\Theta_{dh}\cap\Theta_{dz})$. %

On $\Theta_{dh}$ we have $$
\begin{aligned}
\|\Proj{>m'}(\tilde W - \tilde f_h)\|_2^2 &= \sum_{j>m'} j^{-2p}(\epsilon_j \lambda_j^{1/2} - \<f_h, \varphi_j\>_2)^2 \le (m')^{-2p}\sum_{j>m'} (\epsilon_j \lambda_j^{1/2} - \<f_h, \varphi_j\>_2)^2  \\
&= (m')^{-2p} \|\Proj{>m'}{(W-f_h)}\|_2^2 \le m'^{-2p} \tau_{m}^2 
 \lesssim n^{-\frac{2p}{b+2p+1}-\frac{b}{b+2p+1}} \asymp \delta_n^2.
\end{aligned}
$$
Thus, after a trivial rescaling we have $\Theta_{dh}\subset \Theta_{dzh}$ (recall $W$ and $\tilde W$ are coupled). 
Moreover, both $\Theta_{dzh}$ and $\Theta_{dh}$ only depend on $\{\epsilon_{j}: j>m'\}$ and are independent of $\Theta_{dzl}$. Therefore, 
$$
\log\PP(\Theta_{dh}\cap\Theta_{dz}) = \log\PP(\Theta_{dzl}) + \log\PP(\Theta_{dh}) \gtrsim -n\delta_n^2
$$ which concludes the proof.
\end{proof}

\begin{lemma}\label{lem:den-subset}
Let $\delta_n^2 := n^{-\frac{b+2p}{b+2p+1}}, \epsilon_n^2 := n^{-\frac{b}{b+2p+1}}$. 
There exist constants $C_1, C_2, C_3, C_4>0$ s.t. %
$$
\Pi(\Theta_{d0}) := 
\Pi(\{
    \|E(f-f_0)\|_2 \le C_1 \delta_n, ~
    \|f-f_0\|_2 \le C_2 \epsilon_n, ~
    \|f-f_0\| \le C_3
\}) \ge \exp\!\left(-C_4 n \delta_n^2\right).
$$
\end{lemma}
\begin{proof}
For $m' = [n^{\frac{1}{b+2p+1}}]$, 
by Assumption~\ref{ass:approx-relaxed} we can decompose $f_0 = \tilde f_h + \tilde f_e$ where 
$\tilde f_h, \tilde f_e$ satisfies \eqref{eq:refined-decomp-in}. 
Invoking Lemma~\ref{lem:refined-sieve} we obtain $f_h+f_e = f_0$ where $\|f_h\|_\cH\lesssim \|\tilde f_h\|_\cH, \|E f_e\|_2 \lesssim \delta_n, \|f_e\|_2\lesssim \epsilon_n$. 
Thus, by the triangle inequality, the parameter subset created by applying Lemma~\ref{lem:E-smallball} to $f_h$ satisfies 
$$
\|E(f-f_0)\|_2 \le C_1 \delta_n, ~
\|f-f_0\|_2 \le C_2 \epsilon_n
$$
and has the advertised prior probability.  It remains to deal with $\|f-f_0\|$. 

Let $m=[n^\frac{b+1}{b+2p+1}]$ and 
$\Theta_{d0}$ be the intersection of the aforementioned subset with 
$
C \Theta_{m} := \{f: C^{-1}f\in\Theta_m\},
$
where $C, C'\ge 1$ are large constants, and $\mb{B}_1$ denotes the $L_2$ unit norm ball. 
Since $1-\Pi(C\Theta_{m}) \le e^{-C'' n\delta_n^2}$, 
\SkipNOTE{
    I can't provide the reference that scaling by $C$ scales the probability by this, but we can consider $m'/\log n$ in any case.
}
we can choose $C$ so that $\Theta_{d0}$ still has the advertised prior probability. 
Now, by Corollary~\ref{corr:valid-sieve}, %
for any $f\in\Theta_{d0}\subset C\Theta_m$, %
we have $$
f=\bar f_h+\bar f_e, ~~~~\text{where}~~~~ 
\|\bar f_h\|_\cH^2\lesssim n^{\frac{1}{b+2p+1}},~~ 
\|\bar f_e\|_2^2 \lesssim n^{-\frac{b}{b+2p+1}},~~
\|\bar f_e\|^2\lesssim 1.
$$%
As by assumption $f_0$ admits a similar decomposition \eqref{eq:approx-relaxed-orig}, we have 
$f-f_0 = f_h' + f_e'$, where 
$\|f_h'\|_\cH^2\lesssim n^{\frac{1}{b+2p+1}}, 
\|f_e'\|_2^2\lesssim n^{-\frac{b}{b+2p+1}}\asymp\epsilon_n^2, 
\|f_e'\|\lesssim 1$. 
We also have  
$\|f_h'\|_2 \le \|f-f_0\|_2 + \|f_e'\|_2 \lesssim \epsilon_n$. Thus, by Lemma~\ref{lem:interpolation}, 
$$
\|f-f_0\|^2 \lesssim \|f_h'\|^2 + \|f_e'\|^2 \lesssim 
\|f_h'\|_\cH^{\frac{2b}{b+1}} \|f_h'\|_2^{\frac{2}{b+1}} + \|f_e'\|^2 \lesssim 
n^{\frac{1}{b+2p+1}\cdot \frac{b}{b+1}} n^{\frac{-b}{b+2p+1}\cdot \frac{1}{b+1}} + 1 \le 2.
$$
\end{proof}

\begin{proposition}[marginal likelihood bound]\label{lem:den}
Let $\lambda=1, \effnu\ge 2C\delta_n^2$ where $C>0$ is a sufficiently large constant. Then 
there exists constants $C_{1,\ldots,5}>0$, and a $\data$-measurable event $E_{den}$ with probability $\rightarrow 1$, on which 
\begin{equation}\label{eq:den-claim-1}
\Pi(\{f\in\Theta_{d0}: \ell_n(f) \le C_1 \|E(f-f_0)\|_2^2 + C_2 \delta_n^2\}) \ge \frac{1}{2}\Pi(\Theta_{d0}).
\end{equation}
Consequently, on $E_{den}$ we have $
\Pi(\{f: \ell_n(f) \le C_3 \delta_n^2\}) \ge e^{-C_4 n\delta_n^2}
$, and
\begin{equation}\label{eq:marg-lh}
\int \exp\left(-\frac{n}{\lambda} \ell_n(f)\right) \Pi(df) \ge \exp(-C_5 n\delta_n^2).
\end{equation}
In the above, $C_2,C_3,C_5\lesssim B^2$. 
\end{proposition}
\begin{proof} We proceed in two steps.

\underline{(Step 1)} In Step 2 below we will show that there exists $C_1,C_2>0$ and a sequence $\eta_n\rightarrow 0$ s.t.~for any fixed $f\in\Theta_{d0}$, 
\begin{equation}\label{eq:den-fixed-f}
\PP(I(f, \data)) := 
\PP(\{
\ell_n(f) \le C_1 \|E(f-f_0)\|_2^2 + C_2 \delta_n^2 
    \}) \ge 1 - \eta_n.
\end{equation}
We claim that if \eqref{eq:den-fixed-f} is true, 
\eqref{eq:den-claim-1} will hold with $\data$-probability greater than $1 - 4\eta_n$, and 
\eqref{eq:marg-lh} by the definition of $\Theta_{d0}$ in Lemma~\ref{lem:den-subset}; 
hence the proof will complete. 
To prove this claim, suppose by contrary that \eqref{eq:den-fixed-f} holds, yet 
\begin{equation}\label{eq:den-c-h}
\PP(\{\text{\eqref{eq:den-claim-1} holds}\}) \le 1-4\eta_n.
\end{equation}
Define %
$\Pi_{dr}(df) := \Pi(df\cap \Theta_{d0}) / \Pi(\Theta_{d0})$. 
By Fubini's theorem and \eqref{eq:den-fixed-f}, 
$$
\PP(\Pi_{dr}(I(f,\data))) = \Pi_{dr}(\PP(I(f, \data))) \ge 1-\eta_n.
$$
Yet \eqref{eq:den-c-h} would imply 
\begin{align*}
 \PP(\Pi_{dr}(I(f, \data))) &= 
\PP(\Pi_{dr}(I(f, \data)) \cdot \mbf{1}_{\{\text{\eqref{eq:den-claim-1} holds}\}}) + 
\PP(\Pi_{dr}(I(f, \data)) \cdot \mbf{1}_{\{\text{\eqref{eq:den-claim-1} does not hold}\}})    \\ 
&\le 1\cdot (1-4\eta_n) + \frac{1}{2}\cdot 4\eta_n = 1-2\eta_n,
\end{align*} 
a contradiction. 

\underline{(Step 2)} It remains to prove \eqref{eq:den-fixed-f}. We follow the argument in \cite{dikkala_minimax_2020}. 
Denote by $\hat{\delta}_n$ the critical radius of the empirical local Radamacher complexity of $3\cI_1$.\SkipNOTE{\citep{wainwright2019high}, with $R=\sigma=B$ in their notation} 
Then for some universal $c>0$, $\PP_{\data}(\hat{\delta}_n\le c\delta_n) \rightarrow 1$ \citep[p.~455]{wainwright2019high}. 
And by \citet[Theorem 14.1 and 14.20]{wainwright2019high} applied to $g/\|g\|_\cI$, 
there exists $c'>0$ such that with probability greater than 
$1-c' e^{-c' n\delta_n^2} =: 1-\eta_n \rightarrow 1$, 
we have, for fixed $f\in\Theta_{d0}$ and all $g\in\cI$, 
\begin{align*}
\|g\|_{2,n}^2 &\ge  \frac{1}{2}\|g\|_2^2 - C \delta_n^2 (1+\|g\|_\cI^2), \\ 
\Psi_n(f,g) &\le 
\Psi(f,g) + 10LC\delta_n (\|g\|_2 + \delta_n(1+\|g\|_\cI)).
\end{align*}
where $L =\|f-f_0\|+\max_{i=1}^n |y_i-f_0(x_i)|\lesssim 1+B,$ \SkipNOTE{the sup norm bound follows by the last lemma} 
$\Psi_n(f,g) := \sum_{i=1}^n g(z_i) (f(x_i) - y_i)$, 
$\Psi(f,g) := \EE(g(\bz)(f(\bx)-\by))$, and $C=\sup_{\|g\|_\cI=1}\|g\|$. Define $C_n := (20L + 2)C$. 
Define $E_{den}$ as the event when the above displays hold. 
On this event we have, for $\effnu>2C\delta_n^2$,  
\begin{align*}
\ell_n(f) &= 
\sup_{g\in\cI} 2\Psi_n(f, g) - \|g\|_{n,2} - \effnu \|g\|_\cI^2 \\ 
&\le 
\sup_g 2\Psi(f, g) + 20L(\delta_n\|g\|_2 + \delta_n^2(1+\|g\|_\cI)) - \frac{1}{2}\|g\|_2^2 + C\delta_n^2(1+\|g\|_\cI^2) 
- \effnu\|g\|_\cI^2  
\\ 
&=
\sup_g 2\Psi(f, g) - \frac{1}{4}\|g\|_2^2
    +  C_n\delta^2  %
 + \left(C_n \delta^2\|g\|_\cI -(\effnu-C\delta_n^2)\|g\|_\cI^2\right)
    + \left(C_n\delta\|g\|_2 - \frac{1}{4}\|g\|_2^2\right) \\ 
&\overset{(a)}{\le} \sup_{g\in\cI} \Psi(f, g) - \frac{1}{4}\|g\|_2^2 
    +  C_n'^2\delta^2 
\le \sup_{g\in L_2}\Psi(f, g) - \frac{1}{4}\|g\|_2^2 + C_n'^2\delta^2  \\
&\overset{(b)}{=}
2\|E(f-f_0)\|_2^2 + C_n'^2\delta^2.
\end{align*}
In the above (a) follows from the inequality $a\|x\|_s - b\|x\|_s^2\le \frac{a^2}{4b}$, where $\|\cdot\|_s$ is an arbitrary norm; and (b) by definition of $\Psi$ and properties of the conditional expectation. 
Finally, we note that $C_n'^2 = C_n + C_n^2(1+1/4C) \lesssim B^2$. 
\end{proof}

\subsection{Proof of Theorem~\ref{thm:L2}}\label{app:proof-thm-L2}

At a high level, we will follow the posterior contraction framework and bound the expectation of the posterior mass 
$$
\PP_{\data}\Pi(\err\givendata) = \PP_{\data} \frac{\int_{\err} \exp(-\lambda^{-1}n\ell_n(f))\Pi(df)}{\int \exp(-\lambda^{-1}n\ell_n(f))\Pi(df)}, 
$$
where $\err$ denotes the parameter set in the theorem, that is, either $\{\|f-f_0\|_2\ge M\epsilon_n\}$ or $\{\|E(f-f_0)\|_2\ge M\delta_n\}$. 
The denominator has been bounded in Appendix~\ref{app:qloglh}, where we sketch the proof at the beginning of the subsection; the numerator will be addressed in this subsection. For this purpose, 
\begin{enumerate}
    \item We first decompose the posterior mass further, in \eqref{eq:contraction-decomposition-1}-\eqref{eq:contraction-reduced-goal-intermediate} below, where we pick out some events on which estimation of the quasi-likelihood fails, and a small parameter set $\Theta_m^c$ consisting of ``bad'' parameters. As we show below, these events have vanishing contribution to the above display, and will not affect the final result. 
    \item Lemma~\ref{lem:num-step-2} below shows that on $\Theta_m$, a large error in the $L_2(P(dx))$ norm is equivalent to a large error in the $\|E(\cdot)\|_{L_2(P(dz))}$ norm. This shows the equivalence of \eqref{eq:L2-rate-maintext} and \eqref{eq:L2-rate-CE-maintext}, and enables us to lower bound the quasi-likelihood which (for the most part) estimates $\|E(f-f_0)\|_2$. 
    \item The proof proceeds by bounding various estimation errors. 
\end{enumerate} 
Throughout the proof we rely on the \emph{sieve set} $\Theta_m$ consisting of ``well-behaved'' functions in the prior support; its properties have been established in Appendix~\ref{app:gp-sieve}. 
We also refer readers to the additional notations and conventions in Appendix~\ref{app:proof-conventions}.

\begin{lemma}\label{lem:num-step-2}
There exists $M_0,C>0$ such that for all $n\in\mb{N}$,  
$m=[n^{\frac{b+1}{b+2p+1}}]$, and $f\in\Theta_m$, 
\begin{align}
\|f-f_0\|_2\ge M_0\epsilon_n ~~\text{only if}~~
    \|E(f-f_0)\|_2 \ge \frac{\|f-f_0\|_2}{C\epsilon_n} \delta_n. \label{eq:err-implication} \\ 
\|E(f-f_0)\|_2 \ge M_0 \delta_n ~~\text{only if}~~ 
    \|f-f_0\|_2 \ge \frac{\|E(f-f_0)\|_2}{C\delta_n} \epsilon_n. \label{eq:err-implication-2}
\end{align}
The constant $C$ is determined by Assumption~\ref{ass:link}. 
\end{lemma}

\begin{proof}
Since $f\in\Theta_m$ and $f_0$ satisfies \eqref{eq:approx-relaxed-orig}, we can invoke
Lemma~\ref{lem:refined-sieve} on $\Delta f := f-f_0$, showing that there exists $\DeltaFApprox\in\cH$ such that
\begin{equation}\label{eq:delta-f-approx}
\|\DeltaFApprox\|_{\cH}^2 \lesssim n\delta_n^2, ~~ \|E(\Delta f-\DeltaFApprox)\|_2^2\lesssim \delta_n^2, ~~ 
\|\Delta f-\DeltaFApprox\|_2^2 \le %
C \epsilon_n^2, ~~ 
\|\Delta f-\DeltaFApprox\|^2 \le C, 
\end{equation} 
 where $C>0$ is a constant. Introduce the notations 
\begin{equation}\label{eq:thm-L2-notations-2}
j := \bigl[n^{\frac{1}{b+2p+1}}\bigr], ~~
g := E\Delta f, ~~
\tilde g:=E(\DeltaFApprox), ~~
g_j := E(\Proj{j}{\DeltaFApprox}).
\end{equation}

From \eqref{eq:delta-f-approx} we can see that for sufficiently large $M_0$ and $\|f-f_0\|_2\ge {M_0}\epsilon_n$, 
\begin{align*}
\|\DeltaFApprox\|_2 &\ge \|\Delta f\|_2 - \|\Delta f-\DeltaFApprox\|_2 
\ge \|\Delta f\|_2 - \sqrt{C}\epsilon_n\ge 
\frac{1}{2}\|\Delta f\|_2, \\  
\|\Proj{j}{\DeltaFApprox}\|_2^2 &= \|\DeltaFApprox\|_2^2 - \|\Proj{>j}{\DeltaFApprox}\|_2^2 \overset{(a)}{\ge} 
\frac{1}{4}\|\Delta f\|_2^2 -  \|\DeltaFApprox\|_\cH^2 j^{-(b+1)}
\SkipNOTE{\ge \frac{M}{2}\epsilon_n^2 - C n^{\frac{1}{b+2p+1} - \frac{b+1}{b+2p+1}} }
\ge \frac{1}{5} \|\Delta f\|_2^2,
\end{align*}
where (a) follows from Lemma~\ref{lem:sobolev} with $\tilde b = b+1$. And
\begin{align*}
\|E (\Delta f)\|_2^2 &\ge \frac{1}{2}\|E(\DeltaFApprox)\|_2^2 -  \|E(\Delta f-\DeltaFApprox)\|_2^2  \\ 
&\overset{(a)}{\ge} C_1' j^{-2p} \|\Proj{j}{\DeltaFApprox}\|_2^2 -  \|E(\Delta f-\DeltaFApprox)\|_2^2  %
\\ 
&\ge  \frac{1}{5}C_1' j^{-2p} \|\Delta f\|^2_2
-  \|E(\Delta f-\DeltaFApprox)\|_2^2   \\ 
&= C_1'' \epsilon_n^{-2}\|\Delta f\|_2^2\delta_n^2
- C_2 n^{-\frac{b+2p}{b+2p+1}}  
\gtrsim \epsilon_n^{-2}\|\Delta f\|_2^2\delta_n^2. 
\end{align*}
In the above, (a) follows from Assumption~\ref{ass:link}, 
the last inequality holds for $M\ge 2C_2/C_1''$, and we recall $C_2$ is from \eqref{eq:delta-f-approx}. This concludes the proof of \eqref{eq:err-implication}. 

It remains to prove \eqref{eq:err-implication-2}. First, for any $f\in\Theta_m$, we have
\begin{align}
\|g - g_j\|_2^2 &\le 2\|g-\tilde g\|_2^2 + 2\|\tilde g - g_j\|_2^2
    \overset{(a)}{\lesssim} \cancel{\delta_n^2 +} \|\tilde g-g_j\|_2^2
    = \|E(\Proj{>j}{\DeltaFApprox})\|_2^2 \nonumber \\ 
    &\overset{(b)}{\le}  j^{-2p}\|\Proj{>j}{\DeltaFApprox}\|_2^2  
    \overset{(c)}{\le}  j^{-2p} \|\DeltaFApprox\|_\cH^2 j^{-(b+1)} \lesssim 
    n^{\frac{-2p + 1 - (b+1)}{b+2p+1}} = \delta_n^2. \label{eq:g-gj-diff}
\end{align}
where (a) follows by \eqref{eq:delta-f-approx}, (b) by Assumption~\ref{ass:link}, and (c) by Lemma~\ref{lem:sobolev} with $\tilde b=b+1$. 
Thus, for $M_0$ sufficiently large and $f\in\Theta_m$ s.t.~$\|g\|_2=\|E \Delta f\|_2 \ge M_0\delta_n$, we have, by Assumption~\ref{ass:link}, 
\begin{align*}
\|\DeltaFApprox\|_2 &\gtrsim j^{p} \|g_j\|_2 \ge j^{p} (\|g\|_2-\|g-g_j\|_2) 
\ge \frac{\|g\|_2}{2\delta_n} \epsilon_n, \\ 
\SkipNOTE{$\|g-g_j\|_2\le \delta_n$ is Eq.~\eqref{eq:g-gj-diff}.}
\|\Delta f\|_2 &\ge \|\DeltaFApprox\|_2 - \|\Delta f-\DeltaFApprox\|_2 \overset{(a)}{\gtrsim} 
\left(\frac{\|g\|_2}{2\delta_n}-\sqrt{C}\right) \epsilon_n \ge 
\frac{\|g\|_2}{3\delta_n} \epsilon_n,
\end{align*}
where (a) also follows from \eqref{eq:delta-f-approx}. This completes the proof.
\end{proof}

We now complete the proof of the theorem, which we restate below with minor changes in notations. 

\begingroup
\def\thetheorem{\ref{thm:L2}}
\begin{theorem}
Fix $\lambda=1, \effnu = C \delta_n^2 \asymp n^{-\frac{b+2p}{b+2p+1}}$ where $C>0$ is a large constant. 
Then there exists $M>0$ such that for $\epsilon_n^2 = n^{-\frac{b}{b+2p+1}}$, we have 
\begin{align}
\PP_{\data} \Pi\!\left(
    \{f: \|f-f_0\|_2^2\ge M \epsilon_n^2\}
    \givendata
\right) &\rightarrow 0. \label{eq:L2-rate} \\ %
\PP_{\data} \Pi\!\left(
    \{f: \|E(f-f_0)\|_2^2\ge M \delta_n^2\}
    \givendata
\right) &\rightarrow 0. \label{eq:L2-rate-CE}
\end{align}
Furthermore, if Assumption~\ref{ass:std} is relaxed so that only $\by-f_0(\bx)$ is subgaussian will hold, the above two equations will hold with $M$ replaced by $M_n$, where $M_n$ is any slowly growing sequence such that 
$
\lim_{n\to\infty} \frac{M_n}{\log n} = \infty. 
$
\end{theorem}
\addtocounter{theorem}{-1}
\endgroup

\subsubsection{Proof of Theorem~\ref{thm:L2}, Bounded Noise Case}\label{app:proof-L2-main}

Before we start, 
note that throughout the proof the constants hidden in $\lesssim,\asymp,\gtrsim$ will not depend on $M$ or $B$. We use $\err$ to denote the function set in either \eqref{eq:L2-rate} or \eqref{eq:L2-rate-CE}. 

\parspace
\underline{{(Step 1)}} 
Let $E_n(\data) := E_{den}$ which is defined in Proposition~\ref{lem:den}. 
Consider the decomposition 
\begin{align}
\PP_{\data} \Pi\!\left(
    \err\givendata
\right) 
&\le 
\PP_{\data}(E_n^c) + \PP_{\data}\!\left[
\mbf{1}_{E_n}\cdot \frac{\tilde\Pi(\err\givendata)}{\tilde\Pi(\Theta\givendata)}
\right],\label{eq:contraction-decomposition-1}
\end{align}
where $\tilde\Pi(A\givendata) := \int_A \exp(-\lambda^{-1}n\ell_n(f))\Pi(df)$ is the unnormalized posterior measure. 
By Proposition~\ref{lem:den}, on $E_n$ the denominator above is greater than $\exp(-C_{den} n\delta_n^2)$, for some $0<C_{den}\lesssim B^2$. So 
it remains to show that, there exists some $C_{num}>C_{den}$ such that 
when $M$ is sufficiently large, 
\begin{equation}\label{eq:contraction-reduced-form-goal}
\PP_{\data}\left[\mbf{1}_{E_n}\tilde\Pi(\err\givendata)\right] \le \exp\left(-C_{num} n \delta_n^2\right).
\end{equation}
Fix $m := [n^{\frac{b+1}{b+2p+1}}]$. 
We decompose the LHS above as 
\begin{align*} 
&{\phantom{{}={}}}\PP_{\data}\!\left[\mbf{1}_{E_n}\int_{\err} \Pi(df) \exp\!\left(-\cancel{\lambda^{-1}}n\ell_n(f)\right)\right] \\ 
&\le
\Pi(\Theta_m^c) + \int_{\Theta_m\cap\err} \Pi(df) \left(\PP_{\data} A(f,\data)^c + \PP_{\data}\big(\mbf{1}_{A(f,\data)} e^{-n\ell_n(f)}\big)\right),\numberthis\label{eq:contraction-reduced-goal-intermediate}
\end{align*}
where the event $A$ will be defined in Step 2. 
Since $\Pi(\Theta_m^c)\le \exp(-C_{gp'} n\delta_n^2)$, where the constant $C_{gp'}$ can be chosen to be smaller than $C_{num}$,\footnote{
We can replace $\Theta_m$ with $\Theta_{[m/C]}$ at the cost of inflating $\tau_m$ by a constant, which can be absorbed into $M$.
}
it suffices to deal with the integral term. 
For its former half involving $A(f,\data)^c$, we will specify $A$ so that it has the required probability. 
For the latter half, observe that by Lemma~\ref{lem:num-step-2}, 
for sufficiently large $M$, 
the function sets $\{\|f-f_0\|_2\ge M\epsilon_n\}$ and $\{\|E(f-f_0)\|_2\ge M\delta_n\}$ are equivalent (up to an inconsequential scaling) after intersecting with $\Theta_m$. Therefore, it suffices to prove one of \eqref{eq:L2-rate} or \eqref{eq:L2-rate-CE}, and the other will follow; and in the following, we can consider $f\in\err\cap\Theta_m$ as satisfying
\begin{equation}\label{eq:err-implication-in-thm}
\|f-f_0\|_2^2 \ge M\epsilon_n^2, ~~ \|E(f-f_0)\|_2^2 \ge C\delta_n^2 \frac{\|f-f_0\|_2^2}{\epsilon_n^2},
\end{equation}
which is valid for sufficiently large $M$; and we recall $C$ is determined by Assumption~\ref{ass:link}.  
In the remainder of the proof, we lower bound $\ell_n(f)$ using the above display. %

\parspace
\underline{(Step 2)} 
We use an argument inspired by \cite{dikkala_minimax_2020}. 
Introduce the notations 
$$
\Psi_{n}(f,g) := \frac{1}{n}\sum_{i=1}^n (f(x_i)-y_i)g(z_i), ~~ \Psi(f,g) =\EE_{\data} \Psi_n(f,g),
$$
so that $\ell_n(f) = \sup_{g\in\cI} 2\Psi_n(f,g) - \|g\|_n^2 - \effnu\|g\|_\cI^2.$ 
Recall the notations from Lemma~\ref{lem:num-step-2}: 
$g := E(f-f_0)$, $\tilde g := E(\DeltaFApprox)$, $g_j := E(\Proj{j}{\DeltaFApprox})\in\cI$, $j := [n^{\frac{1}{b+2p+1}}]$; 
and $\DeltaFApprox$ defined in \eqref{eq:delta-f-approx}.

Let $\bl := (f(\bx) - \by)g_j(\bz) - g_j(\bz)^2$, and $l_1,\ldots,l_n$ be its realizations using $x_i,y_i,z_i$. Then 
$$
\begin{aligned}
|\bl| \le (\|f-f_0\|+B+\|g_j\|)^2, ~~& 
\mrm{Var} \bl \le \EE_{\data} \bl^2 \le (\|f-f_0\|+B+\|g_j\|)^2 \|g_j\|_2^2,\\ 
\Psi_n(f,g_j)-\|g_j\|_n^2 &= \frac{1}{n} \sum_{i=1}^n l_i, ~~ 
\Psi(f,g_j)-\|g_j\|_2^2 = \EE_{\data} \bl. 
\end{aligned}
$$
(In the first inequality above, observe $\bl = (f(\bx)-f_0(\bx)+f_0(\bx)-\by-g_j(\bz))g_j(\bz)$, and recall $|y_i-f_0(x_i)|\le B$ by assumption.)  
By %
Lemma~\ref{lem:bernstein}\SkipNOTE{ with $\eta=1/16$}, we have, for any fixed $f\in\Theta_m$, %
\begin{equation}\label{eq:Afd-defn}
\PP\left(\Psi_n(f,g_j)-\|g_j\|_n^2 \ge \Psi(f,g_j)-\frac{3}{2}%
\|g_j\|_2^2 \right)\ge 1- \exp\left(-\frac{%
n\|g_j\|_2^2}{16(\|f-f_0\|+B+\|g_j\|)^2}\right).
\end{equation}
We define $A(f,\data)$ as the event above; its required probability bound from Step 1 will be verified shortly. On this event we have,
\begin{align*}
\ell_n(f) &= \sup_{g\in\cI} 2\Psi_n(f,g) - \|g\|_n^2 - \effnu\|g\|_\cI^2  \\ 
&\overset{(a)}{\ge} 2\Psi_n(f,g_j) - \|g_j\|_n^2 - \effnu\|g_j\|_\cI^2  
\overset{(b)}{\ge} 2\Psi(f,g_j) - \frac{3}{2}\|g_j\|_2^2 - \effnu\|g_j\|_\cI^2  \\ 
&= 2\EE((f(\bx)-\by)(g(\bz)-g(\bz)+g_j(\bz)) - \frac{3}{2}\|g_j-g+g\|_2^2 - \effnu\|g_j\|_\cI^2  \\ 
&\overset{(c)}{=} {2\|g\|_2^2 - 2\|g\|_2\|g-g_j\|_2} - 
{\frac{3}{2}(\|g\|_2^2 + \|g_j-g\|_2^2 + 2\|g\|\|g_j-g\|)} 
- \effnu\|g_j\|_\cI^2  \\ 
&\overset{(d)}{\ge} c_1\|g\|_2^2 - c_2\|g-g_j\|_2^2 - \effnu\|g_j\|_\cI^2 
\overset{(e)}{\ge} c_1\|g\|_2^2 - C_2 \delta_n^2 - \effnu\|g_j\|_\cI^2 \numberthis\label{eq:ell-n-rhs}. 
\end{align*}
In the above, (a) follows because $g_j\in\cI$, (b) by definition of $A(f,\data)$, (c) by the tower property, (d) follows 
from the inequality $2\|g\|_2\|g-g_j\|_2 \le \epsilon^2 \|g\|_2^2 + \epsilon^{-2}\|g-g_j\|_2^2$ for any $\epsilon>0$, $c_1,c_2$ are universal constants, and (e) follows from \eqref{eq:g-gj-diff}. 

It remains to bound $\|g_j\|_\cI$ and $\PP(A^c(f,\data))$. 
Let $U := {\|f-f_0\|_2}/{\epsilon_n}$, so that $U\ge \sqrt{M}$ on $\err$. 
\SkipNOTE{
    Was $\gtrsim \sqrt{M}$ before, I think it's a typo.
}
When $M$ is sufficiently large, we have $U>1$, and 
\begin{align}
    U\delta_n &\overset{(a)}{\lesssim} 
\|g\|_2  \le \|f-f_0\|_2 = U\epsilon_n \le U, \label{eq:g-L2} \\ 
\|\DeltaFApprox\|_2 &\le \|f-f_0-\DeltaFApprox\|_2 + \|f-f_0\|_2 \overset{(b)}{\le} 2U\epsilon_n, \label{eq:deltafapprox-L2}
\end{align}
where (a) is by \eqref{eq:err-implication-in-thm}, and (b) is by \eqref{eq:delta-f-approx}. 
\SkipNOTE{\eqref{eq:err-implication} holds for sufficiently large $M$, so it will hold for $U^2>M$}
Now, %
\begin{align*}
\|f-f_0\|^2 &\le 2\|\DeltaFApprox\|^2 + 2\|f-f_0-\DeltaFApprox\|^2 
\overset{(a)}{\lesssim} \|\DeltaFApprox\|^2 + 1 \\
&\overset{(b)}{\lesssim} (\|\DeltaFApprox\|_\cH^2)^{\frac{b}{b+1}} (\|\DeltaFApprox\|_2^2)^{\frac{1}{b+1}} + 1
\overset{(c)}{\lesssim} 
n^{\frac{1}{b+2p+1}\cdot \frac{b}{b+1}} (U^2n^{\frac{-b}{b+2p+1}})^{\frac{1}{b+1}} + 1 \le 
2U^{\frac{2}{b+1}}, \numberthis\label{eq:f-sup-norm-bound} \\ 
\|g_j\|_\cI^2 &= \|E(\Proj{j}{\DeltaFApprox})\|_\cI^2 
\overset{(d)}{\lesssim} \|\Proj{j}{\DeltaFApprox}\|_{\bar\cH}^2 
\\ 
&\overset{(b)}{\lesssim}
\|\Proj{j}{\DeltaFApprox}\|_\cH^{\frac{2b}{b+1}} \|\Proj{j}{\DeltaFApprox}\|_2^{\frac{2}{b+1}}
\overset{(c')}{\lesssim} 
n^{\frac{1}{b+2p+1}\cdot \frac{b}{b+1}} 
(U^2n^{\frac{-b}{b+2p+1}})^{\frac{1}{b+1}}
= U^{\frac{2}{b+1}}. \numberthis\label{eq:gj-I-norm-bound}
\end{align*}
In the above, (a) follows from \eqref{eq:delta-f-approx}, (b) from Lemma~\ref{lem:interpolation}, (c) from \eqref{eq:delta-f-approx} and \eqref{eq:deltafapprox-L2}, (c') additionally from the fact that $\mrm{Proj}_j$ is an orthogonal projection in both $L_2$ and $\cH$, and (d) by Assumption~\ref{ass:I-approx} (ii). 

It now follows that $
\|g_j\| \lesssim \|g_j\|_\cI \lesssim U^{\frac{2}{b+1}}. 
$
Plugging this inequality, \eqref{eq:g-L2} and \eqref{eq:f-sup-norm-bound} into \eqref{eq:Afd-defn}, we have 
$$
-\log \PP_{\data}(A^c(f,\data)) \ge \frac{n\|g_j\|_2^2}{16(\|f-f_0\|+B+\|g_j\|)^2} 
\gtrsim \frac{U^2}{(U^{\frac{1}{b+1}}+B)^2} n\delta_n^2.
$$
Therefore, when $U\gtrsim B$ is sufficiently large, we have 
$$
\sup_{f\in\Theta_m\cap\err}\PP_{\data}(A^c(f,\data)) \le e^{-2C_{num}n\delta_n^2}, 
$$
which fulfills the requirement in Step 1. 

Plugging \eqref{eq:g-L2}, \eqref{eq:gj-I-norm-bound} and $\effnu\asymp \delta_n^2$ into \eqref{eq:ell-n-rhs}, we have 
$$
\ell_n(f) \ge C_1 U^2 \delta_n^2 - C_2 \delta_n^2 - C_3 U^{\frac{2}{b+1}} \delta_n^2. 
$$
As $b>1$, there exists $C,C'>0$ such that when $M=U^2$ is sufficiently large, we have, for all $f\in\Theta_m\cap\err$, 
$$
\ell_n(f) \ge C U^2 \delta_n^2 = C' M\delta_n^2.
$$
This completes the proof for \eqref{eq:contraction-reduced-form-goal}, and by Step 1, the proof for \eqref{eq:L2-rate} and \eqref{eq:L2-rate-CE}. 

\SkipNOTE{
removed the last remark. Turns out the union bound over $g$ wasn't from Dikkala; it was my stupidity.
}

\subsubsection{General Subgaussian Case}

In this case we replace $B$ with $B_n:=\gamma_n\sqrt{\log n}$, where $\gamma_n$ is any slowly growing sequence $\to\infty$, so that $\PP_{\data}(\max_i |\by_i-f_0(\bx_i)|\le B_n)\to 1$. %
We add this event to the definition of $E_{den}$. Then the proof follows, with all occurrences of $B$ replaced by $B_n$. As mentioned above, %
it suffices to have $M_n\gtrsim B_n^2$. 

Now the proof for Theorem~\ref{thm:L2} is complete. \hfill \BlackBox \\

\subsection{Proof of Corollary~\ref{thm:contraction-alt-norm}}\label{app:proof-alt-norm}

The proof is also based on reduction to Theorem~\ref{thm:L2}, but requires a specific sieve. In particular, we choose 
$
\Theta_m=\Theta_{m,b''}
$ defined in Corollary~\ref{corr:valid-sieve}, where $\max\{b_{emb},b'\}< b''<b$. Recall it also satisfies \eqref{eq:valid-sieve}. We also need to derive some additional properties for $\Theta_m$. 

Below we denote the norm $\| \cdot \|_{\cH^{\frac{b^\prime}{b+1}}}$ as $\| \cdot \|_{b^\prime}$  for $b^\prime \in [0, b)$.

\begin{lemma}\label{lem:sobolev-reduction}
Let $m=[n^{\frac{b+1}{b+2p+1}}], \Theta_m$ be defined as above. 
Let $f\in\Theta_m$. Then 
(i) For the decomposition $f=\tilde f_h+\tilde f_e$ as defined in \eqref{eq:valid-sieve}, it additionally holds that 
$
\|\tilde f_e\|_{b'}^2 \lesssim m^{-\frac{b-b'}{b+1}}. 
$
(ii) Suppose $f$ additionally satisfies 
$\|f - f_0\|_{b'}^2 > M n^{-\frac{b-b'}{b+2p+1}},$ where $M>0$ is sufficiently large. Then 
$$
\|f-f_0\|_2^2 \gtrsim M n^{-\frac{b}{b+2p+1}}. 
$$
\end{lemma}
\begin{proof}
(i): 
Recall the construction in Corollary~\ref{corr:valid-sieve}: $\tilde f_e = \Proj{>j}{f_e}$ for $j=[m]^{\frac{1}{b+1}}$. By a similar argument as in Corollary~\ref{corr:valid-sieve} and Lemma~\ref{lem:sobolev}~\ref{item:sobolev-norm}, we find 
$
\|\tilde f_e\|_{b'}^2 \lesssim m^{-\frac{b-b'}{b+1}}. 
$

(ii):
Using a similar argument to Lemma~\ref{lem:approx-validation}, the condition $f_0\in\bar\cH$ implies that $\fApproxM$ defined in Assumption~\ref{ass:approx-relaxed} also satisfy 
\begin{equation}\label{eq:approx-sobolev}
\|f_0 - \fApproxM\|_{b'}^2 \le m^{-\frac{b-b'}{b+1}}.
\end{equation}
By Assumption~\ref{ass:approx-relaxed}, \eqref{eq:approx-sobolev} and the definition of $\Theta_m'$, we have $f-f_0 = \tilde f_h' + \tilde f_e'$, where 
$$
\|\tilde f_h'\|_\cH^2 \lesssim n^{\frac{1}{b+2p+1}}, ~~ 
\|\tilde f_e'\|_2^2 \lesssim n^{-\frac{b}{b+2p+1}}, ~~
\|\tilde f_e'\|_{b'}^2 \lesssim n^{-\frac{b-b'}{b+2p+1}}. 
$$
Let $m'=[n^{\frac{1}{b+2p+1}}], f_h' = \Proj{m'}{\tilde f_h'}, f_e' = f-f_0-f_h'$. Then by Lemma~\ref{lem:sobolev} and a similar argument for the $\|\cdot\|_{b'}$ norm, we can show that $f_e'$ satisfies the same norm inequalities as above. And since $f_h'$ is a truncation, $$
\|f_h'\|_2^2 \ge \lambda_{m'}^{-\frac{b'}{b+1}} \|f_h'\|_{b'}^2 
\asymp n^{-\frac{b'}{b+2p+1}} \|f_h'\|_{b'}^2 
\gtrsim n^{-\frac{b'}{b+2p+1}} (\|f-f_0\|_{b'}^2  - \|f_e'\|_{b'}^2)
\gtrsim M n^{-\frac{b}{b+2p+1}}.
$$
And $\|f-f_0\|_2^2 \ge (\|f_h'\|_2 - \|f_e'\|_2)^2 \gtrsim M n^{-\frac{b}{b+2p+1}}$ as claimed.
\end{proof}

\begin{proof}[Proof of Corollary~\ref{thm:contraction-alt-norm}]
Lemma~\ref{lem:sobolev-reduction} shows the intersection of $\Theta_m$ and the parameter region defined in Theorem~\ref{thm:contraction-alt-norm} is a subset of $\err$ in Theorem~\ref{thm:L2}. 
Thus, we can follow the proof in Appendix~\ref{app:proof-L2-main} and obtain Corollary~\ref{thm:contraction-alt-norm}. %
\end{proof}

\subsection{Proof for Theorem~\ref{thm:credible-ball}}\label{app:proof-cred-ball}

Consider the inequality 
\begin{align*}
\PP(\Pi(\{\|f-\hat f_n\|_{b'} \le \rho \epsilon_{n,b'}\}\givendata)) &\le \PP(E_{den}^c) + \PP\!\left(\mbf{1}_{E_{den}}
\frac{\Pi(\{\|f-\hat f_n\|_{b'} \le \rho \epsilon_{n,b'}\})}{\int \exp(-\lambda^{-1}n\ell_n(f)) \Pi(df)}
\right), \\ 
&\overset{(a)}{\le} o(1) + 
e^{C_1 n^{\frac{1}{b+2p+1}}} \Pi(\{\|f-\hat f_n\|_{b'} \le \rho \epsilon_{n,b'}\}) \\ 
&\overset{(b)}{\le} o(1) + 
e^{C_1 n^{\frac{1}{b+2p+1}}} \Pi(\{\|f\|_{b'} \le \rho \epsilon_{n,b'}\}), 
\end{align*}
where (a) holds by Proposition~\ref{lem:den}, and (b) by \citet[Lemma I.28]{ghosal2017fundamentals}. 

Recall that by Lemma~\ref{lem:preliminary-sieve} we have, for sufficiently small $\epsilon$, 
\begin{equation}\label{eq:smb-prob-L2}
\log \Pi(\{\|f\|_{b'} \le \epsilon\}) \le -C_{b'} \epsilon^{-\frac{2}{b}} 
\end{equation}
where $C_{b'}$ is also determined by $k_x$. Thus, for 
$
\rho \le (2C_1/C_{b'})^{-\frac{b}{2}}, 
$ we have 
$$
\log \Pi(\{\|f\|_{b'} \le \rho\epsilon_{n,b'}\}) \le -2 C_1 \epsilon^{-\frac{2}{b}}_n = -2C_1 n^{\frac{1}{b+2p+1}}, 
$$
leading to 
$$
e^{C_1 n^{\frac{1}{b+2p+1}}} \Pi(\{\|f\|_{b'} \le \rho \epsilon_{n,b'}\}) \le 
    e^{-C_1 n^{\frac{1}{b+2p+1}}} \rightarrow 0, 
$$
which completes the proof. 
\hfill\BlackBox\\[2mm]

Note that a similar result hold for the $\|E(\cdot)\|_2$ norm as well, if we replace Lemma~\ref{lem:preliminary-sieve} with \eqref{eq:small-ball-condexp-ub}. But $\|E(\cdot)\|_2$-norm balls cannot be constructed with finite samples. 

\subsection{Proof of Theorem~\ref{thm:GPR-L2}}\label{app:proof-GPR}

Recall that in the GPR setting, we have $p=0$, $m=[n^{\frac{b+1}{b+2p+1}}]=n, \Theta_m=\Theta_n$, and $\delta_n=\epsilon_n$. The posterior is 
$
\frac{d\Pi(\cdot\givendata)}{d\Pi}(f) \propto e^{-n\ell_n(f)}, 
$ where 
$$
\ell_n(f) = \frac{1}{n}\sum_{i=1}^n (f(x_i)-y_i)^2 \equiv 
    \frac{1}{n}\sum_{i=1}^n (f(x_i)-f_0(x_i))^2 - 2(y_i-f_0(x_i))(f(x_i)-f_0(x_i)).
$$
In the second equality above we dropped the term $\frac{1}{n}\sum_i e_i^2$, which is independent of $f$; and we recall $\be:=\by-f_0(\bx)$ now satisfies $\EE(\be \mid \bx)=0$, $|\be|\le B$. 

Let $\bl := (f(\bx)-f_0(\bx))^2 - 2(\by-f_0(\bx))(f(\bx)-f_0(\bx))$, and $l_1,\ldots,l_n$ be its $n$ i.i.d.~realizations using $x_i,y_i$. Then 
$
\frac{1}{n}\sum_{i=1}^n l_i = \ell_n(f), 
\EE \bl = \|f-f_0\|_2^2, 
\mrm{Var}(\bl) \le \EE \bl^2 \le \EE(f(\bx)-f_0(\bx))^2 (\|f-f_0\|+2B)^2 =: \sigma^2_l,
$ and is bounded by $(\|f-f_0\|+2B)^2 =: B_l$. By Bernstein's inequality as in Lemma~\ref{lem:bernstein}, we have, for any fixed $f$,
\begin{align}
\PP_{\data}\left(\ell_n(f)\ge 2\|f-f_0\|_2^2\right) &\le \exp\left(-\frac{n\|f-f_0\|_2^2}{5(\|f-f_0\|+2B)^2}\right), \label{eq:gpr-lh-ub} \\ 
\PP_{\data}\left(\ell_n(f)\le \frac{1}{3}\|f-f_0\|_2^2\right) &\le \exp\left(-\frac{n\|f-f_0\|_2^2}{10(\|f-f_0\|+2B)^2}\right). \label{eq:gpr-lh-lb}
\end{align}

Now we can prove the theorem. 
It suffices to prove the $L_2$ case, as the $b'>0$ case follows by reduction as in Corollary~\ref{thm:contraction-alt-norm}. We follow the strategy in Theorem~\ref{thm:L2}, in particular the decomposition in Step 1. All we need are 
\begin{enumerate}
    \item a new bound for the denominator, 
    \item a new event $A(f,\data)$ with probability $\PP_{\data}(A^c(f,\data))\le \exp\!\left(-C_{num} n^{\frac{1}{b+1}}\right)$,  
    and a lower bound for $\ell_n(f)$ on the event $A(f,\data)\cap E_{den}$, 
    for any $f\in \Theta_m\cap \err$. 
\end{enumerate}
We address them in turn: 
\begin{enumerate}
    \item 
As all assumptions in Section~\ref{sec:assumptions} hold, Lemma~\ref{lem:den} will still hold, leading to 
$$
\Pi(\Theta_{d0}) := 
\Pi(\{
    \|f-f_0\|_2 \le C_1 \epsilon_n, ~
    \|f-f_0\| \le C_2
\}) \ge \exp\!\left(-C_3 n \epsilon_n^2\right).
$$
By \eqref{eq:gpr-lh-ub} and the definition of $\Theta_{d0}$ above, we have 
\begin{equation}\label{eq:den-bound-gpr-fixed-f}
\inf_{f\in\Theta_{d0}}\PP_{\data}(\{
    \ell_n(f) \le 2 \|f-f_0\|_2^2 
\}) \ge 1- \exp\!\left(-\frac{nC_1^2\epsilon_n^2}{5(C_2+2B)^2}\right)\rightarrow 1.
\end{equation}
Therefore, by the argument of Step 1 in Proposition~\ref{lem:den}, we have 
\begin{equation}\label{eq:den-bound-gpr}
\int e^{-n\ell_n(f)}\Pi(df) \ge e^{-C_3'n\epsilon_n^2}. 
\end{equation}

\item
Let $A(f,\data) := \{\ell_n(f)\ge \frac{1}{3}\|f-f_0\|_2^2\}$ be the complement of the event in \eqref{eq:gpr-lh-lb}. 
As all assumptions in Theorem~\ref{thm:L2} hold, \eqref{eq:f-sup-norm-bound} holds, which says that for all $f\in\Theta_m\cap\err$, %
$$
\|f-f_0\|^2 \lesssim U^{\frac{2}{b+1}} = \left(\frac{\|f-f_0\|_2^2}{\epsilon_n^2}\right)^{\frac{1}{b+1}}. %
$$
Plugging back to \eqref{eq:gpr-lh-lb}, we have, 
for some fixed constant $C>0$, %
$$
\PP_{\data} A^c(f,\data) \le 
\exp\left(
    -C n (\|f-f_0\|_2^2)^{\frac{b}{b+1}} (\epsilon_n^2)^{\frac{1}{b+1}}
\right) 
\le 
\exp\left(
    -C M^{\frac{b}{b+1}} 
    n \epsilon_n^2 
\right),
$$
where the last inequality holds because $\|f-f_0\|_2^2\ge M\epsilon_n^2$. On the event $E_{den}\cap A(f,\data)$ we have, for all $f\in\err\cap \Theta_m$, 
$$
\ell_n(f)\ge \frac{1}{4}\|f-f_0\|_2^2 - 3Bn^{-1} \gtrsim M\epsilon_n^2.
$$
Plugging 
the two displays above to \eqref{eq:contraction-reduced-goal-intermediate}, we can establish \eqref{eq:contraction-reduced-form-goal} with $C_{num}\gtrsim M^{\frac{b}{b+1}}$. 
\end{enumerate}
Now the proof for Theorem~\ref{thm:GPR-L2} is complete. \hfill\BlackBox\\[2mm]

\subsection{Additional Discussion on the Assumptions}

We prove the following claim referenced in the main text. 

\begin{lemma}\label{lem:approx-validation}
In Assumption~\ref{ass:approx-relaxed}, (a) implies (b). 
\end{lemma}
\begin{proof}
Let $j_n=[n^\frac{1}{b+1}]$ and $\fApprox := \Proj{j_n}{f_0}$. By Lemma~\ref{lem:sobolev} with $\tilde b=b$, $\fApprox$ satisfies the $L_2$ approximation condition, and $
\|\fApprox\|_{\bar\cH}\le \|f_0\|_{\bar\cH}
$ which implies $\|f_0-\fApprox\|\lesssim 1$ (recall $\bar\cH$ has a bounded kernel, Lemma~\ref{lem:power-space}). 
Finally, 
$
\|\fApprox\|_\cH^2 \lesssim \sum_{j=1}^{j_n} j^{b+1} \<\fApprox, \bar\varphi_j\>_2  \le 
j_n\cdot  \sum_{j=1}^{j_n} j^b\<\fApprox, \bar\varphi_j\>_2 = j_n \|\fApprox\|_{\bar\cH}^2 \le 
j_n \|f_0\|_{\bar\cH}^2
\lesssim n^{\frac{1}{b+1}}. 
$
\end{proof}

\subsubsection{Relaxation of Assumption~\ref{ass:I-approx}}\label{app:proof-nystrom}

Consider the following relaxation of Assumption~\ref{ass:I-approx}~(ii), in which we keep the choice of norm for $g$ unchanged, but relax the inner optimization problem \eqref{eq:quasi-loglh-defn} to allow an approximation error of $O(\delta_n)$: 
\begingroup
\def\theassumption{\ref{ass:I-approx} (ii')}
\begin{assumption}\label{ass:I-approx-2-relaxed}
For any $n\in\mb{N}$, let $\tilde\cI_n\subset \cI$ be a possibly random subset, and $\{M_n\}$ be a sequence of positive numbers. Then 
with $\PP_{\data}$ probability $\to 1$, it holds that 
$
\text{for any }g\in\cI_1,\text{ there exists }g_n\in\tilde\cI_n\text{ s.t.~}$
\begin{equation}\label{eq:relaxed-I-approx-cond}
\|g_n-g\|_2 \le M_n \delta_n, ~~ \|g_n\|_\cI \le M_n.
\end{equation}
The definition of the (scaled log) quasi-likelihood \eqref{eq:quasi-loglh-defn} is changed into 
\begin{equation*}
d_n^2(\hat E_n f - \hat b) := \max_{g\in\cI_n}\frac{1}{n}\sum_{j=1}^n\!\left(
    (f(x_j)-y_j)g(z_j) - \frac{g(z_j)^2}{2}
    \right) - \frac{\effnu}{2}\|g\|_\cI^2.
\end{equation*}
\end{assumption}
\addtocounter{assumption}{-1}
\endgroup

Now Theorem~\ref{thm:L2} will hold, if we multiply the rates  
$\epsilon_n,\delta_n$ by $M_n$, as its proof 
follows with minimal modifications: we add the event defined by \eqref{eq:relaxed-I-approx-cond} to the test event $E_{den}$ (and thus $E_n$) in the proof. 
Now the denominator bound (Proposition~\ref{lem:den}) will hold, because the $\sup$ term to be upper bounded is now defined with a smaller scope. 
For the denominator bound, we replace $g_j$ with its approximation defined in \eqref{eq:relaxed-I-approx-cond}, and account for their difference in \eqref{eq:ell-n-rhs}. 
The subtracted terms will be inflated by a factor of $M_n^2$, which can be accounted by multiplying the large constant $M$ in the theorem by $M_n^2$. Then the proof follows, and the theorem will hold with $M$ modified accordingly, i.e., with the rate inflated by $M_n$. 

Now we verify the relaxed assumption allows for Nystr\"om approximation for $\hat E_n$, 
where 
$\tilde\cI_n := \mrm{span}\{k(z_{i_j},\cdot): j\in [m]\}$, and  
$z_{i_1},\ldots,z_{i_m}$ are sampled from the training data. 
We verify \eqref{eq:relaxed-I-approx-cond} using
$
g_n := \min_{g_n\in\tilde\cI_n} \|g_n-g\|_n^2 + \tilde\nu_n\|g_n\|_\cI,
$ %
by adapting the analysis of \citet{rudi_less_2016}.\footnote{
This is noiseless KRR, so their analysis applies; recall that our definition of $g_n$ does not have to match $\hat E_n$ in any sense. 
We also note that we use the extended version of \cite{rudi_less_2016} on arXiv.
} Theorem 1 therein shows that 
for their choice of $m$, we have, for any $\|g\|_\cI\le 1$,  
\begin{equation}\label{eq:rr-l2-event}
\PP_{\data}(\|g_n-g\|_2 \lesssim \log \eta^{-1} \delta_n) \ge 1- \eta.
\end{equation}
From their proof (in particular, the last display on p.~24) we can see that, in our noiseless setting, the event defined above is independent of $g$, so the first half of \eqref{eq:relaxed-I-approx-cond} holds. 
To show $\|g_n\|_\cI\lesssim 1$, we modify their argument in the proof of Theorem 2 on p.~24. 

\emph{We now switch to the notations in \cite{rudi_less_2016}}, under which we need to show $\|\hat f_{\lambda,m}\|_\cI \lesssim 1$, or equivalently $\|\hat f_{\lambda,m} - f_\cH\|_\cI \lesssim 1$ 
for all $\|f_\cH\|_\cH\le 1$.\footnote{Our notations $g_n,g,\effnu,\cI$ corresponds to their $\hat f_{\lambda,m},f_\cH,\lambda,\cH$, respectively.}  It has a similar decomposition as the $L_2$ error they analyzed:
$
\|\hat f_{\lambda,m} - f_\cH\|_\cH \le \|(I-g_{\lambda,m}(C_n)C_n)f\|_\cH.
$
(Note we have dropped their term $\msf{A}$, because we are in the noiseless regime.) Following the last display on their p.~24, with the leading $C^{1/2}$ removed on the first two lines, we find 
$
\|\hat f_{\lambda,m} - f_\cH\|_\cH \le 
R\lambda^{-1/2}(1+\theta)\cdot \msf{B.1} + R\lambda^{-1/2}\cdot \msf{B.2}. 
$
Following the remainder of their proof, and subsequently the proof for their Proposition~2, we find that on the event defined by \eqref{eq:rr-l2-event}, the RHS above is $O(\lambda^{-1/2} \delta_n \log \eta^{-1})$, where $\delta_n$ denotes the $L_2$ error rate, and $\lambda\asymp \delta_n^2$ is their ridge regularization hyperparameter. Therefore, \emph{switching back to our notations}, we have, on the event \eqref{eq:rr-l2-event},
$$
\sup_{\|g\|_\cI\le 1}\|g_n-g\|_\cI\lesssim \log \eta^{-1}.
$$
Thus, we can set $M_n$ as any slowly increasing sequence, so that $\eta=e^{-M_n}\rightarrow 0$, and the assumption will hold. It may be possible to replace the 
increasing $\{M_n\}$ with a fixed $M$, using a more sophisticated choice of $g_n$ and/or a slightly larger $m$, but such a discussion is beyond the scope for this work.

\section{Deferred Derivations in Section~\ref{sec:nonasymptotic}}

\subsection{Proof of Proposition~\ref{prop:gp-covariance-unconfounded-dgp}}\label{sec:gp-cov-unconfounded-proof}
Define $$
\hat{\cV} := \empCxz \invEmpCz \empCzx, ~
A := (\hat{\cV} + \efflam I)^{-1} \empCxz \invEmpCz \frac{S_z^*}{n}, ~
P := A S_x = (\hat{\cV} + \efflam I)^{-1}\hat{\cV}, ~
E := Y - f(X),
$$
so that $A Y$ maps the observed data $Y\in\mb{R}^n$ to the IV mean prediction, i.e., $\hat f_n = A(f(X)+E)$. 
By \eqref{eq:equiv-cov-form} below we can see that 
$$\cC = I - P.$$ 

For \eqref{eq:cov-as-worst-case-error}, 
we have 
\begin{align*}
    \sup_{\|f\|_\cH=1} \EE\<\hat f_n-f_0, k_*\>_\cH^2
&=\sup_{\|f\|_{\cH}=1} \EE(\<A(f(X)+E), k_*\>_{\cH} - \<f,k_*\>_\cH)^2  \\ 
&= 
    \sup_{\|f\|_{\cH}=1} \<A f(X), k_*\>_{\cH} - f(x_*))^2 + 
    \lambda k_*^\top A  A^\top k_*    \\ 
    &= 
    \sup_{\|f\|_{\cH}=1} (\<f, \tilde h\>_{\cH})^2 +
    \lambda k_*^\top A  A^\top k_* = 
    \|\tilde h\|_{\cH}^2 + 
    \lambda k_*^\top A  A^\top k_*,
\end{align*}
where 
$
\tilde h := k_*^\top A k(X, \cdot) - k(x_*, \cdot).
$
Plugging back, we have 
\begin{align*}
\sup_{\|f\|_\cH=1} \EE\<\hat f_n-f_0, k_*\>_\cH^2
 &= k_{**} - 2 k_*^\top A k_{X*} + k_*^\top A k(X, X) A^\top k_* + \lambda k_*^\top A A^\top k_* 
=: k_*^\top \tilde\cC k_*,
\end{align*}
where 
\begin{align*}
\tilde\cC &= I - 2 A S_x + A (S_x S_x^* + \lambda I) A^\top \\ 
&= I - 2 (\hat{\cV} + \efflam I)^{-1}\hat{\cV}  + 
(\hat{\cV} + \efflam I)^{-1} \empCxz \invEmpCz \frac{S_z^*}{n}\left({S_x S_x^*} + {n\efflam I}\right) \frac{S_z}{n} \invEmpCz \empCzx (\hat{\cV} + \efflam I)^{-1}  \\ 
&\overset{(a)}{=} I - 2(\hat{\cV} + \efflam I)^{-1}\hat{\cV} + 
(\hat{\cV} + \efflam I)^{-1} ({\hat{\cV}\hat{\cV}} + 
{\efflam \hat{\cV} - \empCxz\invEmpCz \cdot\efflam\effnu I\cdot \invEmpCz \empCzx)}
(\hat{\cV} + \efflam I)^{-1}  \\  
&= I - (\hat{\cV} + \efflam I)^{-1}\hat{\cV} - 
(\hat{\cV} + \efflam I)^{-1} \empCxz\invEmpCz \cdot\efflam\effnu I\cdot \invEmpCz \empCzx 
(\hat{\cV} + \efflam I)^{-1}  \\ 
&\overset{(b)}{=:} \cC - \Delta\cC.
\end{align*}
In the above, (a) holds because we can replace the $\mblue{\hat{\cV}}$ term in RHS with 
$\empCxz \regEmpCz^{-1}\regEmpCz\regEmpCz^{-1}\empCzx$, and (b) since 
$\cC = I-(\hat\cV+\efflam I)^{-1}\hat\cV$. 
The last equality shows that \eqref{eq:cov-as-worst-case-error} holds, and 
\begin{equation}
    \Delta\cC = (\hat{\cV} + \efflam I)^{-1} \empCxz\invEmpCz \cdot\efflam\effnu I\cdot \invEmpCz \empCzx 
(\hat{\cV} + \efflam I)^{-1} 
\end{equation}
is non-negative. 

For \eqref{eq:cov-as-avg-error}, we have 
\begin{align*}
\EE_f \EE_{E}(\<A(f(X)+E), k_*\>_{\cH} - \<f, k_*\>_{\cH})^2 &= 
k_*^\top A (k(X,X)+\lambda I) A^\top k_* + 
k_{**} - 2 k_*^\top A K_{X*},
\end{align*}
so the proof follows \eqref{eq:cov-as-worst-case-error}. \hfill\BlackBox

\subsection{Additional Discussion} \label{app:DeltaC}

Consider the \emph{stochastic} estimator 
$$
\hat f_{sn} := \bar A Y := 
     \min_{f\in\cH}\max_{g\in\cI} \frac{1}{n}\sum_{i=1}^n\!\left((f(x_i)-y_i)g(z_i) - \frac{g(z_i)^2}{2}\right) - \frac{\effnu}{2}\|g-g_{rp}\|_\cI^2 + \frac{\efflam}{2}\|f\|_\cH^2,
$$
where $g_{rp}\sim \mc{GP}(0, \lambda\nu^{-1}k_z)$. 
With direct calculations similar to the above, we can show that, in the unconfounded setting as Proposition~\ref{prop:gp-covariance-unconfounded-dgp}, 
\begin{align}
\<k_*, \cC k_*\>_\cH &= \sup_{\|f_0\|_\cH=1} \EE_{Y\mid X,Z} \<k_*, \hat f_{sn} - f_0\>_{\cH}^2 
=  \EE_{f_0\sim\mc{GP}(0,k_x)} \EE_{Y\mid X,Z} \<k_*, \hat f_{sn} - f_0\>_{\cH}^2, \label{eq:marginal-covariance-alt-repr}
\end{align}
i.e., the marginal variance actually represents the estimation error using the stochastic estimator, due to the additional term $\Delta\cC$. 

We now discuss the behavior of $\hat f_{sn}$. 
Clearly we have $\EE_{g_{rp}} \hat f_{sn} = \hat f_n$ (in the sense that evaluations of any bounded linear functional equal). 
Conditioned on $\data$, $\hat f_{sn}$ has a non-zero variance, but the variance is smaller than the posterior draw \eqref{eq:rf-obj-fs}. 
We can verify that it replaced the conditional expectation estimate $\hat E_n (f-f_0)$, which is the inner loop optima in \eqref{eq:dualiv-obj}, with a draw from the GPR posterior constructed from the problem of conditional expectation estimation. 
It thus accounts for the uncertainty in the first-stage estimation, and one may hope that it has a more stable behavior in the presence of confounding, in which case the estimation error may be higher. %

Readers may find it counterintuitive that the marginal variance of the quasi-posterior represents estimation error of a different estimator $\hat f_{sn}$, as opposed to the posterior mean $\hat f_n$, which is different from the GPR setting. %
Fundamentally, the reason is that the quasi-likelihood involves the regularization term $\effnu\|g\|_\cI^2$. Also, recall that contrary to the GPR setting, the unconfounded data generating process considered in Proposition~\ref{prop:gp-covariance-unconfounded-dgp} is not fully representative of the typical scenario; hence, the difference may be reasonable.

\section{Analysis of the Approximate Inference Algorithm}\label{app:approx-inf-analysis}

\subsection{Proof of the Double Randomized Prior Trick}\label{app:dual-algo-deriv}

\subsubsection{A Function-Space Equivalent to Proposition~\ref{prop:randomized-prior}}\label{app:rf-props-equiv}

\newcommand{\approxH}{{\tilde{\cH}}}
\newcommand{\approxI}{{\tilde{\cI}}}

We first claim that Proposition~\ref{prop:randomized-prior} is equivalent to the following function-space version, the proof of which is deferred to Section~\ref{app:rf-func-space-proof}: 
\begin{proposition}\label{prop:rf-func-space}
Let ${\approxH}, {\approxI}$ be finite-dimensional RKHSes with kernels ${k}_x, {k}_z$, respectively,
 $$
g_0\sim\mc{GP}(0,\lambda\nu^{-1} \tilde{k}_z),~ f_0\sim\mc{GP}(0, \tilde{k}_x),~ \tilde{y_i}\sim\cN(y_i,\lambda).
$$Then the optima $f^*$ of 
\begin{align}
\min_{f\in\approxH}\max_{g\in\approxI} \cL(f,g) &:= \sum_{i=1}^n \left((f(x_i)-\tilde{y}_i)g(z_i) - \frac{g(z_i)^2}{2}\right) - \frac{\nu}{2}\|g-g_0\|_{{\approxI}}^2 + \frac{\lambda}{2}\|f-f_0\|_{{\approxH}}^2 \label{eq:rf-obj-fs}
\end{align}
follows the posterior distribution \eqref{eq:qbkdiv}, with the kernels replaced by $\tilde{k}_x,\tilde{k}_z$. 
\end{proposition}

\begin{proof}[Proof of the equivalence]
Observe that 
\eqref{eq:rf-obj-fs} is exactly the same as \eqref{eq:rf-obj} when the random feature parameterization $\phi\mapsto g(z;\phi)$ is injective,\footnote{Most random feature models, such as the random Fourier feature model, satisfies this property almost surely.} in which case we have $\|\phi\|_2=\|g(\cdot;\phi)\|_\approxI$. %
Otherwise, observe that on the subspace 
$$
\Phi_s := \mrm{span}\{\phi_{z,m}(z^\prime):z^\prime\in\cZ\},
$$
$\|\phi\|_2=\|g(\cdot;\phi)\|_\approxI$ always holds: this 
follows by definition of $\tilde{k}_z$
when $\phi$ is a finite linear combination of the $\phi$'s, 
and the general case follows by continuity (note that %
$\approxI$ is already defined by $\tilde{k}_z$). Clearly any $g-g_0\in\tilde{\cI}$ can be parameterized with some $\phi$ in this subspace, so the optima of \eqref{eq:rf-obj-fs} is a valid candidate solution for \eqref{eq:rf-obj}. On the other hand, 
for any $\phi-\phi_0$ outside the aforementioned subspace, we have $\|\phi-\phi_0\|_2>\|g(\cdot;\phi)-g(\cdot;\phi_0)\|_\approxI$. Therefore, the optimal $\phi$ of \eqref{eq:rf-obj} must satisfy $\|\phi-\phi_0\|_2=\|g(\cdot;\phi)-g(\cdot;\phi_0)\|_\approxI$, and thus solves \eqref{eq:rf-obj-fs}. As a similar result also holds for $f$, we conclude that the two objectives are equivalent. 
\end{proof}
\begin{remark}
The non-injective setting above justifies the formal analysis of \eqref{eq:ntkrf-obj} in the main text. We also remark that any parameter $\theta,\phi$ visited by the SGDA algorithm on \eqref{eq:rf-obj} or \eqref{eq:ntkrf-obj} (starting from $\theta_0,\phi_0$) satisfies 
$$
\theta-\theta_0\in\Theta_s, ~~
\phi-\phi_0\in \Phi_s.
$$
Thus $\|\phi-\phi_0\|_2=\|g(\cdot;\phi)-g(\cdot;\phi_0\|_\approxI$ (and similarly for $\theta$), and from the perspective of the SGDA algorithm, the objectives \eqref{eq:rf-obj-fs} and \eqref{eq:rf-obj} are \emph{always} the same. 
This can be proved by induction. Take $\phi$ for example; clearly $\phi=\phi_0$ satisfies the above. For $\phi_\ell$ obtained at the $\ell$-th step of SGDA, we have 
$$
\phi_\ell-\phi_0 = (1-\nu)(\phi_{\ell-1}-\phi_0) + V_\ell^\top \phi_{z,m}(Z), 
$$
where $V_\ell\in\RR^n$ is independent of $\phi_\ell$. Thus $\phi_\ell-\phi_0\in\Phi_s$ by definition of $\Phi_s$ and the inductive hypothesis. 
\end{remark}

\subsubsection{Matrix Identities}

We list two identities here that will be used in the derivations.

\begin{lemma}
  Let $U, C, V, S$ be operators between appropriate Banach spaces, $\lambda \in \RR \setminus \{ 0 \}$, then
\begin{align*}
  (\lambda I+UCV)^{-1} 
  &= \lambda^{-1}(I-U(\lambda C^{-1}+VU)^{-1}V),  \numberthis{\label{eq:woodbury-scaleI}}\\ 
  S(S^* S+\lambda I)^{-1} 
  &= (SS^*+\lambda I)^{-1}S \numberthis \label{eq:woodbury-corr}.
\end{align*}
\end{lemma}
\begin{proof}
Recall the Woodbury identity:
$$
(A+UCV)^{-1} = A^{-1} - A^{-1}U(C^{-1}+V A^{-1} U)^{-1}VA^{-1}.
$$

Then, we have
\begin{align*}
  (\lambda I+UCV)^{-1} &= 
    \lambda^{-1} I - \lambda^{-2}U(C^{-1}+\lambda^{-1}VU)^{-1}V
  = \lambda^{-1}(I-U(\lambda C^{-1}+VU)^{-1}V).   \\
  S(S^* S+\lambda I)^{-1} &= S(\lambda^{-1}I - \lambda^{-2}S^*(\lambda^{-1}SS^*+I)^{-1}S) 
  = \lambda^{-1}(S - SS^*(SS^*+\lambda I)^{-1}S) \\ 
  &= (SS^*+\lambda I)^{-1}S.
\end{align*}
\end{proof}

\subsubsection{Proof of Proposition~\ref{prop:rf-func-space}} \label{app:rf-func-space-proof}
  Define $Y = (y_1, \dots, y_n)$, $\tilde Y = (\tilde y_1, \dots, \tilde y_n)$.
We rewrite the objective as %
\begin{align*}
\cL(f,g) 
&= \left(\<n\empCzx f - \empSz^* \tilde{Y}, g\>_\approxI - \frac{1}{2}\<n\empCz g,g\>_\approxI - \frac{\nu}{2}\|g-g_0\|_\approxI^2\right) + \frac{\lambda}{2}\|f-f_0\|_\approxH^2 \\ 
&= n\!\left(
\<\empCzx f - n^{-1}\empSz^* \tilde{Y}, g\>_\approxI - \frac{1}{2}\<\regEmpCz\,g, g\>_\approxI + \effnu\<g, g_0\>_\approxI - {\frac{\effnu}{2}\|g_0\|^2_\approxI}
\right) + \frac{\lambda}{2} \|f-f_0\|_\approxH^2,
\end{align*}
where $S_z,\empCzx,\empCz$ are now defined w.r.t.~the approximate kernels. 
The optimal $g^*$ for fixed $f$ is 
\begin{equation}\label{eq:rf-obj-g-opt}
g^*(f) = \invEmpCz (\empCzx f-n^{-1}\empSz^* \tilde{Y} + \effnu g_0).
\end{equation}
Plugging $g^*$ back to the objective, we have 
\begin{align*}
&\cL(f,g^*(f)) = \frac{n}{2}\<g^*, \regEmpCz g^*\>_\approxI + \frac{\lambda}{2}\|f-f_0\|_\approxH^2 - \frac{n\effnu}{2} \|g_0\|_\approxI^2, \\ 
&\partial_f \cL = n \empCxz {\invEmpCz\regEmpCz} g^* + {\lambda}(f-f_0)
=  n\empCxz\invEmpCz(\empCzx f-n^{-1}\empSz^*\tilde{Y}+\effnu g_0) + \lambda(f-f_0). 
\end{align*}
Setting $\partial_f \cL$ to zero, we obtain
\begin{equation}\label{eq:rf-obj-optima}
f^* = 
  (n\empCxz\invEmpCz\empCzx+\lambda I)^{-1}(n\empCxz\invEmpCz(n^{-1}\empSz^*\tilde{Y}-\effnu g_0)+\lambda f_0) .
\end{equation}
Since 
\begin{align*}
(n\empCxz\invEmpCz\empCzx+\lambda I)^{-1} &= 
    (n^{-1} S_x^* S_z \regEmpCz^{-1} S_z^* S_x + \lambda I)^{-1}
= (S_x^* L S_x + \lambda I)^{-1}  \\
\overset{\eqref{eq:woodbury-scaleI}}&{=} \lambda^{-1}(
  \underbrace{I - S_x^*(\lambda L^{-1} + S_x S_x^*)^{-1}S_x}_{\text{defined as }\cC}
  ) %
,\numberthis \label{eq:equiv-cov-form}
\end{align*}
we can rewrite $f^*$ as 
$$
f^* = {\lambda^{-1}\cC}({\empCxz\invEmpCz}({\empSz^*\tilde{Y}}-\nu g_0) + \lambda f_0).
$$
Clearly, $f^*$ is a Gaussian process. Suppose $f^*(x_*) \sim \cN(S_* \mu^\prime, S_* \cC^\prime S_*^*)$, then
\begin{align*}
\mu^\prime &= \lambda^{-1} \cC \:
n\empCxz\invEmpCz(n^{-1}\empSz^*Y)
= 
\lambda^{-1} (I - S_x^*(\lambda L^{-1} + S_x S_x^*)^{-1}S_x) S_x^* L Y \\ 
&= 
\lambda^{-1} S_x^*(I - (\lambda L^{-1} + S_x S_x^*)^{-1}S_x S_x^*) L Y 
=
 S_x^* (\lambda L^{-1} + S_x S_x^*)^{-1} Y.
\end{align*}
The RHS above 
matches the posterior mean \eqref{eq:post-mean-calc} (with $k_x,k_z$ replaced by their random feature approximations) since $S_xS^*_x = K_{xx}$ and
\[ 
 S_*\mu^\prime 
 = S_* S_x^* (\lambda L^{-1} + S_x S_x^*)^{-1} Y
 = K_{*x} (\lambda L^{-1} + K_{xx} )^{-1} Y
 = K_{*x} (\lambda  + LK_{xx} )^{-1}L Y.
 \]
As $\tilde Y-Y, g_0$ and $f_0$ are independent, the covariance operator of $f^*$ is 
\begin{align*}
\cC^\prime &= {\lambda^{-1}\cC}(
    \empCxz\invEmpCz({n\lambda\empCz} + \lambda \nu I)
    \invEmpCz\empCzx + \lambda^2 I
  )\lambda^{-1}\cC  \\ 
&= \lambda^{-1}\cC( 
  \lambda n\empCxz\invEmpCz\empCzx + \lambda^2 I
 )\lambda^{-1}\cC  
\overset{\eqref{eq:equiv-cov-form}}{=} 
 \cC. %
\end{align*}
In view of \eqref{eq:equiv-cov-form}, we know
\[ \begin{aligned}
S_*\cC^\prime S^*_* 
&=  S_*S_*^* - S_*S_x^*(\lambda L^{-1} + S_x S_x^*)^{-1}S_xS_*^* 
=  K_{**} - K_{*x}(\lambda L^{-1} + K_{xx})^{-1}K_{x*},
\end{aligned} \]
which matches the posterior covariance matrix~\eqref{eq:post-cov-calc} with replaced kernels. 

\begin{remark} For the discussion in Section~\ref{sec:nonasymptotic}, 
observe that \eqref{eq:rf-obj-optima} implies the posterior mean estimator $\hat f_n$ satisfies 
\begin{equation}\label{eq:posterior-mean-functional}
\hat f_n = 
  (n\empCxz\invEmpCz\empCzx+\lambda I)^{-1} n\empCxz\invEmpCz 
  \Biggl(\empCzx f_0 + \frac{\empSz^*(Y-f_0(X))}{n}\Biggr),
\end{equation}
The derivations do not rely on the RKHSes being finite-dimensional.
\end{remark}

\subsection{Assumptions in Proposition~\ref{prop:approx-inf}}\label{app:rf-assumptions}

The subsequent analysis will rely 
on the following assumptions on the random feature expansion. We only state them for $x$ for conciseness; the requirements for $z$ are similar.

The following assumption holds for, e.g., random Fourier features \cite{rahimi2007random}. 
\begin{assumption}\label{ass:rf-uniform-approx}
  As $m\to\infty$, it holds that 
$
\sup_{x,x'\in\cX} \left | k_x(x,x') - \tilde{k}_{x,m}(x,x') \right |  \overset{p}{\rightarrow} 0.
$
\end{assumption}

\begin{assumption}\label{ass:rf-bounded}
There exists a constant $\tilde{\kappa}>0$ such that 
$\max_{m\in\mb{N}} \sup_{x\in\cX} \tilde{k}_{x,m}(x,x)\le \tilde{\kappa}$.
\end{assumption}

\subsection{Analysis of Random Feature Approximation}\label{app:rf-analysis}

We recall the following facts: for $A,B\in\mb{R}^{n\times n}$, 
$$\|A\| \le \|A\|_F \le \sqrt{n}\|A\|, ~~ A^{-1}-B^{-1} = A^{-1}(B-A)B^{-1}.$$ 

\begin{lemma}\label{lem:rf-approx-appendix}
For all $m\in\mb{N}$, let $k_{x,m}$ be a random feature approximation to $k_x$ such that Assumption~\ref{ass:rf-uniform-approx} holds, and 
let $\tilde{k}_{z,m}$ be an approximation to $k_z$ satisfying a similar requirement as above. Then the random feature-approximated posterior 
$
\Pi_m(f(x_*) \givendata) = \cN(\tilde{\mu}, \tilde{S})
$
satisfies 
$$
\lim_{m\rightarrow\infty} \sup_{x^*\in\cX^l} \|\mu - \tilde{\mu}\|_2 = 0, ~~
\lim_{m\rightarrow\infty} \sup_{x^*\in\cX^l} \|\tilde{S} - S\|_F= 0,
$$
for any fixed training data $(X,Y,Z)$, $l\in\mb{N}$, and $\lambda,\nu>0$. In the above, $\tilde \mu$ and $\tilde S$ are defined as in \eqref{eq:post-mean-calc}-\eqref{eq:post-cov-calc}, but with 
the Gram matrices redefined %
using $\tilde{k}_{x,m}$ and $\tilde{k}_{z,m}$.
\end{lemma}
\begin{proof}%
Define $$
\epsilon_m := \max\Big(
  \sup_{x,x'\in\cX} \left | k(x,x') - \tilde{k}_{x,m}(x,x') \right | , 
  \sup_{z,z'\in\cZ} \left | k(z,z') - \tilde{k}_{z,m}(z,z') \right | \Big).
  $$
By assumption $\epsilon_m\overset{p}{\rightarrow} 0$. 
For $\tilde{S}$ we consider the decomposition
\begin{align*}
  \|\tilde{S} - S\| %
  &\le \|\tilde{K}_{**} - K_{**}\|   \\
&\phantom{{}={}}+\|\tilde{K}_{*x} - K_{*x}\|  \| 
  \tilde{L}\|\|(\lambda I + \tilde{K}_{xx}\tilde{L})^{-1}\tilde{K}_{x*} \| \\
&\phantom{{}={}} +\|K_{*x}\| \| \tilde{L} - L\| \|(\lambda I+\tilde{K}_{xx}\tilde{L})^{-1}\tilde{K}_{x*}\|  \\ 
&\phantom{{}={}} +
\|K_{*x} L\|\|(\lambda I+\tilde{K}_{xx}\tilde{L})^{-1} 
- (\lambda I+ K_{xx}L)^{-1}\| \|\tilde{K}_{x*}\| \\
&\phantom{{}={}} +
\|K_{*x}L(\lambda I+K_{xx}L)^{-1}\|\|\tilde{K}_{x*}-K_{x*}\| \\ 
&=: (\text{I}) + (\text{II}) + (\text{III}) + (\text{IV}) + (\text{V}).
\end{align*}
In the following, we use $O(\cdot)$ and $O_p(\cdot)$ to represent the asymptotic behaviour when $m \to \infty$.
Since $n$ and $l$ are fixed, the operator norms of the matrices $K_{*x},L,K_{xx}$ are $O(1)$. 
Observe that $\|K_{zz}-\tilde{K}_{zz}\|\le \sqrt{n}\epsilon_m$. 
By the triangle inequality, the inequality $\|\cdot\|\le \|\cdot\|_F$ and the boundedness of $\tilde{k}_{x,m}$ and $\tilde{k}_{z,m}$, we have $\|\tilde{K}_{*x}\| = O(1)$. Both $O(\cdot)$ terms above are independent of $x^*$. Finally, recall that $\|L\| = \|K_{zz}(K_{zz}+\nu I)^{-1}\|\le 1$ and similarly $\|\tilde{L}\|\le 1$. Using these facts, we have 
\begin{align*}
(\text{I}) &\le \|\tilde{K}_{**}-K_{**}\|_F \le l \epsilon_m  \rightarrow 0. \\ 
(\text{II}) &\le \sqrt{ln}\epsilon_m \cdot 1 \cdot \lambda^{-1}\cdot O(1) \rightarrow 0. \\ 
\|\tilde{L}-L\| &= \|K_{zz}(K_{zz}+\nu I)^{-1} - \tilde{K}_{zz}(\tilde{K}_{zz}+\nu I)^{-1}\| \rightarrow 0. \\
&\le 
\|K_{zz}-\tilde{K}_{zz}\|\cdot \nu^{-1} + \|\tilde{K}_{zz}(\tilde{K}_{zz}+\nu I)^{-1}\|\|(K_{zz}-\tilde{K}_{zz})(K_{zz}+\nu I)^{-1}\| \\
&\le 2\sqrt{n}\epsilon_m \cdot \nu^{-1} \rightarrow 0.
\\
(\text{III}) &\le O(1)\cdot \|\tilde{L}-L\|\cdot \lambda^{-1}O(1)
\rightarrow 0 \\  
(\text{IV}) &= O(1)\cdot \|(\lambda I+\tilde{K}_{xx}\tilde{L})^{-1}\|\|\tilde{K}_{xx}\tilde{L}-K_{xx}L\|\|(\lambda I+ K_{xx}L)^{-1}\| \\
&\le  O(1)\cdot\lambda^{-2}\cdot (\|\tilde{K}_{xx}-K_{xx}\|\|\tilde{L}\| + \|K_{xx}\|\|\tilde{L}-L\|)
\rightarrow 0. \\ 
(\text{V}) &= O(1) \cdot \sqrt{ln}\epsilon_m \rightarrow 0.
\end{align*}
Moreover, the converges above are all independent of the choice of $x^*$. 
Thus we have $$
\sup_{x^*\in\cX^l} \|\tilde{S}-S\|_F\le l \sup_{x^*\in\cX^l} \|\tilde{S}-S\|\rightarrow 0.
$$
Using a similar argument we have 
$
\sup_{x^*\in\cX^l} \|\tilde{\mu}-\mu\|_2 \rightarrow 0.
$
\end{proof}

\subsection{Analysis of the Optimization Algorithm}\label{app:ana-gda}

\begin{algorithm}[htb]
  \SetAlgoLined
  \KwIn{Hyperparameters $\nu,\lambda\in\RR$. Random feature models  
  $\theta\mapsto f(\cdot;\theta)$, $\varphi\mapsto g(\cdot;\varphi)$.
  }
  \KwResult{A single sample from the approximate posterior}
  Initialize: draw $\theta_0\sim \cN(0,I), \varphi_0\sim \cN(0,\lambda\nu^{-1}I),\tilde{Y}\sim\cN(Y,\lambda I)$\;
  \For(){$\ell\gets 1,\ldots,L-1$}{
    $\hat{\theta}_\ell \gets \theta_{\ell-1} - \eta_\ell \hat \nabla_{\theta} \cL_{\mrm{rf}}(\theta_{\ell-1},\varphi_{\ell-1},\theta_0,\varphi_0)$\;
    $\hat{\varphi}_\ell \gets \varphi_{\ell-1} + \eta_\ell \hat \nabla_{\varphi} \cL_{\mrm{rf}}(\theta_{\ell-1},\varphi_{\ell-1},\theta_0,\varphi_0)$\;
    $\theta_{\ell+1} \gets \mrm{Proj}_{B_f}(\hat{\theta}_{\ell})$\;
    $\varphi_{\ell+1} \gets \mrm{Proj}_{B_g}(\hat{\varphi}_{\ell})$\;
  }
  \Return{$f(\cdot;\theta_L)$}
  \caption{Modified randomized prior algorithm for approximate inference. }\label{alg:ana}
\end{algorithm}

For the purpose of the analysis we consider the standard SGDA algorithm as outlined in Algorithm~\ref{alg:ana}. 
In the algorithm $\cL_{\mrm{rf}}$ denotes the objective in \eqref{eq:rf-obj}, 
and $\mrm{Proj}_B$ denotes the projection into the $\ell_2$-norm ball with radius $B$,
and $\hat \nabla \cL_{\mrm{rf}}$ represents a stochastic (unbiased) approximation of the gradient $\nabla \cL_{\mrm{rf}}$.
In the following, we will suppress the dependency of $\cL_{\mrm{rf}}$ on $\theta_0, \varphi_0$ for simplicity.

Concretely, 
we introduce the notations %
$$
\Phi_f := \frac{1}{\sqrt{m}}\begin{bmatrix}
\phi_{x,m}(x_1)^\top \\ 
\vdots \\
\phi_{x,m}(x_n)^\top
\end{bmatrix}\in\mb{R}^{n\times m}, 
\quad
\Phi_g := \frac{1}{\sqrt{m}}\begin{bmatrix}
\phi_{z,m}(z_1)^\top \\ 
\vdots \\
\phi_{z,m}(z_n)^\top
\end{bmatrix}\in\mb{R}^{n\times m}, 
$$
where we recall $X := (x_1,\ldots,x_n)$ and $Z := (z_1,\ldots,z_n)$ are the training data.

Observe that $\Phi_f \theta = f(X;\theta),\Phi_g \varphi = g(Z;\varphi)$, we can rewrite the objective \eqref{eq:rf-obj} as 
\begin{align}
    \cL_{\mrm{rf}}(\theta, \varphi) &= 
    \theta^\top \Phi_f^\top \Phi_g\varphi -\tilde{Y}^\top \Phi_g \varphi - \frac{1}{2} \varphi^\top \Phi_g^\top \Phi_g \varphi - \frac{\nu}{2}\|\varphi - \varphi_0\|_2^2 + \frac{\lambda}{2}\|\theta - \theta_0\|_2^2.
\end{align}
We additionally define
\begin{align*}
    \cL_i(\theta, \varphi) = n\left(\theta^\top \Phi_f^\top E_i \Phi_g\varphi -\tilde{Y}^\top E_i \Phi_g \varphi - \frac{1}{2} \varphi^\top \Phi_g^\top E_i \Phi_g \varphi \right)- \frac{\nu}{2}\|\varphi - \varphi_0\|_2^2 + \frac{\lambda}{2}\|\theta - \theta_0\|_2^2,
\end{align*}
where $E_i := e_ie_i^\top$ and $\{e_i\}_{i\in[n]}$ is the standard orthogonal basis of $\mathbb{R}^n$. 
We can see that
$
    \cL_{\mrm{rf}}(\theta, \varphi) = \frac{1}{n} \sum_{i\in[n]} \cL_{i}(\theta, \varphi).
$
Therefore, the stochastic gradient in Algorithm~\ref{alg:ana} can be defined as 
\begin{equation}
  \label{eqn:stochastic-gradient-rf}
    \hat \nabla \cL_{\mrm{rf}}(\theta, \varphi) := 
    \nabla \cL_\cI (\theta, \varphi)
    =\sum_{i\in[n]} \nabla \cL_i(\theta, \varphi) \mathbf{1}_{i = \cI}, 
\end{equation}
where $\cI$ is a random variable sampled from the uniform distribution of the set $[n]$.

In practice we run the algorithm concurrently on $J$ sets of parameters, starting from independent draws of initial conditions $\{\theta^{(j)}_0,\phi^{(j)}_0\}$; moreover, the projection is not implemented, and there are various other modifications to further improve stability, as described in Appendix~\ref{app:impl-details}. 

The following lemma is a convergence theorem of Algorithm~\ref{alg:ana} under the choice of stochastic gradient defined in~\eqref{eqn:stochastic-gradient-rf}.

\begin{lemma}\label{lem:opt-final-result}
Fix an $m\in\mb{N}$. 
Denote by $\theta^*$ the optima of \eqref{eq:rf-obj} and take $\eta_\ell := \frac{1}{\mu(\ell + 1)}$ with $\mu = \min\{ \lambda, \nu \}$. 
Then for any $\epsilon,B_1,B_2,B_3>0$, there exist $B_f, B_g >0$
such that when $L = \Omega( \delta^{-1} \epsilon^{-2})$, 
the approximate optima $\theta_{L}$ returned by Algorithm~\ref{alg:ana} satisfies 
$$
\PP\left (\{\|\theta_L - \theta^*\|_2 > \epsilon\}
\cap E_n\right ) \leq \delta,
$$
where
$$
E_n := 
\left \{\|\theta_0\|_2 +  \|\varphi_0\|_2\le B_1, \|\tilde{Y}\|_2\le B_2, \sup_{z\in \cZ} \tilde{k}_{z,m}(z,z) +   \sup_{x\in\cX} \tilde{k}_{x,m}(x,x)\le B_3
\right \},
$$
and $\tilde{k}_{\cdot,m}$ denotes the random feature-approximated kernel. The randomness in the statement above is from the sampling of the initial values $\theta_0,\varphi_0$, the gradient noise. %
\end{lemma}

\begin{proof}
Recall from \eqref{eq:rf-obj-optima} that 
$\theta^*$ is a sum of bounded linear transforms of $\theta_0,\varphi_0$ and $\tilde{Y}_0$. Thus on the event $E_n$, %
$\|\theta^*\|_2$ is bounded. 
Similarly, $\|\varphi^*\|_2$ is also bounded on $E_n$ by \eqref{eq:rf-obj-g-opt}.
We choose $B_f$ and $B_g$ to be their maximum values on the event $E_n$.

As 
$\cL_{\mrm{rf}}$ is strongly-convex in $\theta$, and strongly-concave in $\varphi$, it has the unique stationary point $(\theta^*, \varphi^*)$.
We will then bound 
$
    \|\theta_\ell - \theta^{*}\|_2^2 + \|\varphi_\ell - \varphi^*\|_2^2
$.
Let $\sigma_f,\sigma_g$ be the minimal constants such that $\|\nabla_{\theta} \cL_{i}(\theta, \varphi)\|_2^2 \leq \sigma_f^2$, $\|\nabla_{\varphi} \cL_{i}(\theta, \varphi)\|_2^2 \leq \sigma_g^2$ for all $i \in [n], \| \theta \|_2 \leq B_f$ and $\| \varphi \|_2 \leq B_g$. Denote 
$B:=\max\{B_f,B_g\}$, so we have $\|\theta\|_2, \|\varphi\|_2 \leq B$. 
Define 
\begin{align*}
    r_\ell = \mathbb{E} \left[\|\theta_\ell - \theta^*\|_2^2 + \|\varphi_\ell - \varphi^*\|_2^2\right].
\end{align*}
We want to know how $r_\ell$ contracts. 
We first make a stochastic gradient step on $\theta_\ell$ with step size $\eta_\ell$, i.e., $\hat \theta_{\ell + 1} := \theta_\ell - \eta_\ell \hat \nabla_\theta \cL_{\mrm{rf}}(\theta_\ell, \varphi_\ell)$ with $\hat\nabla \cL_{\mrm{rf}}$ defined in \eqref{eqn:stochastic-gradient-rf}.
Then, 
\begin{align*}
     \mathbb{E}[\|\hat{\theta}_{\ell+1} - \theta^*\|_2^2 ~|~ \theta_{\ell}, \varphi_{\ell}] 
    \leq  \|\theta_{\ell} - \theta^*\|_2^2 - 2\eta_\ell\langle \theta_{\ell} - \theta^*, \nabla_{\theta} \cL(\theta_{\ell}, \varphi_{\ell})\rangle + \eta_\ell^2 \sigma_f^2,
\end{align*}
where the expectation is taken with respect to the randomness of the gradient.
For the above inner product term, we have that
\begin{align*}
     \langle \theta_{\ell} - \theta^*, \nabla_{\theta} \cL_{\mrm{rf}}(\theta_{\ell}, \varphi_{\ell})\rangle
    &= \langle \theta_{\ell} - \theta^*, \nabla_{\theta} \cL_{\mrm{rf}}(\theta_{\ell}, \varphi_{\ell}) - \nabla_{\theta} \cL_{\mrm{rf}}(\theta^*, \varphi^*)\rangle\\
    &= \lambda\|\theta_{\ell} - \theta^*\|_2^2 + \langle \theta_{\ell} - \theta^*, \Phi_f^\top \Phi_g (\varphi_{\ell} - \varphi^*) \rangle.
\end{align*}
Next, we consider the gradient step on $\varphi_\ell$ with step size $\eta_\ell$, i.e., 
$\hat \varphi_{\ell + 1} := \varphi_\ell + \eta_\ell \hat \nabla_\varphi \cL_{\mrm{rf}}(\theta_\ell, \varphi_\ell)$.
Then, we have that
\begin{align*}
     \mathbb{E}[\|\hat{\varphi}_{\ell + 1} - \varphi^*\|_2^2 ~|~ \theta_{\ell}, \varphi_{\ell}]
    \leq  \|\varphi_{\ell} - \varphi^*\|_2^2 + 2\eta_\ell \langle \varphi_{\ell} - \varphi^*, \nabla_{\varphi} \cL_{\mrm{rf}}(\theta_{\ell}, \varphi_{\ell}) \rangle + \eta_\ell^2\sigma_g^2.
\end{align*}
We similarly deal with the inner product term:
\begin{align*}
     \langle \varphi_{\ell} - \varphi^*, \nabla_{\varphi} \cL_{\mrm{rf}}({\theta}_{\ell}, \varphi_{\ell}) \rangle
    &= \langle \varphi_{\ell} - \varphi^*, \nabla_{\varphi} \cL_{\mrm{rf}}({\theta}_{\ell}, \varphi_{\ell}) - \nabla_{\varphi} \cL_{\mrm{rf}}(\theta^*, \varphi^*)\rangle\\
    &= -\langle \varphi_{\ell} - \varphi^*, (\Phi_g^\top \Phi_g + \nu I)(\varphi_{\ell} - \varphi^*)\rangle + \langle \varphi_{\ell} - \varphi^*, \Phi_g^\top \Phi_f({\theta}_{\ell} - \theta^*) \rangle \\
    &\leq -\nu\| \varphi_{\ell} - \varphi^* \|^2_2 + \langle \varphi_{\ell} - \varphi^*, \Phi_g^\top \Phi_f({\theta}_{\ell} - \theta^*) \rangle,
\end{align*}
Combining the above results, we have
\begin{align*}
r_{\ell + 1} \leq \EE [\|\hat{\theta}_{\ell+1}-\theta^*\|_2^2+\|\hat{\varphi}_{\ell+1}-\varphi^*\|_2^2\mid \theta_\ell,\varphi_\ell]  
\le (1 - 2 \mu \eta_\ell) r_\ell 
+ \eta_\ell^2(\sigma_f^2 + \sigma_g^2),
\end{align*}
where we have set $\mu := \min\{ \nu, \lambda \}$, and the first inequality follows from the fact that the projection onto a convex set is a contraction map, i.e., $\| \mrm{Proj}_B(x) - \mrm{Proj}_B(y) \| \leq \|x - y \|$.

Let $\sigma^2 = \sigma_f^2 + \sigma_g^2$ and $\eta_\ell = \frac{\xi}{\ell+1}$ for some $ \xi > \frac{1}{2\mu}$, by induction we have
\begin{align*}
    r_{\ell} \leq \frac{c_{\xi}}{\ell + 1}, \quad \text{where }
c_{\xi} = \max\left \{ r_0, \frac{2\xi^2 \sigma^2}{2\mu\xi - 1}\right \}.
\end{align*}
Specifically, taking $\xi = \mu^{-1}$, we have
\begin{equation}
  \label{eqn:r_ell-contraction}
   r_\ell \leq \frac{1}{\ell + 1} \max\left \{ r_0, \frac{2\sigma^2}{\mu^2} \right \}.
\end{equation}

We now track the constants we have used in \eqref{eqn:r_ell-contraction}. 
Note that on the event $E_n$, 
\[ r_0 
\leq 2 \left(\|\theta_0\|_2^2 + \|\theta^*\|_2^2 + \|\varphi_0\|_2^2 + \|\varphi^*\|_2^2\right) 
\leq 4(B_1^2 + B^2). \]  

Recall the definition of $\sigma^2$:
\begin{align*}
    \sigma^2= \max_{i\in [n], \|\theta\|_2, \|\varphi\|_2\leq B} \|\nabla_{\theta}\cL_i(\theta, \varphi)\|_2^2 + \max_{i\in [n], \|\theta\|_2, \|\varphi\|_2\leq B} \|\nabla_{\varphi} \cL_i(\theta, \varphi)\|_2^2 =: (\mrm I) + (\mrm{II}).
\end{align*}

For the first term, we have
\begin{align*}
      (\mrm I)
    & =  \max_{i\in [n], \|\theta\|_2, \|\varphi\|_2\leq B} \|\lambda (\theta - \theta_0) + n \Phi_f^\top E_i \Phi_g\varphi \|_2^2 \\
    & \leq \max_{i\in[n], \|\theta\|_2, \|\varphi\|_2\leq B} \left(2 \lambda^2 \|\theta - \theta_0\|_2^2 + 2n^2 \|\Phi_f^\top E_i \Phi_g\varphi\|_2^2\right)
     \leq  4\lambda^2 (B^2 + B_1^2) + 2n^2B_3^2B^2.
\end{align*}
Similarly, for the second term, we have
\begin{align*}
    (\mrm{II})  %
     & =  \max_{i\in [n], \|\theta\|_2, \|\varphi\|_2\leq B}\|n(\theta^\top \Phi_f^\top E_i \Phi_g - \tilde{Y}^\top E_i\Phi_g - \Phi_g^\top E_i \Phi_g\varphi) - \nu(\varphi - \varphi_0)\|_2^2 \\
     &\leq 4n^2 B_3^2B^2 + 2n^2 B_2^2 B^2 + 4\nu^2(B^2 + B_1^2).
\end{align*}
Thus, we know that
\begin{align*}
    \sigma^2 \leq 8(\lambda^2 + \nu^2) (B^2 + B_1^2) + 6 n^2 B_3^2B^2 + 2n^2 B_2^2B^2 =: \tilde C.
\end{align*}
Taking $L_\delta = \delta^{-1} \epsilon^{-2} \max\{4B_1^2 + 4B^2, \tilde C \mu^{-1} \}$ and $\eta_{ \ell} = \frac{1}{\mu(\ell + 1)}$, by \eqref{eqn:r_ell-contraction}, we know that
\begin{align*}
    \mathbb{P}(\|\theta_{L} - \theta^*\|_2 > \epsilon) 
    \leq \epsilon^{-2} \EE \| \theta_L - \theta^* \|_2^2 
    \leq \epsilon^{-2} r_\ell
    \leq \delta.
\end{align*}
\end{proof}

\subsection{Proof of Proposition~\ref{prop:approx-inf}}\label{app:prop-approxinf-proof}

By Lemma~\ref{lem:rf-approx-appendix}, for any $\epsilon_1>0$ we have 
\begin{equation}\label{eq:approxinf-0}
\lim_{m\rightarrow\infty}
\PP\left (
\left \{\sup_{x^*\in\cX^l}\|\tilde{\mu}-\mu\|_2 > \epsilon_1 \right \}\cup 
\left \{\sup_{x^*\in\cX^l}\|\tilde{S}-S\|_F > \epsilon_1     \right \}
\right ) = 0,
\end{equation}
where the randomness is from the sampling of random feature bases.

Fix an arbitrary set of $\epsilon_1>0,\delta_0>0$. Then we can find 
$m\in\mb{N}$ such that the event in \eqref{eq:approxinf-0} has probability smaller than $\delta_0$. 
Combining Assumption~\ref{ass:rf-bounded} with the fact that $\theta_0,\phi_0,\tilde{Y}_0$ are now Gaussian random variables with fixed dimensionality, %
for any $\delta_1>0$,  
we can choose $B_1,B_2,B_3$ such that %
the event $E_n$ defined in Lemma~\ref{lem:opt-final-result} has probability $1-\delta_1$. 
Thus for any $\epsilon_2>0$, when the number of iteration steps exceeds $\Omega( \delta_1^{-1} \epsilon^{-2}_2)$, we have
\begin{equation}\label{eq:approxinf-2}
\PP(\|\hat{\theta}_m-\theta^*_m\|_2 > \epsilon_2) 
\leq \PP( \{ \|\hat{\theta}_m-\theta^*_m\|_2 > \epsilon_2 \} \cap E_n) + \PP(E_n^c)
\leq 2\delta_1,
\end{equation}
where $\hat{\theta}_m$ denotes the approximate optima returned by Algorithm~\ref{alg:ana} after $\Omega( \delta_1^{-1} \epsilon^{-2}_2)$ iterations, $\theta^*_m$ denotes the exact optima of the minimax objective, and the randomness is from the gradient noise as well as the perturbations $f_0,g_0,\tilde{Y}$. 
Thus we have 
$$
\EE \|\hat{\theta}_m-\theta^*_m\|_2 
\leq \epsilon_2 + 2\delta_1 ( \EE \| \hat \theta_m \|_2 + \EE \| \theta^*_m \|_2) 
\leq \epsilon_2 + 4\delta_1 B.
$$
From the choice of $B$ in Lemma~\ref{lem:opt-final-result},
 we can see that $
\delta_1 B \le \EE(\|\theta^*_m\| \cdot (1-\mbf{1}_{E_n}))$, and thus converges to $0$ as $\delta_1\rightarrow 0$. 
Therefore, %
$\EE \|\hat{\theta}_m-\theta^*_m\|_2 $ converges to 0, and for any $x^*\in\cX^l$, 
\[ \begin{aligned}
\EE \sup_{x^*\in\cX^l}\|f(x^*;\hat{\theta}_m)-f(x^*;\theta^*_m)\|_2
&= \EE \sup_{x^*\in\cX^l}\|\phi_{x,m}(x^*)^\top (\hat{\theta}_m-\theta^*_m)\|_2 \\
&\le l \sqrt{\tilde{\kappa}} \cdot \EE \|\hat{\theta}_m-\theta^*_m\|_2 
\rightarrow 0,
\end{aligned} \]
where the expectation is taken with respect to the gradient noise, perturbations, and random feature draws. 
Hence, the mean and covariance of $f(x^*;\hat{\theta}_m)$ converges to that of $f(x^*;\theta^*_m)$ as intended, 
and we know that  the following holds with probability at least $1 - \delta_0$
$$
\sup_{x^*\in\cX^l} \max\left\{
    \|\EE(f(x^*;\hat{\theta}_m)) - \EE(f(x^*;\theta_m)) \|_2, 
    \|\mrm{Cov}(f(x^*;\hat{\theta}_m)) - \mrm{Cov}(f(x^*;\theta_m)) \|_F
\right\} \leq \epsilon_1
$$ 
Combining this with \eqref{eq:approxinf-0} completes the proof.

\section{Implementation Details, Experiment Setup and Additional Results}\label{app:impl-exp-details}

\subsection{Hyperparameter Selection}\label{app:hps-sel}

We follow the strategy in previous work \cite[e.g.,][]{singh_kernel_2020,muandet_dual_2020} and select hyperparameters by minimizing the \emph{observable} first or second stage loss, depending on which part they directly correspond to. 

For the first stage, the loss is 
$$
\cL_{v1} = \mrm{Tr}(K_{xx} - 2K_{x\tilde{x}} L + K_{\tilde x\tilde x}L^\top L) = \EE_{f\sim\mc{GP}(0,k)} \|f(X) - L f(\tilde{X})\|_2^2
$$
where $L := K_{z \tilde z}(K_{\tilde{z}\tilde{z}}+\nu I)^{-1}$, and tilde indicates the held-out data. 
From the above equality we can see that a Monte-Carlo estimator for $L_1$ can be constructed with the following procedure: 
\begin{enumerate}[label=(\roman*).]
    \item Draw $f\sim\mc{GP}(0,k_x)$.
    \item Perform kernel ridge regression on the dataset $\{(\tilde z_i, f(\tilde x_i))\}$.
    \item Return the mean squared error on the dataset $(X,Z)$.
\end{enumerate}
This procedure can also be implemented for the NN-based models. 

For the second stage, the loss $\sum_{i=1}^n \hat{d}_n(\hat{E}_n f, \hat{b})$ can be computed directly, for both the closed-form quasi-posterior and the random feature approximation. 
For the approximate inference algorithm, as we can see from \eqref{eq:rf-obj-fs} that the dual functions $\{g(\cdot;\varphi^{(k)})\}$ are samples from Gaussian process posteriors centered at the needed point estimates $\hat{E}_n f(\cdot;\theta^{(k)})$, instead of the point estimates themselves, 
we train separate validator models to approximate the latter. The validator models have the architecture to the dual functions used for training, and follow the same learning rate schedule. The validator models are trained before estimating the validation statistics, and we run SGD until convergence to ensure an accurate estimate. 

\subsection{Details in the Approximate Inference Algorithm}\label{app:impl-details}

To draw multiple samples from the quasi-posterior efficiently, our algorithm runs $J$ SGDA chains in parallel, with different perturbations $
\{(\tilde{Y}^{(j)}, f_0^{(j)}, g_0^{(j)}): j\in [J]\}
$. 
While the convergence analysis works with the extremely simple Algorithm~\ref{alg:ana}, in practice we extend it to improve stability and accelerate convergence: 
\begin{enumerate}[label=(\roman*).]
    \item we employ early stopping based on the validation statistics; 
    \item 
before the main optimization loop we initialize the dual parameters at the approximate optima $\arg\min_\varphi L_{\mrm{rf}}(f^{(j)}, g(\cdot;\varphi))$, by running SGD until convergence; 
\item 
in each SGDA iteration, we use $K_1>1$ GD steps on $g$ and one GA step for $f$; 
\item 
after every $K_2$ epochs, we fix $\theta^{(j)}$ and train the dual parameters $\varphi^{(j)}$ for one epoch. 
\end{enumerate}
All the above choices are %
shown to improve the observable validation statistics. We fix $K_1=3,K_2=2$ which are determined on the 1D datasets using the validation statistics. %

\subsection{Nystr\"om Approximation for the Closed-form Quasi-Posterior}\label{app:nystrom-expression}

Recall the definition of the Nystr\"om-approximated quasi-posterior: we change the definition of the quasi-likelihood to Assumption~\ref{ass:I-approx-2-relaxed}, where $\cI_n := \mrm{span}\{k(z_{i_j},\cdot): j\in [m]\}$, and
$z_{i_1},\ldots,z_{i_m}$ are inducing points sampled from the training data. The new closed-form expressions are then 
\begin{align}
\Pi(f(x_*)\givendata) &= \cN(
 K_{*x} \tilde\Lambda Y, 
 K_{**} - K_{*x} \tilde\Lambda K_{x*}),  \\ 
\text{where}~~ \tilde \Lambda &:=
(\lambda I + K_{xx}\tilde L)^{-1}\tilde L, \\ 
 \tilde L &:= K_{z\tilde z}(\nu K_{\tilde z\tilde z} + K_{\tilde z z}K_{z\tilde z})^{-1}K_{\tilde z z},
\end{align}
and $\tilde z$ denotes the inducing points. 
In the implementation, we compute a low-rank factorization $A^\top A := \tilde L$, defined through the Cholesky factorization of $\effnu K_{\tilde z\tilde z}+K_{\tilde zz}K_{z\tilde z}$, and use the following equivalent expression of $\tilde\Lambda$ obtained from the Woodbury identity:
$$
\tilde \Lambda = \lambda^{-1} A^\top (I - (\lambda I + A K_{xx}A^\top)^{-1} A K_{xx}A^\top) A = 
A^\top U \mrm{diag}\!\left(
    (\lambda + \gamma_1)^{-1}, \ldots, (\lambda+\gamma_n)^{-1}
    \right) U^\top A,
$$
where $U \mrm{diag}(\gamma_1,\ldots,\gamma_n) U^\top = AK_{xx}A^\top$ denotes the eigendecomposition. This ensures that matrix inversion is only performed on $m\times m$ matrices.

We use uniform sampling for inducing points, which is equivalent to the leverage score sampling scheme studied in \cite{rudi_less_2016} as we work on the symmetric $\cU[\mb{S}^1]$. 

\subsection{Experiment in Section~\ref{sec:exp-asymp-val}: Additional Results}\label{app:details-gt}

Additional results are plotted in Figure~\ref{fig:exp-asymp-val-full}. For the contraction rate of the derivative, observe the power space of $W^{s,2}(\mb{T})$ matches lower-order Sobolev spaces, so the contraction rate follows that of the power space, as in Section~\ref{sec:power-space-contraction}. 

\begin{figure}[htbp]
    \centering
    \includegraphics[width=\linewidth]{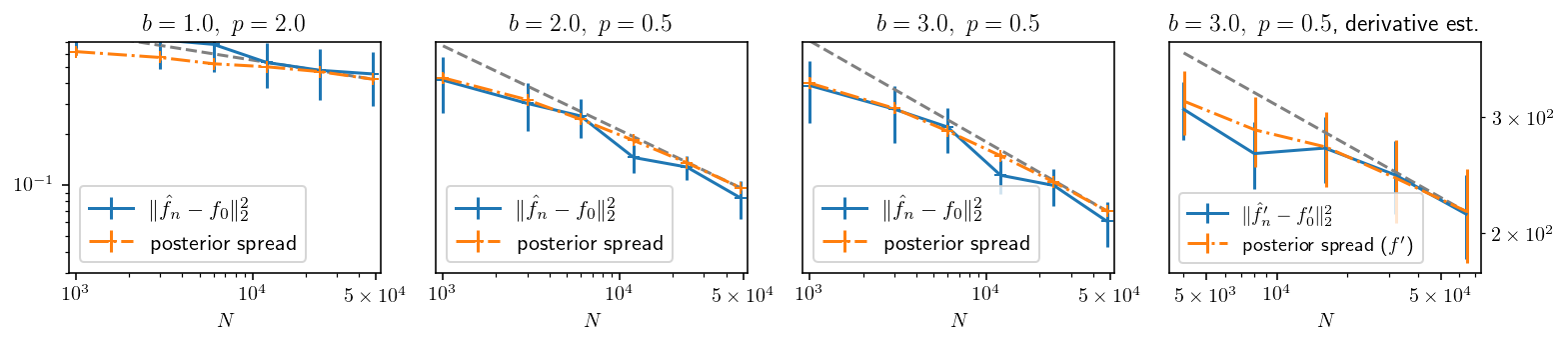}
        \caption{Additional results. 
    The left three plots correspond to the estimation of $f$; the rightmost plot corresponds to the estimation of $f'$. Dashed line indicates the theoretical rate. %
    }\label{fig:exp-asymp-val-full}
\end{figure}

\subsection{1D Simulation: Experiment Setup Details}\label{app:details-1d}

In constructing the datasets, let $\tilde f_0$ denote the sine, step, abs or linear ($\tilde f_0(x)=x$) function; we then set $f_0 = \tilde f_0(4\cdot(2x-1))$ if $\tilde f_0$ is sine, abs or linear, $\mbf{1}_{\{2x-1<0\}}+2.5\cdot \mbf{1}_{\{2x-1\ge 0\}}$ otherwise. 
These choices are made to maintain consistency with previous work \citep{lewis2018adversarial,bennett2019deep}, which used the same transformed step function and defined $\bx$ so that it has a range of approximately $[-4,4]$. 

For 2SLS and the kernelized IV methods, 
we determine $\lambda$ and $\nu$ following~\ref{app:hps-sel}. To improve stability, we repeat the procedures on $50$ random partitions of the combined training and validation set, and choose the hyperparameters that minimize the average loss. The hyperparameters are chosen from a log-uniform grid consisting of $10$ values in the range of $[0.1, 30]$. The occasional instability of hyperparameter selection is also reported in \cite{muandet_dual_2020}. %
For BayesIV, we run the MCMC sampler for 25000 iterations, discard the first 
5000 iterations for burn in, and take one sample out of every 80 consecutive iterations to construct the approximate posterior. 
For bootstrap we use 20 samples. In both cases, further increasing the computational budget will not improve performance. 

We normalize the dataset to have zero mean and unit variance. For all kernel methods we set the kernel bandwidth using the median trick. 

\subsection{1D Simulation: Full Results}\label{app:addi-1d}

Normalized MSE, average CI coverage and CI width on test set are provided in Table~\ref{tbl:cubic-sim-0}-\ref{tbl:cubic-sim-1}. 
We can see that the quasi-posterior consistently provides the most reliable coverage, and 
is most advantageous in the small-sample setting, or when instrument strength is weak. 
All methods suffer from lower coverage on the \texttt{step} design, due to its discontinuity, and in the weak IV setting, due to instabilities in the hyperparameter selection process. 

We provide the following visualizations: (i) Figure~\ref{fig:cubic-viz-oversmoothed} plots the quasi-posterior on the \texttt{abs} design, for the Mat\'ern and RBF kernels; we also include the approximate inference algorithm in the plot. 
(ii) Figure~\ref{fig:cubic-viz-weak} compares the quasi-posterior and bootstrap in the weak instrument setting. 
The hyperparameters $\lambda,\nu$ are set to the median of all values selected in the repeated experiments, and hyperparameters in the approximate inference algorithm are set as in the demand experiment below.
Additional visualizations are provided in the conference version of this work.

\begin{sidewaystable}
    \centering\scriptsize
    \begin{tabular}{cccccccccccc}
\toprule Method & bayesIV &  bs-lin &  qb-lin &  bs-poly &  qb-poly &  bs-ma3 &  qb-ma3 &  bs-ma5 &  qb-ma5 &  bs-rbf &  qb-rbf \\
\midrule\multicolumn{12}{l}{ $f_0=\;$sin, $N=200, \alpha=0.5$ } \\ \midrule
MSE & .024 (.038) & .111 (.011)  & .109 (.010)  & .243 (.037)  & .243 (.034)  & .023 (.010)  & .025 (.014)  & .022 (.011)  & .021 (.015)  & .025 (.013)  & .026 (.015) \\
CI Cvg. & .895 (.252) & .232 (.039)  & .110 (.019)  & .077 (.032)  & .045 (.017)  & .965 (.065)  & 1.00 (.000)  & .972 (.065)  & 1.00 (.000)  & .972 (.079)  & 1.00 (.013) \\
CI Wid. & .188 (.035) & .143 (.028)  & .072 (.004)  & .078 (.025)  & .041 (.006)  & .283 (.031)  & .661 (.066)  & .288 (.032)  & .569 (.067)  & .293 (.040)  & .408 (.063) \\
\midrule\multicolumn{12}{l}{ $f_0=\;$sin, $N=1000, \alpha=0.5$ } \\ \midrule
MSE & .016 (.003) & .103 (.006)  & .103 (.006)  & .237 (.020)  & .239 (.020)  & .009 (.011)  & .008 (.014)  & .008 (.012)  & .007 (.015)  & .007 (.012)  & .007 (.013) \\
CI Cvg. & .598 (.155) & .097 (.020)  & .049 (.006)  & .036 (.012)  & .017 (.004)  & .962 (.113)  & 1.00 (.000)  & .954 (.111)  & 1.00 (.000)  & .957 (.168)  & 1.00 (.036) \\
CI Wid. & .085 (.006) & .061 (.011)  & .032 (.001)  & .038 (.009)  & .019 (.002)  & .186 (.032)  & .602 (.037)  & .173 (.029)  & .509 (.037)  & .164 (.029)  & .326 (.041) \\
\midrule\multicolumn{12}{l}{ $f_0=\;$abs, $N=200, \alpha=0.5$ } \\ \midrule
MSE & .042 (.038) & .454 (.052)  & .456 (.053)  & .478 (.085)  & .477 (.089)  & .033 (.026)  & .035 (.025)  & .032 (.024)  & .031 (.025)  & .031 (.019)  & .031 (.021) \\
CI Cvg. & .863 (.184) & .190 (.039)  & .085 (.198)  & .055 (.027)  & .020 (.072)  & .945 (.110)  & 1.00 (.004)  & .942 (.118)  & 1.00 (.004)  & .920 (.125)  & 1.00 (.030) \\
CI Wid. & .207 (.027) & .214 (.035)  & .077 (.220)  & .110 (.038)  & .043 (.091)  & .316 (.037)  & .676 (.072)  & .317 (.034)  & .599 (.079)  & .277 (.037)  & .462 (.082) \\
\midrule\multicolumn{12}{l}{ $f_0=\;$abs, $N=1000, \alpha=0.5$ } \\ \midrule
MSE & .024 (.011) & .448 (.016)  & .449 (.016)  & .468 (.025)  & .469 (.026)  & .020 (.006)  & .019 (.005)  & .019 (.006)  & .018 (.006)  & .017 (.007)  & .016 (.006) \\
CI Cvg. & .507 (.196) & .083 (.018)  & .111 (.092)  & .028 (.007)  & .009 (.003)  & .857 (.075)  & 1.00 (.000)  & .823 (.079)  & 1.00 (.000)  & .829 (.102)  & 1.00 (.016) \\
CI Wid. & .092 (.005) & .098 (.019)  & .127 (.103)  & .061 (.016)  & .019 (.002)  & .181 (.024)  & .646 (.050)  & .174 (.026)  & .530 (.056)  & .168 (.024)  & .383 (.061) \\
\midrule\multicolumn{12}{l}{ $f_0=\;$step, $N=200, \alpha=0.5$ } \\ \midrule
MSE & .045 (.026) & .075 (.010)  & .075 (.010)  & .179 (.025)  & .180 (.023)  & .041 (.013)  & .047 (.017)  & .043 (.012)  & .046 (.015)  & .046 (.011)  & .048 (.013) \\
CI Cvg. & .845 (.176) & .347 (.069)  & .220 (.045)  & .110 (.048)  & .067 (.022)  & .797 (.101)  & 1.00 (.022)  & .787 (.085)  & .998 (.038)  & .710 (.085)  & .917 (.051) \\
CI Wid. & .194 (.018) & .139 (.023)  & .072 (.004)  & .068 (.022)  & .041 (.006)  & .300 (.047)  & .665 (.083)  & .285 (.041)  & .593 (.082)  & .252 (.035)  & .453 (.061) \\
\midrule\multicolumn{12}{l}{ $f_0=\;$step, $N=1000, \alpha=0.5$ } \\ \midrule
MSE & .023 (.009) & .070 (.004)  & .069 (.004)  & .178 (.012)  & .176 (.011)  & .035 (.012)  & .038 (.015)  & .038 (.011)  & .040 (.014)  & .039 (.012)  & .040 (.014) \\
CI Cvg. & .616 (.116) & .185 (.042)  & .098 (.014)  & .053 (.021)  & .030 (.005)  & .784 (.153)  & 1.00 (.016)  & .739 (.164)  & .976 (.028)  & .661 (.134)  & .839 (.077) \\
CI Wid. & .086 (.004) & .060 (.011)  & .032 (.001)  & .032 (.010)  & .019 (.002)  & .206 (.062)  & .565 (.053)  & .196 (.073)  & .483 (.065)  & .168 (.031)  & .312 (.054) \\
\midrule\multicolumn{12}{l}{ $f_0=\;$linear, $N=200, \alpha=0.5$ } \\ \midrule
MSE & .009 (.011) & .002 (.002)  & .001 (.002)  & .128 (.019)  & .128 (.017)  & .012 (.008)  & .017 (.013)  & .011 (.009)  & .014 (.014)  & .011 (.011)  & .013 (.016) \\
CI Cvg. & .948 (.091) & 1.00 (.153)  & 1.00 (.130)  & .087 (.026)  & .060 (.018)  & .995 (.044)  & 1.00 (.001)  & .995 (.050)  & 1.00 (.003)  & .990 (.060)  & 1.00 (.043) \\
CI Wid. & .129 (.025) & .088 (.012)  & .072 (.004)  & .055 (.017)  & .041 (.006)  & .269 (.031)  & .520 (.049)  & .261 (.031)  & .438 (.058)  & .239 (.034)  & .298 (.041) \\
\midrule\multicolumn{12}{l}{ $f_0=\;$linear, $N=1000, \alpha=0.5$ } \\ \midrule
MSE & .005 (.002) & .000 (.001)  & .000 (.001)  & .121 (.012)  & .121 (.012)  & .006 (.004)  & .007 (.005)  & .006 (.003)  & .007 (.005)  & .006 (.003)  & .005 (.004) \\
CI Cvg. & .626 (.130) & 1.00 (.174)  & 1.00 (.230)  & .034 (.008)  & .026 (.004)  & .992 (.094)  & 1.00 (.000)  & .990 (.095)  & 1.00 (.000)  & .975 (.103)  & 1.00 (.000) \\
CI Wid. & .051 (.002) & .039 (.006)  & .032 (.001)  & .026 (.006)  & .019 (.002)  & .171 (.031)  & .508 (.043)  & .156 (.032)  & .418 (.044)  & .135 (.024)  & .242 (.043) \\

        \bottomrule
        \end{tabular}
        \caption{Full results in the 1D simulation, for $\alpha=0.5$}\label{tbl:cubic-sim-0}
        \end{sidewaystable}

\begin{sidewaystable}
    \centering\scriptsize
    \begin{tabular}{cccccccccccc}
\toprule Method & bayesIV &  bs-lin &  qb-lin &  bs-poly &  qb-poly &  bs-ma3 &  qb-ma3 &  bs-ma5 &  qb-ma5 &  bs-rbf &  qb-rbf \\
\midrule\multicolumn{12}{l}{ $f_0=\;$sin, $N=200, \alpha=0.05$ } \\ \midrule
MSE & .275 (.045) & .133 (.070)  & .193 (.125)  & .202 (.121)  & .215 (.161)  & .231 (.037)  & .190 (.068)  & .209 (.037)  & .163 (.073)  & .183 (.047)  & .142 (.081) \\
CI Cvg. & .165 (.077) & .992 (.080)  & 1.00 (.134)  & .325 (.092)  & .425 (.080)  & .270 (.082)  & .960 (.084)  & .332 (.121)  & .955 (.086)  & .468 (.219)  & .952 (.134) \\
CI Wid. & .192 (.030) & .589 (.166)  & .971 (.481)  & .394 (.171)  & .771 (.265)  & .192 (.036)  & .712 (.045)  & .233 (.054)  & .694 (.065)  & .297 (.085)  & .638 (.105) \\
\midrule\multicolumn{12}{l}{ $f_0=\;$sin, $N=1000, \alpha=0.05$ } \\ \midrule
MSE & .146 (.025) & .123 (.077)  & .169 (.315)  & .213 (.145)  & .228 (.192)  & .246 (.071)  & .216 (.113)  & .238 (.085)  & .214 (.131)  & .188 (.096)  & .188 (.131) \\
CI Cvg. & .082 (.060) & .880 (.296)  & .888 (.346)  & .289 (.086)  & .344 (.135)  & .373 (.111)  & .819 (.131)  & .436 (.145)  & .840 (.185)  & .562 (.218)  & .897 (.212) \\
CI Wid. & .095 (.018) & .552 (.367)  & .852 (.568)  & .405 (.294)  & .588 (.400)  & .254 (.036)  & .605 (.049)  & .298 (.042)  & .568 (.050)  & .373 (.086)  & .536 (.055) \\
\midrule\multicolumn{12}{l}{ $f_0=\;$abs, $N=200, \alpha=0.05$ } \\ \midrule
MSE & .806 (.478) & .487 (.259)  & .500 (.505)  & .489 (.233)  & .472 (.453)  & .392 (.064)  & .349 (.156)  & .368 (.074)  & .350 (.180)  & .336 (.094)  & .393 (.221) \\
CI Cvg. & .122 (.197) & .435 (.167)  & .775 (.201)  & .247 (.119)  & .545 (.122)  & .217 (.102)  & .795 (.137)  & .287 (.130)  & .832 (.163)  & .352 (.243)  & .805 (.250) \\
CI Wid. & .239 (.049) & .526 (.311)  & 1.30 (.473)  & .303 (.152)  & .895 (.279)  & .294 (.047)  & .712 (.045)  & .351 (.065)  & .694 (.065)  & .424 (.103)  & .638 (.105) \\
\midrule\multicolumn{12}{l}{ $f_0=\;$abs, $N=1000, \alpha=0.05$ } \\ \midrule
MSE & 1.45 (.250) & .479 (.105)  & .500 (.083)  & .472 (.159)  & .472 (.111)  & .376 (.090)  & .390 (.144)  & .374 (.109)  & .367 (.181)  & .304 (.134)  & .265 (.214) \\
CI Cvg. & .019 (.014) & .464 (.179)  & .742 (.249)  & .332 (.165)  & .562 (.191)  & .306 (.109)  & .665 (.207)  & .374 (.165)  & .667 (.248)  & .505 (.269)  & .625 (.308) \\
CI Wid. & .139 (.017) & .460 (.328)  & 1.24 (.576)  & .340 (.238)  & .923 (.384)  & .339 (.044)  & .605 (.049)  & .367 (.056)  & .561 (.049)  & .504 (.114)  & .536 (.061) \\
\midrule\multicolumn{12}{l}{ $f_0=\;$step, $N=200, \alpha=0.05$ } \\ \midrule
MSE & .226 (.127) & .105 (.069)  & .148 (.199)  & .157 (.092)  & .192 (.148)  & .214 (.036)  & .183 (.070)  & .194 (.038)  & .176 (.077)  & .160 (.056)  & .147 (.090) \\
CI Cvg. & .193 (.090) & .777 (.158)  & .787 (.197)  & .382 (.079)  & .438 (.066)  & .262 (.103)  & .952 (.068)  & .300 (.158)  & .920 (.089)  & .432 (.282)  & .890 (.149) \\
CI Wid. & .184 (.032) & .621 (.297)  & .767 (.452)  & .456 (.179)  & .652 (.292)  & .262 (.051)  & .712 (.045)  & .310 (.070)  & .694 (.065)  & .365 (.103)  & .638 (.112) \\
\midrule\multicolumn{12}{l}{ $f_0=\;$step, $N=1000, \alpha=0.05$ } \\ \midrule
MSE & .079 (.444) & .083 (.044)  & .131 (.207)  & .152 (.105)  & .236 (.120)  & .231 (.059)  & .214 (.107)  & .224 (.070)  & .208 (.127)  & .192 (.082)  & .170 (.130) \\
CI Cvg. & .149 (.231) & .715 (.171)  & .696 (.201)  & .366 (.087)  & .390 (.120)  & .290 (.075)  & .841 (.165)  & .353 (.144)  & .853 (.212)  & .515 (.229)  & .800 (.189) \\
CI Wid. & .125 (.021) & .544 (.295)  & .815 (.564)  & .408 (.187)  & .539 (.423)  & .298 (.037)  & .605 (.049)  & .330 (.052)  & .561 (.049)  & .420 (.095)  & .543 (.065) \\
\midrule\multicolumn{12}{l}{ $f_0=\;$linear, $N=200, \alpha=0.05$ } \\ \midrule
MSE & .083 (.023) & .013 (.032)  & .041 (.202)  & .103 (.065)  & .151 (.149)  & .052 (.015)  & .046 (.024)  & .046 (.016)  & .045 (.023)  & .046 (.020)  & .053 (.023) \\
CI Cvg. & .238 (.133) & 1.00 (.284)  & 1.00 (.172)  & .372 (.099)  & .412 (.109)  & .490 (.194)  & 1.00 (.000)  & .750 (.219)  & 1.00 (.000)  & .895 (.188)  & 1.00 (.026) \\
CI Wid. & .121 (.030) & .559 (.213)  & .595 (.406)  & .495 (.181)  & .543 (.327)  & .220 (.028)  & .712 (.045)  & .270 (.041)  & .694 (.065)  & .323 (.074)  & .638 (.105) \\
\midrule\multicolumn{12}{l}{ $f_0=\;$linear, $N=1000, \alpha=0.05$ } \\ \midrule
MSE & .051 (.007) & .018 (.034)  & .022 (.165)  & .125 (.092)  & .188 (.411)  & .070 (.023)  & .069 (.041)  & .070 (.028)  & .070 (.047)  & .057 (.031)  & .066 (.043) \\
CI Cvg. & .094 (.041) & 1.00 (.215)  & 1.00 (.200)  & .338 (.113)  & .298 (.122)  & .455 (.107)  & 1.00 (.028)  & .528 (.192)  & 1.00 (.066)  & .684 (.237)  & 1.00 (.096) \\
CI Wid. & .058 (.007) & .535 (.289)  & .423 (.470)  & .435 (.309)  & .406 (.365)  & .205 (.020)  & .605 (.049)  & .230 (.033)  & .561 (.049)  & .297 (.089)  & .530 (.061) \\

        \bottomrule
        \end{tabular}
        \caption{Full results in the 1D simulation, for $\alpha=0.05$}\label{tbl:cubic-sim-1}
        \end{sidewaystable}

\begin{figure}[htb]
    \centering %
    \begin{subfigure}[c]{.22\linewidth}
    \includegraphics[width=\linewidth]{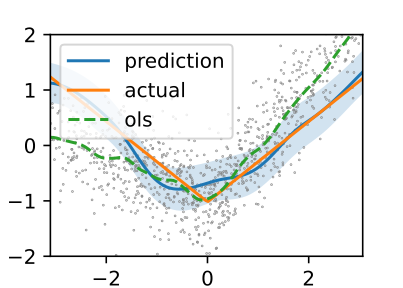}
    \caption{Mat\'ern-$3/2$ kernel}
    \end{subfigure}
    \begin{subfigure}[c]{.22\linewidth}
    \includegraphics[width=\linewidth]{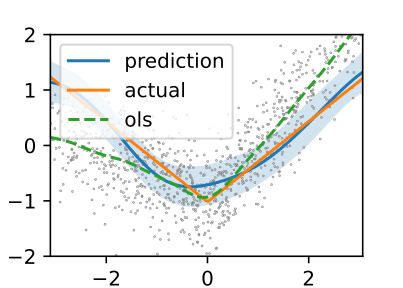}
    \caption{Mat\'ern-$5/2$ kernel}
    \end{subfigure}
    \begin{subfigure}[c]{.22\linewidth}
    \includegraphics[width=\linewidth]{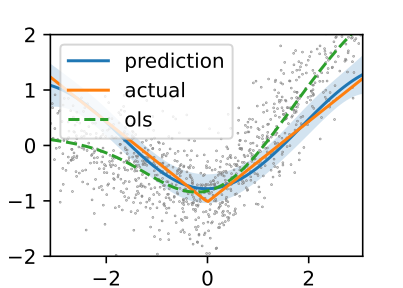}
    \caption{RBF, closed form}
    \end{subfigure}
    \begin{subfigure}[c]{.22\linewidth}
    \includegraphics[width=\linewidth]{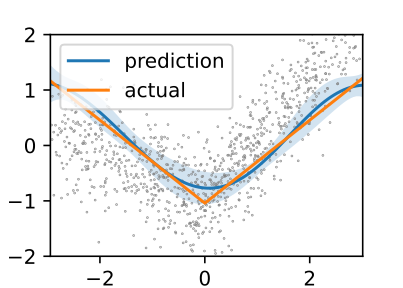}
    \caption{RBF, RF approx}
    \end{subfigure}
    \caption{1D datasets: visualization of the quasi-posterior on the \texttt{abs} design using various models. Dots indicate training data. We fix $N=1000,\alpha=0.5$.}
    \label{fig:cubic-viz-oversmoothed}
\end{figure}

\begin{figure}[htb]
    \centering
    \begin{subfigure}[c]{.47\linewidth}
        \includegraphics[width=\linewidth]{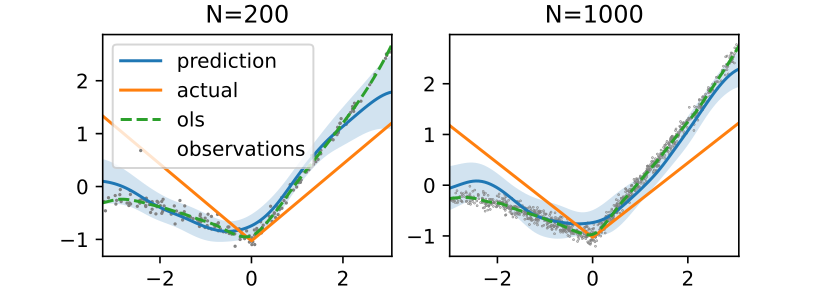}
        \caption{Bootstrap}
    \end{subfigure}
    \begin{subfigure}[c]{.47\linewidth}
        \includegraphics[width=\linewidth]{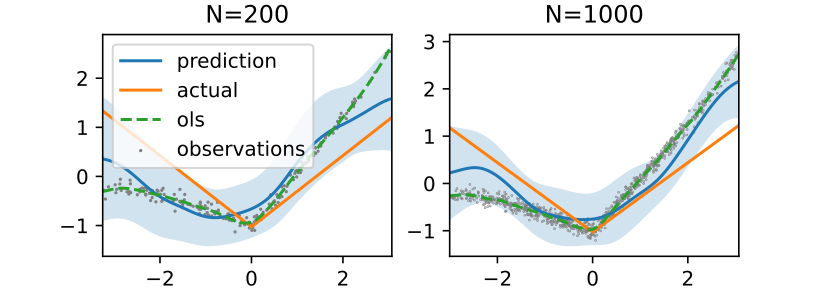}
        \caption{Quasi-Bayes}
    \end{subfigure}
    \caption{1D datasets: visualization of the weak instrument setting, on the \texttt{abs} design with the Mat\'ern$-3/2$ kernel. Results for other kernels are qualitatively similar.}\label{fig:cubic-viz-weak}
\end{figure}

\subsection{Demand Simulation: Experiment Setup Details}\label{app:demand-setup}

All variables in the dataset are normalized to have zero mean and unit variance. 
For BayesIV, we run the MCMC sampler for 50000 iterations, discard the first 10000 samples as burn in, and take every 80th sample for inference. 
For the kernelized methods, hyperparamater selection follows the 1D experiments. 
For the NN-based methods, implementation details are discussed in Appendix~\ref{app:impl-details}; 
for both our method and bootstrap, we draw $J=10$ samples from the predictive distribution. 

We select hyperparameters by applying the procedure in Appendix~\ref{app:hps-sel} to a fixed train / validation split, since on this dataset we observe little variation in its results. Hyperparameters include $\lambda,\nu$, and the learning rate schedule (initial learning rate $\eta_0$ and period of learning rate decay $\tau$). The learning rate is adjusted by multiplying it by a factor of $0.8$ every $\tau$ iterations. 
We fix the optimizer to Adam, and train until validation statistics no longer improves.

For the lower-dimensional setup, we select 
$\lambda$ and $\nu$ from a log-linearly scaled grid of 10 values, with the range of $[5\times 10^{-3}, 5]$ and $[0.05, 1]$, respectively. The ranges are chosen based on preliminary experiments using the range of $[0.1, 30]$. We determine $\eta$ from $\{5\times 10^{-4}, 10^{-3}, 5\times 10^{-3}, 1\times 10^{-2}, 5\times 10^{-2}\}$, and $\tau$ from $\{80,160,320,640\}$. We fix the batch size at $256$. The NN architecture consists of two fully-connected layers, with 50 hidden units and the tanh activation. We also experimented with NNs with 3 hidden layers or with ReLU activation, and made this choice based on the validation statistics.

For the image-based setup, the range of $\lambda$ and $\eta$ follows the above. For $\nu,\tau$ we consider $
\nu\in [1,100], \tau\in\{640,1280,2560,5120\}, 
$ based on preliminary experiments. We fix the batch size at $80$. 
The network architecture is adapted from \cite{hartford2017deep}, and consists of two $3\times 3$ convolutional layers with $64$ filters, followed by max pooling, dropout, and three fully-connected layers with $64,32$ and $1$ units. %

Following the setup in all previous work, we use a uniform grid on $[5,30]\times [0,10]\times \{0,\ldots,6\}$ as the test set.

\subsection{Demand Simulation: Full Results and Visualizations}\label{app:demand-full-results}

Results in the large-sample settings are presented in Table~\ref{tbl:hllt-full}, %
which are consistent with the discussion in the main text. %

We plot the predictive distributions for all methods in Figure~\ref{fig:demand-1k-all}, on the same cross-section as in the main text, for $N=1000$. 
(We omit the plot for $N=10^4$ and the image experiment, since in those settings bootstrap and the quasi-posterior have similar behaviors.) As we can see, 
all non-NN baselines except BayesIV produce overly smooth predictions, presumably due to the lack of flexibility in these models. 
Note that the visualizations only correspond to an intersection of the true function $f(x_0,t_0,s)$, with $x_0,t_0$ fixed; the complete function has the form of $x\cdot s\cdot \psi(t)$, ignoring the less significant terms, and thus may incur a large norm penalty in the less flexible RKHSes. 
The issue is further exacerbated by the discrepancy between the training and test distributions: the former is non-uniform due to confounding. As we can see from Figure~\ref{fig:demand-data-viz}, in the region where $t$ is close to $5$, the data is scarce for most values of $x$, which may explain the reason that the closed-form kernels fail to provide good coverage around $t=5$ (and $s=3,x=17.5$, as used in the visualizations), and the reason that both NN-based methods assign higher uncertainty around this location. 

BayesIV has a different failure mode: as it employs additive regression models for both stages $p(\bx\mid\bz), p(\by\mid\bx)$, it approximates this cross-section relatively well. However, as the true structural function does not have an additive decomposition, its prediction in other regions can be grossly inaccurate; 
we plot one such cross-section in Figure~\ref{fig:demand-1k-od}(a). 

When implemented with the NN model, bootstrap CIs are sharper in regions with more training data, although the difference is often insignificant. The difference in out-of-distribution regions is more significant, where bootstrap is often less robust. An example is shown in Figure~\ref{fig:demand-1k-od}.

\begin{table}[htb]\centering
\small
\begin{tabular}{ccccccc}
\toprule
Setting & \multicolumn{3}{c}{Low-dimensional, $N=10^4$} & \multicolumn{3}{c}{Image, $N=5\times 10^4$} \\
Method & BS-2SLS & BS-NN & QB-NN  & BS-2SLS & BS-NN & QB-NN 
\\ \midrule
NMSE & 
$.371\pm.003$ & 
$.014\pm.003$ &
$.020\pm.002$ &
$.559\pm.008$ & 
$.168\pm.027$ &
$.138\pm.037$ \\
CI Cvg. & 
$.024\pm.005$ & 
$.944\pm.009$ &
$.957\pm.008$ &
$.112\pm.005$ &
$.892\pm.022$ &
$.909\pm.017$  \\
CI Wid. &
$.014\pm .002$& 
$.136\pm .015$& 
$.203\pm .013$& 
$.132\pm .039$ & 
$.636\pm .027$ &
$.597\pm .024$  
\\ \bottomrule
\end{tabular}
\caption{Deferred results on the demand design. Results are averaged over 20 trials for the low-dimensional experiment, and 10 trials for the image experiment.}\label{tbl:hllt-full}
\end{table}

\begin{figure}[h!]
\centering 
\includegraphics[width=0.26\linewidth]{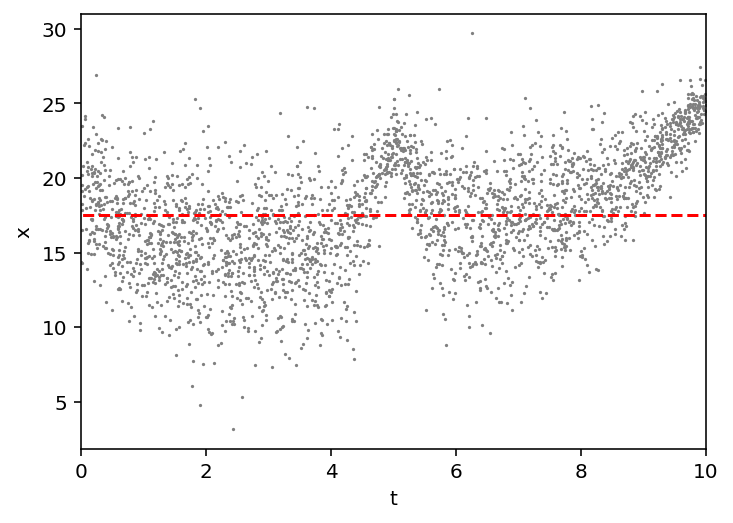}
\caption{Demand experiment: scatter plot of $10^4$ samples from the training data distribution $p(x,t\mid s=4)$. The dashed line indicates the cross-section used in Figure~\ref{fig:demand-main}. 
}\label{fig:demand-data-viz}
\end{figure}

\begin{figure}[h!]
    \centering\conditionalPlot{
\begin{subfigure}[b]{0.24\linewidth}
\includegraphics[width=\linewidth]{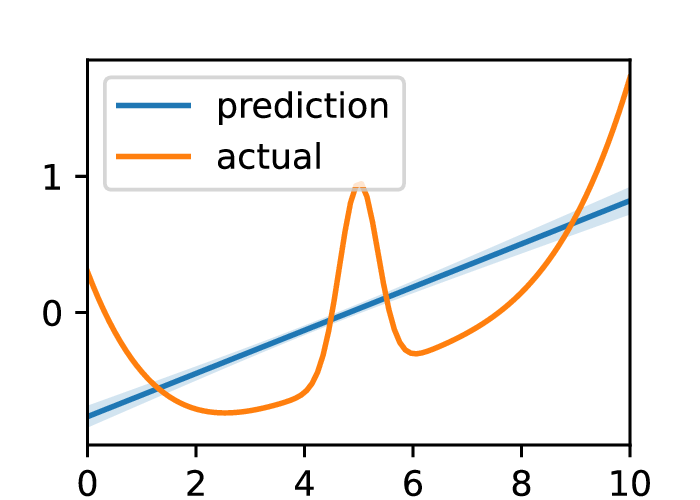}
\caption{BS-Linear}
\end{subfigure}
\begin{subfigure}[b]{0.24\linewidth}
\includegraphics[width=\linewidth]{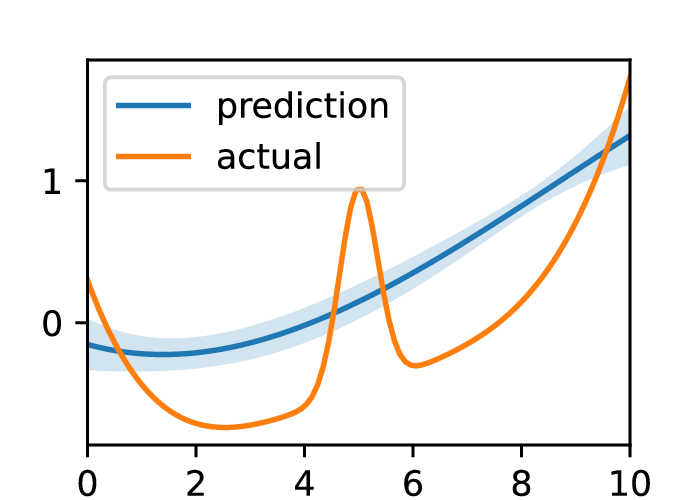}
\caption{BS-Poly}
\end{subfigure}
\begin{subfigure}[b]{0.24\linewidth}
\includegraphics[width=\linewidth]{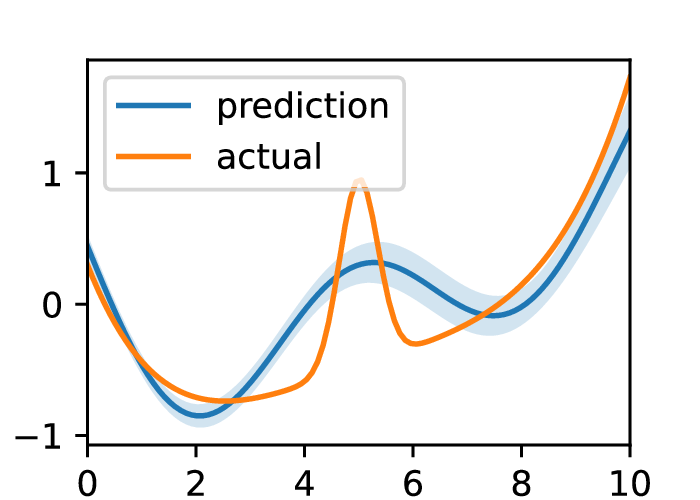}
\caption{BS-RBF}
\end{subfigure}
\begin{subfigure}[b]{0.24\linewidth}
\includegraphics[width=\linewidth]{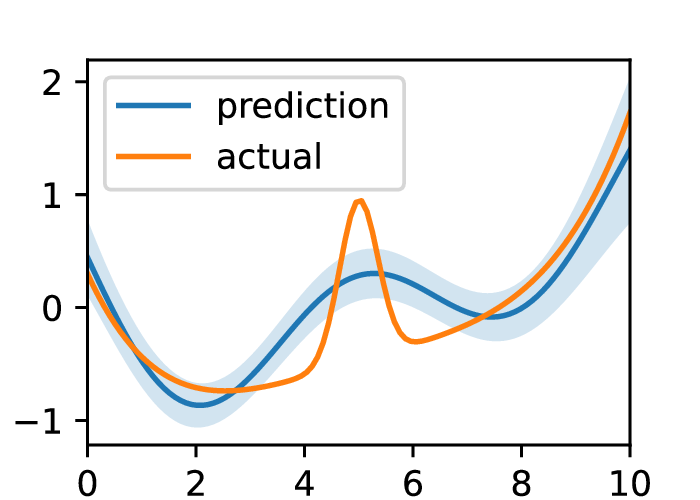}
\caption{QB-RBF}
\end{subfigure}

\begin{subfigure}[b]{0.24\linewidth}
\includegraphics[width=\linewidth]{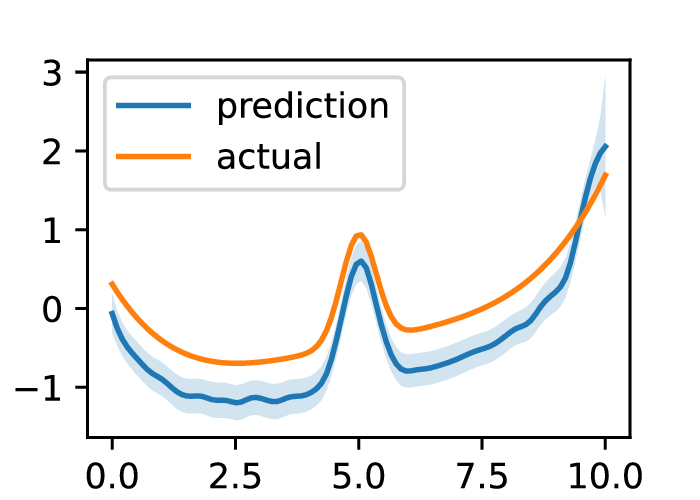}
\caption{BayesIV}
\end{subfigure}
\begin{subfigure}[b]{0.24\linewidth}
\includegraphics[width=\linewidth]{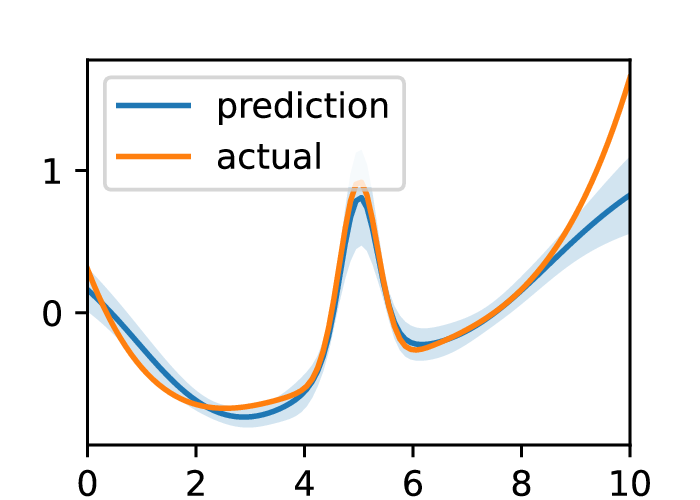}
\caption{BS-NN}
\end{subfigure}
\begin{subfigure}[b]{0.24\linewidth}
\includegraphics[width=\linewidth]{sim/hllt-1k-0.5-qb.png}
\caption{QB-NN}
\end{subfigure}}
\caption{Demand experiment: visualizations of the predictive distributions for $N=1000$, on the same cross-section as in Figure~\ref{fig:demand-main}.
}\label{fig:demand-1k-all}
\end{figure}

\begin{figure}[h!]
    \centering\conditionalPlot{
\begin{subfigure}[b]{0.24\linewidth}
\includegraphics[width=\linewidth]{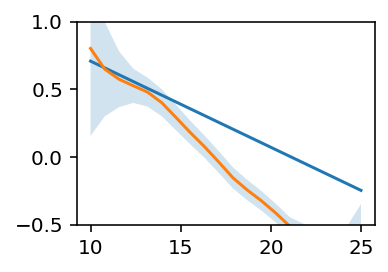}
\caption{BayesIV}
\end{subfigure}
\begin{subfigure}[b]{0.24\linewidth}
\includegraphics[width=\linewidth]{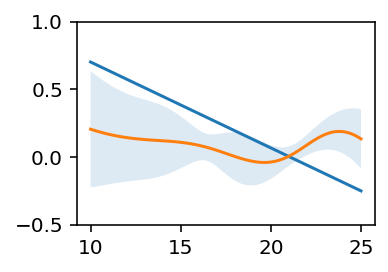}
\caption{BS-NN}
\end{subfigure}
\begin{subfigure}[b]{0.24\linewidth}
\includegraphics[width=\linewidth]{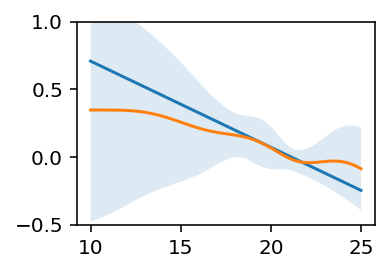}
\caption{QB-NN}
\end{subfigure}}
\caption{Demand experiment: visualizations of the predictive distributions for $N=1000$ on a out-of-distribution cross-section, obtained by fixing $t=9,s=6$ and varying $x$.}\label{fig:demand-1k-od}
\end{figure}

\subsubsection{Computational cost} We report the typical training time for a single set of hyperparameters, excluding JIT compilation time, on a GeForce GTX 1080Ti GPU. 
In the lower-dimensional experiments, training takes around 25 minutes for a single set of hyperparameters when $N=10^3$, or around 30 minutes when $N=10^4$; in both cases 6 experiments can be carried out in parallel on a single GPU.  
In the image experiment, training takes around 7.5 hours. 

The time cost above is for the optimal hyperparameter configuration; experiments using suboptimal hyperparameters usually take a shorter period of time due to early stopping. 
It can also be improved by switching to low-precision numerical operations, or with various heuristics in the hyperparameter search (e.g., using a smaller $J$ in an initial search).

\vskip 0.2in
\bibliography{main}

\end{document}